\documentclass{article}

\usepackage{iclr2024_conference_arxiv,times}

\iclrfinaltrue

\usepackage[utf8]{inputenc}
\usepackage{mathtools}
\usepackage{ifthen}
\usepackage{braket}
\usepackage{amsthm}
\usepackage{url}
\usepackage{appendix}
\usepackage[ruled, vlined]{algorithm2e}
\usepackage{hyperref}
\usepackage{subcaption}
\usepackage{multibib}
\usepackage{xcolor}
\usepackage{enumitem}
\usepackage{wrapfig}
\setlist[itemize]{leftmargin=*}
\setlist[enumerate]{leftmargin=*}
\newif\ifarxiv
\arxivfalse

\definecolor{darkgreen}{rgb}{0,0.5,0}
\definecolor{darkcyan}{rgb}{0,0.5,0.5}
\definecolor{darkred}{rgb}{0.5,0,0}

\newcites{AR}{Additional References}

\DeclareMathOperator*{\E}{\mathbb{E}}

\newcommand{\smallparagraph}[1]{\smallskip\noindent\textbf{#1}}

\title{The Wasserstein Believer\\
\large Learning Belief Updates for Partially Observable Environments through Reliable Latent Space Models}

\author{%
  Raphael Avalos$^{1}$\thanks{Both authors contributed equally to this research, alphabetic order.} \quad Florent Delgrange$^{1,2*}$ \\ \textbf{Ann Now\'e}$^{1}$ \quad \textbf{Guillermo A. P\'erez}$^{2,3}$ \quad  \textbf{Diederik M. Roijers$^{1,4}$}\\
  $^1$ AI Lab, Vrije Universiteit Brussel (BE) \quad $^2$ University of Antwerp (BE) \\ 
  $^3$ Flanders Make (BE) \quad $^4$ Urban Innovation and R\&D, City of Amsterdam (NL) \\
  \texttt{\{raphael.avalos, florent.delgrange\}@vub.be}
}

\usepackage{amsmath,amsfonts,bm}
\usepackage{mathabx}
\usepackage{dsfont}
\usepackage{nicefrac}

\makeatletter
\newcommand{\pushright}[1]{\ifmeasuring@#1\else\omit\hfill$\displaystyle#1$\fi\ignorespaces}
\newcommand{\pushleft}[1]{\ifmeasuring@#1\else\omit$\displaystyle#1$\hfill\fi\ignorespaces}
\makeatother

\newtheorem{theorem}{Theorem}[section]
\newtheorem{definition}[theorem]{Definition}
\newtheorem{corollary}[theorem]{Corollary}
\newtheorem{lemma}[theorem]{Lemma}
\newtheorem{property}[theorem]{Property}
\newtheorem{assumption}{Assumption}
\theoremstyle{remark}
\newtheorem{remark}{Remark}
\newtheorem{example}{Example}
\newtheorem{notation}{Notation}
\newenvironment{proofsketch}{%
  \proof}{\endproof}

\newcommand{\tuple}[1]{\ensuremath{\left\langle #1 \right\rangle}}
\newcommand{\indexof}[2]{\ensuremath{#1_{\scriptscriptstyle #2}}}
\newcommand{\N}{\mathbb{N}}
\newcommand{\fun}[1]{\ensuremath{\mathopen{}\mathclose\bgroup\left(#1\aftergroup\egroup\right)}}
\newcommand{\condition}[1]{\ensuremath{\1_{#1}}}
\newcommand{\diracimpulsesymbol}{\ensuremath{\delta}}
\newcommand{\diracimpulse}[2]{\ensuremath{\diracimpulsesymbol_{#2}\fun{#1}}}

\newcommand{\mdp}{\ensuremath{\mathcal{M}}}
\newcommand{\states}{\ensuremath{\mathcal{S}}}
\newcommand{\actions}{\ensuremath{\mathcal{A}}}

\newcommand{\probtransitions}{\ensuremath{\mathbf{P}}} %
\newcommand{\rewards}{\ensuremath{\mathcal{R}}}

\newcommand{\sinit}{\ensuremath{s_{\mathit{I}}}}
\newcommand{\mdptuple}{\langle \states, \actions, \probtransitions, \rewards, \allowbreak \sinit, \discount \rangle}

\newcommand{\state}{\ensuremath{s}}

\newcommand{\action}{\ensuremath{a}}

\newcommand{\reward}{\ensuremath{r}}

\newcommand{\act}[1]{\ensuremath{\mathit{Act}\ifthenelse{\equal{#1}{}}{}{(#1)}}}

\newcommand{\seq}[2]{\ensuremath{#1_{\scriptscriptstyle 0:#2}}}

\newcommand{\trajectorytuple}[3]{\ensuremath{\tuple{#1_{\scriptscriptstyle 0:#3}, #2_{\scriptscriptstyle 0: #3-1}}}}

\newcommand{\policy}{\ensuremath{\pi}}

\newcommand{\policystates}{\ensuremath{Q}}
\newcommand{\policystate}{\ensuremath{q}}

\newcommand{\qinit}{\ensuremath{q_I}}
\newcommand{\mealyaction}[1]{\ensuremath{#1_{\alpha}}}
\newcommand{\mealyupdate}[1]{\ensuremath{#1_{\mu}}}
\newcommand{\mealymachine}[1]{\ensuremath{\tuple{\policystates, \mealyaction{#1}, \mealyupdate{#1}, \policystate_{I}}}}

\newcommand{\valuessymbol}[2]{\ensuremath{{V}_{#1}^{#2}}}
\newcommand{\values}[3]{\ensuremath{\valuessymbol{#1}{#2}\fun{#3}}}

\newcommand{\stationary}[1]{\ensuremath{\xi_{#1}}}
\newcommand{\resetstate}{\ensuremath{\state_{\mathsf{reset}}}}
\newcommand{\sreset}{\ensuremath{\state_{\mathsf{reset}}}}

\newcommand{\pomdp}{\ensuremath{\mathcal{P}}}
\newcommand{\observations}{\ensuremath{\Omega}}
\newcommand{\observationfn}{\ensuremath{\mathcal{O}}}
\newcommand{\observation}{\ensuremath{o}}
\newcommand{\pomdptuple}{\ensuremath{\tuple{\mdp, \observations, \observationfn}}}
\newcommand{\historytuple}[3]{\ensuremath{\tuple{#1_{\scriptscriptstyle 0:#3-1}, #2_{\scriptscriptstyle 1: #3}}}}
\newcommand{\histories}{\ensuremath{\fun{\actions \cdot \observations}^{*}}}
\newcommand{\historydistribution}{\ensuremath{\mathcal{H}}}

\newcommand{\history}{\ensuremath{h}}
\newcommand{\hreset}{\ensuremath{h}_{\textsf{reset}}}
\newcommand{\belief}{\ensuremath{b}}
\newcommand{\beliefs}{\ensuremath{\mathcal{B}}}
\newcommand{\beliefupdate}{\ensuremath{\tau}}
\newcommand{\beliefmdp}{\ensuremath{\mdp_{\scriptscriptstyle \beliefs}}}
\newcommand{\binit}{\ensuremath{\belief_I}}
\newcommand{\beliefprobtransitions}{\ensuremath{\probtransitions_{\scriptscriptstyle \beliefs}}}
\newcommand{\beliefrewards}{\ensuremath{\rewards_{\scriptscriptstyle \beliefs}}}
\newcommand{\beliefmdptuple}{\ensuremath{\langle\beliefs, \actions, \beliefprobtransitions, \beliefrewards, \binit,\allowbreak \discount \rangle}}
\newcommand{\augmentedpomdp}{\ensuremath{\pomdp^{\scriptscriptstyle\uparrow}}}
\newcommand{\augmentedobservationfn}{\ensuremath{\overline{\observationfn}_{\mu}}}
\newcommand{\augmentedobservationfnoriginal}{\ensuremath{\observationfn^{\scriptscriptstyle\uparrow}}}
\newcommand{\pomdpbisim}{\ensuremath{\rightrightarrows^{\scriptscriptstyle \beliefupdate}_{\scriptscriptstyle \beliefencoder}}}

\newcommand{\decodersymbol}{\ensuremath{\psi}}
\newcommand{\decoder}{\ensuremath{\decodersymbol_{\decoderparameter}}}

 \newcommand{\encoderparameter}{\ensuremath{}}
\newcommand{\decoderparameter}{\ensuremath{}}

\newcommand{\discount}{\ensuremath{\gamma}}

\newcommand{\Prob}{\ensuremath{\displaystyle \mathbb{P}}}
\newcommand{\measurableset}{\ensuremath{\mathcal{X}}}
\newcommand{\varmeasurableset}{\ensuremath{\mathcal{Y}}}
\newcommand{\borel}[1]{\ensuremath{\Sigma\fun{#1}}}

\newcommand{\sampledot}{\ensuremath{{\cdotp}}}
\newcommand{\expectedsymbol}[1]{\ensuremath{\mathop{\mathbb{E}}\ifthenelse{\equal{#1}{}}{}{_{#1}}}}
\newcommand{\expected}[2]{\ensuremath{\expectedsymbol{#1} \left[ #2 \right]}}

\newcommand{\divergencesymbol}{\ensuremath{D}}

\newcommand{\dklsymbol}{\ensuremath{\divergencesymbol_{{\mathrm{KL}}}}}
\newcommand{\dkl}[2]{\ensuremath{\dklsymbol\fun{#1 \parallel #2}}}
\newcommand{\dtvsymbol}{\ensuremath{d_{{TV}}}}
\newcommand{\dtv}[2]{\ensuremath{\dtvsymbol\fun{#1, #2}}}

\newcommand{\normal}[3]{\ensuremath{\displaystyle \ifthenelse{\equal{#3}{}}{\mathcal{N}(#1, #2)}{\mathcal{N}(#3\,;\, #1, #2)}}}
\newcommand{\distributionssymbol}{\ensuremath{\Delta}}
\newcommand{\distributions}[1]{\ensuremath{\distributionssymbol\fun{#1}}}
\newcommand{\support}[1]{\ensuremath{\text{supp}\fun{#1}}}

\newcommand{\temperature}{\ensuremath{\lambda}}

\newcommand{\wassersteinsymbol}[1]{\ensuremath{\mathcal{W}}_{#1}}
\newcommand{\wassersteindist}[3]{\ensuremath{\wassersteinsymbol{#1}\left( #2, #3 \right)}}
\newcommand{\distance}{\ensuremath{d}}
\newcommand{\latentdistance}{\ensuremath{\bar{\distance}}}

\newcommand{\coupling}{\ensuremath{\lambda}}

\newcommand{\Lipschf}[1]{\ensuremath{\mathcal{F}_{#1}}}

\newcommand{\overbar}[1]{\mkern 1.5mu\overline{\mkern-1.5mu#1\mkern-1.5mu}\mkern 1.5mu}
\newcommand{\overbarit}[1]{\,\overline{\!{#1}}}
\newcommand{\embed}{\ensuremath{\phi}}

\newcommand{\latentmdp}{\ensuremath{\overbarit{\mdp}}}
\newcommand{\latentprobtransitions}{\ensuremath{\overbar{\probtransitions}}}
\newcommand{\latentstates}{\ensuremath{\overbarit{\mathcal{\states}}}}
\newcommand{\latentrewards}{\ensuremath{\overbarit{\rewards}}}

\newcommand{\latentmdptuple}{\ensuremath{\langle{\latentstates, \actions, \latentprobtransitions, \latentrewards, \zinit, \discount \rangle}}}
\newcommand{\latentobservationfn}{\ensuremath{\overbar{\observationfn}}}
\newcommand{\latentobservation}{\latentobservationfn}
\newcommand{\latentstate}{\ensuremath{\overbarit{\state}}}
\newcommand{\zinit}{\ensuremath{\latentstate_I}}

\newcommand{\latentaction}{\ensuremath{\overbarit{\action}}}
\newcommand{\latentvaluessymbol}[2]{\overbar{\ensuremath{{V}}}_{#1}^{#2}}
\newcommand{\latentvalues}[3]{\ensuremath{\latentvaluessymbol{#1}{#2}\fun{#3}}}

\newcommand{\latentpolicy}{\ensuremath{\overbar{\policy}}}

\newcommand{\latentbeliefupdate}{\ensuremath{\overbar{\tau}}}
\newcommand{\latentbelief}{\ensuremath{\overbar{\belief}}}
\newcommand{\latentbeliefs}{\ensuremath{\overbar{\beliefs}}}
\newcommand{\latentpomdp}{\ensuremath{\overbar{\pomdp}}}
\newcommand{\beliefencoder}{\ensuremath{\varphi_{\encoderparameter}}}
\newcommand{\localtransitionloss}[1]{L_{\probtransitions}}
\newcommand{\localrewardloss}[1]{L_{\rewards}}
\newcommand{\observationloss}[1]{\ensuremath{L_{\observationfn}}}
\newcommand{\beliefloss}[1]{\ensuremath{L_{\latentbeliefupdate}}}
\newcommand{\onpolicyrewardloss}[1]{\ensuremath{L_{\latentrewards}^{\varphi}}}
\newcommand{\onpolicytransitionloss}[1]{\ensuremath{L_{\latentprobtransitions}^{\varphi}}}
\newcommand{\KV}{\ensuremath{K_{\latentvaluessymbol{}{}}}}
\newcommand{\KR}[1]{\ensuremath{\ifthenelse{\equal{#1}{}}{K_{\latentrewards}}{K_{\latentrewards}^{#1}}}}
\newcommand{\KP}[1]{\ensuremath{\ifthenelse{\equal{#1}{}}{K_{\latentprobtransitions}}{K_{\latentprobtransitions}^{#1}}}}
\newcommand{\Rmax}{\ensuremath{\latentrewards^{\star}}}

\newcommand{\norm}[1]{\ensuremath{\left\| #1 \right\|}}
\newcommand{\gradient}{\ensuremath{\nabla}}

\newcommand{\originaltolatentstationary}[1]{{\latentprobtransitions_{\embed_{\encoderparameter}\stationary{\ifthenelse{\equal{#1}{}}{\policy}{#1}}}}}

\newcommand{\replaybuffer}{\ensuremath{\mathcal{D}}}

\def\1{\bm{1}}

\DeclareMathAlphabet{\mathsfit}{\encodingdefault}{\sfdefault}{m}{sl}
\SetMathAlphabet{\mathsfit}{bold}{\encodingdefault}{\sfdefault}{bx}{n}

\newcommand{\R}{\mathbb{R}}

\newcommand{\abs}[1]{\ensuremath{\left| #1 \right|}}

\begin{document}
\maketitle

\begin{abstract}
    Partially Observable Markov Decision Processes (POMDPs) are used to model environments where the full state cannot be perceived by an agent. As such the agent needs to reason taking into account the past observations and actions. However, simply remembering the full history is generally intractable due to the exponential growth in the history space. Maintaining a probability distribution that models the belief over what the true state is can be used as a sufficient statistic of the history, but its computation requires access to the model of the environment and is often intractable.
    While SOTA algorithms use Recurrent Neural Networks to compress the observation-action history aiming to learn a sufficient statistic, they lack guarantees of success and can lead to sub-optimal policies. To overcome this, we propose the Wasserstein Belief Updater, an RL algorithm that learns a latent model of the POMDP and an approximation of the belief update. Our approach comes with theoretical guarantees on the quality of our approximation ensuring that our outputted beliefs allow for learning the optimal value function.
   \end{abstract}

\section{Introduction}

\emph{Partially Observable Markov Decision Processes} (POMDPs) %
define a powerful framework for modeling decision-making in uncertain environments where the state is not fully observable. These problems are common in many real-world applications, such as robotics \citep{LauriPOMDPRobo}, and recommendation systems \citep{wu2021partially}.
In contrast to \emph{Markov Decision Processes} (MDPs), in a POMDP the agent perceives an imperfect observation of the state that does not suffice as conditioning signal for an optimal policy. As such, optimal policies must take the entire interaction history into account. 
As the space of possible histories scales exponentially in the length of the episode, using histories to condition policies is generally intractable. 
An alternative 
is the notion of \emph{belief}, which is defined as a probability distribution over states based on the agent history.
Beliefs are a sufficient statistic of the history \citep{kaelbling1998planning} but %
the computation of their closed-form expression require the access to a model of the environment and is in general intractable, as it requires to integrate over the full state space, which is thus only applicable to small problems.

To overcome those challenges, SOTA algorithms focus on compressing the history into a fixed-size vector  with the help of \emph{Recurrent Neural Networks} (RNNs) \citep{Hausknecht2015DeepMDPs}.
However, compressing the history using RNNs can lead to information loss, resulting in suboptimal policies. To improve the likelihood of obtaining a sufficient statistic, RNNs can be combined with regularization techniques, including generative models \citep{chen2022flow, Hafner2019DreamImagination, hafner2021mastering}, particle filtering \citep{Igl2018DeepPOMDPs, ma2020particle}, and predicting distant observations \citep{Gregor2018TemporalAuto-Encoder, Gregor2019ShapingRL}.
It is worth noting that \emph{none of these techniques guarantee that the representation of histories induced by RNNs is suitable for optimizing the return.}
Additionally, a limitation of many algorithms is their assumption that beliefs are simple distributions (e.g., Gaussian distributions) which limits their applicability  \citep{Gregor2018TemporalAuto-Encoder, lee2020stochastic, hafner2021mastering}.

In this paper, we propose \emph{Wasserstein Belief Updater} (WBU), a model-based reinforcement learning (RL) algorithm for POMDPs that allows learning the belief space over the unobservable states.
Specifically, WBU learns an approximation of the belief update rule through a 
latent space model whose behaviors (expressed as expected returns) are close to those of the original environment.
Furthermore, we show that WBU is guaranteed to induce a suitable representation of the history to optimize the return.
WBU consists of three components that are learned in a round-robin fashion: the model, the belief learner, and the policy (Fig.~\ref{fig:big-picture}). 
Harnessing only histories to learn a model whose dynamics can be provably linked to the original unobservable environment poses a considerable challenge.
Therefore, in the same spirit as the \emph{Centralized Training with Decentralized Execution} paradigm in \emph{multi-agent} RL (MARL) \citep{Oliehoek2008OptimalPOMDPs, Avalos2022LocalLearning}, where leveraging additional information such as the true state of the environment is a common practice,
\emph{we assume that the POMDP states can be accessed during training}.
While this might seem restrictive at first sight, this assumption is typically met in simulation-based training and can be applied in real-world settings such as robotics, where extra sensors can be used during training in a laboratory setting.

Our core contribution is the \emph{development of a sound framework equipped with theoretical guarantees in the context of RL within partial observability}.
While SOTA algorithms primarily concentrate on enhancing the overall return --- potentially resulting in substantial performance gains --- we contend that performance is not the exclusive goal and that possessing guarantees is equally important, as the balance between these two aspects varies based on the specific application.
By tackling POMDPs with a formal approach, we offer theoretical guarantees that other methods cannot provide: we ensure that \emph{our latent model is able to replicate the dynamics of the original, partially observable environment}, which further \emph{yields a belief representation suitable for learning the value function}.

We learn the latent model of the POMDP via a \emph{Wasserstein auto-encoded MDP} (WAE-MDP, \citealt{aaai-2023-delgrange-nowe-perez}), which embeds bisimulation metrics --- intuitively leading to our guarantees.
In parallel, we maintain a belief distribution over the latent state space via a \emph{belief update network}: we minimize its \emph{Wasserstein distance} %
to the exact belief update rule, through a tractable variational proxy.
To allow for complex belief distributions, we use \emph{normalizing flows} \citep{Kobyzev_2021}. 
In contrast to SOTA algorithms, the beliefs are only optimized towards accurately replicating its update rule. While we call recursively the belief network to maintain the belief distribution, we do not back-propagate through time and thus implement it as a simple feed forward network.
The policy is learned on the latent belief space using a vector integrating the parameters of the belief distribution.
Our experimental results are promising and show the ability of our algorithm to learn to encode the history into a representation 
useful to learn a policy, 
\emph{without using RNNs}.

\smallparagraph{Other related work.}~%
DVRL \citep{Igl2018DeepPOMDPs} extends A2C \citep{Mnih2016AsynchronousLearning} combined with RNNs ({R-A2C}) with auxiliary losses aiming to learn beliefs via a variational autoencoder and particle filtering, but it lacks guarantees.
DVRL further assumes independent normal distributions for beliefs, limiting its applicability.
FORBES \citep{chen2022flow} use normalizing flows but learn  policies conditioned on latent states, which is suboptimal as the state distribution is approximated with a single sample. 
Some works focus on specific POMDP types, like compact image representations (e.g., visual motor tasks, e.g., \citealt{lee2020stochastic}) or states masked with Gaussian noise \citep{wang2021deep}.
While accessing states is common in MARL, it is less common in single-agent but has been explored in kernel-POMDPs \citep{Nishiyama2012HilbertPOMDPs}, using states for RKHS-based models. 
Leveraging additional information available during training  (not necessarily the states) has also been explored by \citet{ORBi-09ced41c-216e-4981-acec-fee81f7d0de5}, but RNNs remain crucial while no abstraction nor representation quality guarantee are provided.
Finally, other works \citep{DBLP:conf/icml/GeladaKBNB19,aaai-2023-delgrange-nowe-perez} study similar value difference bounds to ours and connect them to bisimulation theory \citep{DBLP:conf/popl/LarsenS89,DBLP:journals/ai/GivanDG03}, but in fully observable environments. 

\begin{figure}
    \centering
    \includegraphics[width=.85\linewidth]{img/big_picture_vertical_c.pdf}
    \caption{
    \emph{WBU  framework}.
    The {\color{darkgreen}WAE-MDP} is presented in Sect.~\ref{sec:learning-dynamics}, and {\color{blue}WBU} in Sect.~\ref{sec:latent-belief-learner}.
    Learning the different components is done in a round-robin fashion.
    The {\color{darkgreen}WAE-MDP} learns from data collected by the RL agent and stored in a Replay Buffer. {\color{blue}WBU} uses the transition function {\color{darkcyan}$\latentprobtransitions_{\decoderparameter}$} and observation decoder {\color{darkred}$\latentobservationfn_{\decoderparameter}$} of the {\color{darkgreen}WAE-MDP} to learn to approximate the belief update rule. The agent learns a policy conditioned on the resulting \emph{sub-belief} $\beta_t$ (i.e., the parameters of the latent belief $\latentbelief_t$).%
    }
    \label{fig:big-picture}
\end{figure}

\section{Background}
\subsection{Probability Distributions and Discrepancy Measures}
We write $\borel{\measurableset}$ for the set of all Borel subsets of a complete, separable space $\measurableset$,
$\distributions{\measurableset}$ for the set of measures on $\measurableset$, and
$\diracimpulsesymbol_a \in \distributions{\measurableset}$ for the \emph{Dirac measure} with impulse $a \in \measurableset$.
Let $P, Q \in \distributions{\measurableset}$, the divergence between $P$ and $Q$ can be measured according to the following discrepancies:
\begin{itemize}
    \item the solution of the \emph{optimal transport} problem (OT), defined as $\wassersteinsymbol{c}\,(P, Q) = \inf_{\coupling%
    } \expectedsymbol{x, y \sim \coupling}c(x, y),$ which
    is the \emph{minimum cost of changing $P$ into $Q$} \citep{Villani2009},
    where $c \colon \measurableset \times \measurableset \to \mathopen[ 0, \infty \mathclose)$
    is a cost function and %
    the infimum is taken over the set of all \emph{couplings} of $P$ and $Q$.
    When $c$ is equal to a distance metric $\distance$ over $\measurableset$,  $\wassersteinsymbol{\distance}$ is the \emph{Wasserstein distance} between the two distributions.
    \item the \emph{Kullback-Leibler} (KL) divergence, defined as
    $\dklsymbol\fun{P, Q} = \expectedsymbol{x \sim P}\left[\log\fun{\nicefrac{P\fun{x}}{ Q\fun{x}}}\right].$
    \item the \emph{total variation distance} (TV), defined as
    $\dtv{P}{Q} = \sup_{A \in \borel{\measurableset}} |P\fun{A} - Q\fun{A}|$.
    If $\measurableset$
    is equipped with 
    the discrete metric $\condition{\neq}$, TV coincides with the Wasserstein metric.
\end{itemize}

\subsection{Decision Making under Uncertainty}

\smallparagraph{Markov Decision Processes}~
(MDPs) are tuples $\mdp = \mdptuple$ where $\states$ is a set of \emph{states}; $\actions$, a set of \emph{actions}; $\probtransitions \colon \states \times \actions \to \distributions{\states}$, a \emph{probability transition function} that maps the current state and action to a \emph{distribution} over the next states; $\rewards \colon \states \times \actions \to \R$, a \emph{reward function}; $\sinit \in \states$, the \emph{initial state}; and 
$\discount \in \mathopen[0,1\mathclose)$ a discount factor.
We refer to MDPs with continuous state or action spaces as \emph{continuous MDPs}. 
In that case, we assume $\states$ and $\actions$ are complete separable metric spaces equipped with a Borel $\sigma$-algebra.
An agent interacting in $\mdp$ produces \emph{trajectories}, i.e.,
sequences of states and actions $\trajectorytuple{\state}{\action}{T}$ %
where $\state_0 = \sinit$ and $\state_{t + 1} \sim \probtransitions\fun{\sampledot \mid \state_t, \action_t}$ for $t < T$.

\smallparagraph{Policies and probability measure.}~A \emph{stationary policy} $\policy \colon \states \to \distributions{\actions}$ 
prescribes which action to choose at each step of the interaction.
A policy $\policy$ and $\mdp$ %
induce a unique probability measure $\Prob^\mdp_\policy$ on the Borel $\sigma$-algebra over (measurable) infinite trajectories~\citep{DBLP:books/wi/Puterman94}.
The typical goal of an RL agent is to learn a policy that maximizes the \emph{expected return}, given by $\expectedsymbol{\policy}^{\mdp}\left[{\sum_{t=0}^{\infty}\discount^{t} \cdot \rewards\fun{\state_t, \action_t}}\right]$, by interacting with $\mdp$.
We omit the superscript when the context is clear.

\smallparagraph{Partially Observable MDPs}~(POMDPs) %
are tuples $\pomdp = \pomdptuple$ where $\mdp$ is an MDP with state space $\states$ and action space $\actions$; $\observations$ is a set of \emph{observations}; and $\observationfn \colon \states \times \actions \to \distributions{\observations}$ is an \emph{observation function} that defines the distribution of observations that may occur when the MDP $\mdp$ transitions to a state upon the execution of a particular action.
An agent in $\pomdp$ actually interacts in $\mdp$, but \emph{without directly observing the states} of $\mdp$:
instead, the agent perceives observations, which yields \emph{histories}, i.e., sequences of actions and observations $ \historytuple{\action}{\observation}{T}$
that can be associated to an (unobservable) trajectory $\trajectorytuple{\state}{\action}{T}$  in $\mdp$, where $o_{t + 1} \sim \observationfn\fun{\sampledot \mid \state_{t + 1}, \action_t}$ for all $t < T$.

\smallparagraph{Beliefs.}~%
Unlike in MDPs, stationary policies that are based solely on the current observation of $\pomdp$ \emph{do not induce any probability space} on trajectories of $\mdp$.
Intuitively, due to the partial observability of the current state $\state_t \in \states$ at each interaction step $t \geq 0$,  
the agent must take into account full histories in order to infer the distribution of rewards accumulated up to the current time step $t$, and make an informed decision on its next action $\action_t \in \actions$.
Alternatively, the agent can maintain a \emph{belief} $\belief_t \in \distributions{\states} = \beliefs$ over the current state of $\mdp$ \citep{ASTROM1965174}.
Given the next observation $\observation_{t +1}$, the %
next belief $\belief_{t + 1}$ is computed according to the \emph{belief update function} $\beliefupdate \colon \beliefs \times \actions \times \observations \to \beliefs$, where $\beliefupdate\fun{\belief_{t}, \action_{t}, \observation_{t + 1}} = \belief_{t + 1}$ iff 
 the belief over any next state $\state_{t + 1} \in \states$ has for density
\begin{align}
\label{eq:belief_update}
    \belief_{t+1}\fun{\state_{t+1}} = 
    \frac{\expectedsymbol{\state_{t} \sim \belief_t}{ \probtransitions\fun{\state_{t+1} \mid \state_t, \action_t}} \cdot \observationfn\fun{\observation_{t+1} \mid \state_{t+1}, \action_t} }{\E_{\state_{t} \sim \belief_t} \E_{\state' \sim \probtransitions(\cdot \mid \state_t, \action_t)} \observationfn(\observation_{t+1} \mid \state', \action_t) }.
\end{align}

Each belief $\belief_{t+1}$ constructed this way is a \emph{sufficient statistic} for the history $\tuple{\action_{\scriptscriptstyle{0}: t}, \observation_{\scriptscriptstyle{1}: t + 1}}$ to optimize the return \citep{kaelbling1998planning}. %
We write
$\beliefupdate^*\fun{\indexof{\action}{0:t}, \indexof{\observation}{1:t+1}} = \beliefupdate\fun{\sampledot\,, \action_t, \observation_{t + 1}} \circ \cdots
\circ \beliefupdate\fun{\diracimpulsesymbol_{\sinit}, \action_0, \observation_1} = \belief_{t + 1}$
for the recursive application of $\beliefupdate$ along the history.
The belief update rule derived from $\beliefupdate$ allows to formulate $\pomdp$ as a continuous %
\emph{belief MDP} $\beliefmdp = \beliefmdptuple$, where 
$
\beliefprobtransitions\fun{\belief' \mid \belief, \action} = \expectedsymbol{\state \sim \belief} \expectedsymbol{\state' \sim \probtransitions\fun{\sampledot \mid \state, \action}} \expectedsymbol{\observation' \sim \observationfn\fun{\sampledot \mid \state', \action}} \diracimpulse{\belief'}{\beliefupdate\fun{\belief, \action, \observation'}}
$
$\beliefrewards\fun{\belief, \action} = \expectedsymbol{\state \sim \belief} \rewards\fun{\state, \action}$; and $\binit = \diracimpulsesymbol_{\sinit}$.
As for all MDPs, $\beliefmdp$ %
and any stationary policy for $\beliefmdp$ (thus conditioned on beliefs) induce a well-defined probability space over trajectories of $\beliefmdp$, which allows optimizing the expected return in $\pomdp$. %

\subsection{Latent Space Modeling}\label{sec:latent-space-modeling}

\smallparagraph{Latent MDPs.}~%
Given the original (continuous or very large, possibly unknown) environment $\mdp$, a \emph{latent space model} is another (tractable, explicit) MDP $\latentmdp = \latentmdptuple$ with state space linked to the original one via a \emph{state embedding function}:  $\embed\colon \states \to \latentstates$.

\smallparagraph{Wasserstein Auto-encoded MDPs}~%
(WAE-MDPs, \citealt{delgrange2023wasserstein}) are latent space models that are trained based on the OT %
from trajectories resulting from the execution of the RL agent policy in the real environment $\mdp$, to that reconstructed from the latent model $\latentmdp_{\decoderparameter}$.
The optimization process relies on a temperature
 $\temperature \in \mathopen(0, 1 \mathclose)$ that controls the continuity of the latent space learned, the zero-temperature limit corresponding to a discrete latent state space (see Appendix~\ref{appendix:temperature} for a discussion).
This procedure guarantees $\latentmdp_{\decoderparameter}$ to be probably approximately \emph{bisimilarly close} \citep{DBLP:conf/popl/LarsenS89,DBLP:journals/ai/GivanDG03,aaai-2023-delgrange-nowe-perez} to $\mdp$ as $\temperature \to 0$: 
in a nutshell, \emph{bisimulation metrics} imply the closeness of the two models in terms of probability measures and expected return \citep{DBLP:journals/tcs/DesharnaisGJP04,DBLP:journals/siamcomp/FernsPP11}.
Specifically, 
a WAE-MDP learns the following components:
\begin{align}
    &\text{a \emph{state embedding function}} && \embed_{\encoderparameter} \colon \states \to \latentstates && 
    \text{a \emph{latent transition function}} && \latentprobtransitions_{\decoderparameter} \colon \latentstates \times \actions \to \distributionssymbol(\latentstates)  \notag \\
    &\text{a \emph{latent reward function}} && \latentrewards_{\decoderparameter} \colon \latentstates \times \actions \to \R &&
    \text{a \emph{state decoder}} &&\decoder \colon \latentstates \to \states. \label{eq:wae-mdp-components}
\end{align}

\section{Learning the Dynamics}\label{sec:learning-dynamics}
The agent is assumed to operate within a POMDP.
In an RL setting, the former have no explicit access to the environment dynamics: instead, it reinforces its behaviors through interactions and experiences without directly accessing the  transition, reward, and observation functions of the environment.
To provide the aforementioned guarantees, we henceforth adhere the following assumption.
\begin{assumption}[Access to the state during training]\label{assumption:access-state}
    In addition to the observation, the agent is able to observe to the true state of the environment, {but only during the training phase}.
\end{assumption}
\begin{remark}
Seemingly restrictive, Assumption~\ref{assumption:access-state} can actually be met in a broad range of training scenarios, in particular those relying on simulators, where one could merely consider the RAM as the simulator state.
Otherwise, additional sensors with higher fidelity could be considered to obtain the state.
There are several scenarios where the state is accessible during training but not during execution, e.g., where the state is too large for real-time processing, as in low-power hardware, and in \textsc{SimToReal} frameworks, as when noise is injected to learn robust real-world policies \citep{TomStaessensPhD}.
Other applicable scenarios are model-based design or model-predictive control, where a model is accessible during training, and situations where accessing the state is costly \citep{DBLP:conf/atva/BulychevCDLRR12}, e.g., in embedded systems, or when shipping a product with sensors of low reliability to reduce the production cost.
\end{remark}

Concretely, when the RL agent interacts in a POMDP $\pomdp = \pomdptuple$ with underlying MDP $\mdp = \mdptuple$, \emph{we leverage this access to allow the agent to learn the dynamics of the environment}, i.e., those of $\mdp$, as well as those related to the observation function $\observationfn$.
To do so, we  learn an internal, explicit representation of the experiences gathered, through a latent space model.
\emph{We then use this model as a teacher for the agent to make it learn how to perform its belief updates}.
Hence, acquiring an accurate environment model is crucial to learn a reliable belief update function.
In Sect.~\ref{sec:guarantees}, we further demonstrate that the resulting model is guaranteed to closely replicate the original environment behaviors.
The trick we use to learn such a model is to reason on an equivalent POMDP, where the underlying MDP is refined to encode all the crucial dynamics.

\subsection{The Latent POMDP Encoding}
\label{sec:refinement-model}

We enable learning the dynamics of $\pomdp$ via a WAE-MDP by considering the %
POMDP $\augmentedpomdp = \tuple{\mdp_{\observations}, \observations, \augmentedobservationfnoriginal}$, where
    (i) the state space of the underlying MDP is refined to encode the observations: $\mdp_{\observations} = \langle \states_{\observations}, \allowbreak \actions, \probtransitions_{\observations}, \rewards_{\observations} \allowbreak, \tuple{\sinit, \observation_I}, \allowbreak \discount\rangle$
with $\states_{\observations} = \states \times \observations$,
$\probtransitions_{\observations}\fun{\state', \observation' \mid \state, \observation, \action} = \probtransitions\fun{\state' \mid \state, \action} \cdot \observationfn\fun{\observation' \mid \state', \action}$,
$\rewards_{\observations}\fun{\tuple{\state, \observation}, \action} = \rewards\fun{\state, \action}$, and $\observation_I$ is an observation from $\observations$ linked to the initial state $\sinit$;
    (ii) the observation function $\augmentedobservationfnoriginal \colon \states_\observations \to \observations$ is the \emph{deterministic} projection of the refined state on the observation space, with $\augmentedobservationfnoriginal\fun{\tuple{\state, \observation}} = \observation$.
The POMDPs $\pomdp$ and $\augmentedpomdp$ are equivalent \citep{DBLP:journals/ai/ChatterjeeCGK16}: $\augmentedpomdp$ captures the stochasticity of $\observationfn$ in its transition function through the refinement of the state space, further yielding a deterministic observation function, only dependent on refined states.

Henceforth, the goal is to learn a latent space model $\latentmdp_{\decoderparameter} = \langle \latentstates, \actions, \latentprobtransitions_{\decoderparameter}, \latentrewards_{\decoderparameter}, \zinit, \discount \rangle$ linked to $\mdp_{\observations}$ via the embedding $\embed_{\encoderparameter} \colon \states_{\observations} \to \latentstates$, and we achieve this via the WAE-MDP framework.
Not only does the latter allow learning the observation dynamics through $\latentprobtransitions_{\decoderparameter}$, but it also enables to learn the deterministic observation function 
$\augmentedobservationfnoriginal$
through the use of the state decoder $\decoder$, by decomposing the latter in two networks $\decoder^{\states} \colon \latentstates \to \states$ and $\augmentedobservationfn \colon \latentstates \to \observations$, which yield $\decoder\fun{\latentstate} = \langle\decoder^{\states}(\latentstate), \augmentedobservationfn\fun{\latentstate}\rangle$.
This way, the WAE-MDP learns
all the components of $\augmentedpomdp$, the latter being equivalent to $\pomdp$. 
With this model, we construct a \emph{latent POMDP} $\latentpomdp_{\decoderparameter} = \langle\latentmdp_{\decoderparameter}, \Omega, \latentobservationfn_{\decoderparameter}\rangle$, where the observation function outputs a normal distribution centered in $\augmentedobservationfn$: $\latentobservationfn_{\decoderparameter}\fun{\sampledot \mid \latentstate} = \normal{\augmentedobservationfn\fun{\latentstate}}{\sigma^2}{}$.
Note that the deterministic function is retrieved as the variance approaches zero. 
However, it is worth mentioning that the smoothness of $\latentobservation_{\decoderparameter}$ is favorable for learning belief distributions, in contrast to Dirac measures (see Eq.~\ref{eq:latent-belief-update} and Rermark~\ref{rmk:variance} below).
As with any POMDP, the belief update function $\latentbeliefupdate$ of $\latentpomdp_{\decoderparameter}$ allows to reason on the belief space to optimize the expected return.
The belief update procedure is illustrated in Appendix~\ref{appendix:belief-update}.
Formally,
assuming the latent belief at time step $t \geq 0$ is $\latentbelief_t \in \Delta({\latentstates}) = \latentbeliefs$,
$\action_t$ is executed, and then $\observation_{t + 1}$ observed, $\latentbelief_t$ is updated according to $\latentbeliefupdate({\latentbelief_t, \action_t, \observation_{t + 1}}) = \latentbelief_{t + 1}$ iff, for any (unobservable) next state $\latentstate_{t + 1} \in \latentstates$,
\begin{equation}
    \latentbelief_{t + 1}\fun{\latentstate_{t + 1}} = \frac{\expectedsymbol{\latentstate_t \sim \latentbelief_t} \latentprobtransitions_{\decoderparameter}\fun{\latentstate_{t + 1} \mid \latentstate_t, \action_t} \cdot \latentobservationfn_{\decoderparameter}\fun{\observation_{t + 1}\mid\latentstate_{t + 1}}}{\E_{\latentstate_{t} \sim \latentbelief_t} \E_{\latentstate' \sim \latentprobtransitions_{\decoderparameter}(\cdot \mid \latentstate_t, \latentaction_t)} \latentobservation_{\decoderparameter}\fun{\observation_{t + 1} \mid \latentstate'}}. \label{eq:latent-belief-update}
\end{equation}
\smallparagraph{Latent policies.}~%
Given \emph{any} history $\history \in \histories$, running a latent policy 
$\latentpolicy \colon \latentbeliefs \to \distributions{\actions}$
in $\pomdp$ is possible by converting $\history$
into a latent belief $\latentbeliefupdate^*\fun{\history} = \latentbelief$ and executing the action prescribed by $\latentpolicy({\sampledot \mid \latentbelief})$.
Training $\latentmdp_{\decoderparameter}$ grants access to the dynamics required to update the belief through its closed form (Eq.~\ref{eq:latent-belief-update}).
However, integrating over the full latent space remains computationally intractable. 
\begin{center}
\fbox{
\begin{minipage}{0.95\linewidth}
{As a solution, we propose to leverage the access to the dynamics of $\latentmdp_{\decoderparameter}$ to learn a \emph{latent belief encoder} $\beliefencoder \colon \latentbeliefs \times \actions \times \latentstates \to \latentbeliefs$ that approximates the belief update function} by minimizing
    $
    D\fun{\latentbeliefupdate^*({\history}), \beliefencoder^*({\history})}
    $
for \emph{some} discrepancy $D$ and $\history \in \histories$ drawn from \emph{some} distribution.
The belief encoder $\beliefencoder$ thus enables to learn a policy $\latentpolicy$ conditioned on latent beliefs to optimize the return in $\pomdp$: \emph{given the current history $\history$, the next action to play is given by $\action \sim \latentpolicy\fun{\sampledot \mid \beliefencoder^*\fun{\history}}$.}
\end{minipage}
}
\end{center}

Two main questions arise:
\emph{``Does the latent POMDP induced by our WAE-MDP encoding yields a model whose behaviors are close to $\pomdp$?"} and \emph{``Is the history representation induced by $\beliefencoder$ suitable to optimize the expected return in $\pomdp$?"}.
Clearly, the obtained guarantees depend on the history distribution and chosen discrepancy. The following section provides a detailed theoretical analysis of the required distribution and losses to achieve these learning guarantees.

\subsection{Losses and Theoretical Guarantees} \label{sec:guarantees}

To yield the guarantees, we specifically target the \emph{episodic RL process} setting for drawing histories.
\begin{assumption}[Episodic RL process]\label{assumption:episodic}
    The environment $\pomdp$ embeds a special \emph{reset state} so that (i) under any policy, the environment is almost surely eventually reset;
(ii) when reset, the environment transitions to the initial state;
and (iii) {the reset state is observable}.
\end{assumption}

\begin{lemma}\label{lem:stationary-histories}
There is a well defined probability distribution $\historydistribution_{\latentpolicy} \in \distributions{\histories}$ over histories likely to be perceived at the limit by the agent when it %
executes $\latentpolicy$ in $\pomdp$
\emph{(proof in Appendix~\ref{appendix:stationary-histories})}.
\end{lemma}
\smallparagraph{Local losses.}~%
The objective function of the WAE-MDP incorporates \emph{local losses} \citep{DBLP:conf/icml/GeladaKBNB19} that minimize the expected distance between the original and latent reward and transition functions:
\begin{align}
    \localrewardloss{\historydistribution_{\latentpolicy}} = \! \! \! \! \expectedsymbol{\state, \observation, \action \sim \historydistribution_{\latentpolicy}}  \left| \rewards\fun{\state, \action} - \latentrewards_{\decoderparameter}\fun{\embed_{\encoderparameter}\fun{\state, \observation}, \action}\right|, &&
    \localtransitionloss{\historydistribution_{\latentpolicy}} = \! \! \! \!\expectedsymbol{{\state, \observation, \action} \sim \historydistribution_{\latentpolicy}} \wassersteindist{\latentdistance}{\embed_{\encoderparameter}\probtransitions_{\observations}\fun{\sampledot \mid {\state, \observation}, \action}}{\latentprobtransitions_{\decoderparameter}\fun{\sampledot \mid \embed_{\encoderparameter}\fun{\state, \observation}, \action}}; %
    \notag
\end{align}
and both are optimized \emph{locally}, i.e., under $\historydistribution_{\latentpolicy}$, where $\state, \observation, \action \sim \historydistribution_{\latentpolicy}$ is a shorthand for (i) $\history \sim \historydistribution_{\latentpolicy}$ so that $\observation$ is the last observation of $\history$, (ii) $\state \sim \beliefupdate^{*}\fun{\history}$, and (iii) $\action \sim \latentpolicy\fun{\sampledot \mid \beliefencoder^{*}\fun{\history}}$.
Furthermore, $\embed_{\encoderparameter}\probtransitions\fun{\sampledot \mid \state, \action}$ is the distribution of transitioning to $\state' \sim \probtransitions\fun{\sampledot \mid \state, \action}$, then embedding it to the latent space $\latentstate' = \embed_{\encoderparameter}\fun{\state'}$, and $\latentdistance$ is a metric on $\latentstates$.
In practice, the ability of observing states during learning enables the optimization of those local losses without the need of explicitly storing histories.
Instead, we simply store the transitions of $\mdp_{\observations}$ encountered while executing $\latentpolicy$.
We also introduce an \emph{observation loss} in addition to the reconstruction loss of the decoder, which allows learning $\latentobservationfn_{\decoderparameter}$:
\begin{align}
    \observationloss{} = \expectedsymbol{\state, \observation, \action \sim \historydistribution_{\latentpolicy}}
    \expectedsymbol{\state' \sim \probtransitions\fun{\sampledot \mid \state, \action}}
    \dtv{\observationfn\fun{\sampledot \mid \state', \action}}{\expectedsymbol{\observation' \sim \observationfn\fun{\sampledot \mid \state', \action}} \latentobservation_{\decoderparameter}\fun{\sampledot \mid \embed_{\encoderparameter}\fun{\state', \observation'}}}. \label{eq:observation-loss}
\end{align}

\smallparagraph{Belief Losses.}~%
We set $D$ as the Wassertein distance between the true latent belief update and our belief encoder. %
In addition, we argue that the following reward and transition regularizers are required to bound the gap between the fully observable model $\latentmdp_{\decoderparameter}$ and the partially observable one $\latentpomdp_{\decoderparameter}$:
\begin{gather}
    \begin{aligned}
    \beliefloss{} = \expectedsymbol{\history \sim \historydistribution_{\latentpolicy}} \wassersteindist{\latentdistance}{\latentbeliefupdate^{*}\fun{\history}}{\beliefencoder^{*}\fun{\history}}, \,
    &&
    \onpolicyrewardloss{\historydistribution_{\latentpolicy}}= \expectedsymbol{\history, \state, \observation, \action \sim \historydistribution_{\latentpolicy}}\, \expectedsymbol{\latentstate \sim \beliefencoder^*\fun{\history}} \left| \latentrewards_{\decoderparameter}\fun{\embed_{\encoderparameter}\fun{\state, \observation}, \action} - \latentrewards_{\decoderparameter}\fun{\latentstate, \action} \right|,
    \end{aligned} \notag \\
    \onpolicytransitionloss{\historydistribution_{\latentpolicy}} = \expectedsymbol{\history, \state, \observation, \action \sim \historydistribution_{\latentpolicy}} \, \expectedsymbol{\latentstate \sim \beliefencoder^*\fun{\history}}
    \wassersteindist{\latentdistance}{\latentprobtransitions_{\decoderparameter}\fun{\sampledot \mid \embed_{\encoderparameter}\fun{\state, \observation}, \action}}{\latentprobtransitions_{\decoderparameter}\fun{\sampledot \mid \latentstate, \action}} .\label{eq:on-policy-losses}
\end{gather}
$\onpolicyrewardloss{}$ and $\onpolicytransitionloss{}$ aim at regularizing $\beliefencoder$ and minimize the gap between the rewards (resp. transition probabilities) that are expected when drawing states from the current belief compared to those actually observed.
Again, the ability to observe states during training enables optimizing those losses while the states are not required to execute the policy.
The belief loss and the related two regularizers can be optimized \emph{on-policy}, i.e., coupled with the optimization of $\latentpolicy$ that is used to generate the episodes.

\smallparagraph{Value difference bounds.}~%
We provide guarantees
concerning the \emph{agent behaviors in $\pomdp$}, when the policies are \emph{conditioned on latent beliefs}.
To do so, we formalize the behaviors of the agent through \emph{value functions}.
For a specific policy $\policy$, the value of a history is the expected return that would result from continuing to follow the policy from the latest point reached in that history: 
$\values{\policy}{}{\history} = \expected{\policy}{\sum_{t = 0}^{\infty} \discount^t\, \reward_t \mid \binit = \latentbeliefupdate^{*}\fun{\history}}$.
Similarly, we write $\latentvaluessymbol{\latentpolicy}{}$ for the values of the latent policy $\latentpolicy$ in $\latentpomdp_{\decoderparameter}$.

Intuitively, assuming the agent employs a latent policy whose inputs are produced by $\beliefencoder$, we claim that when the losses are minimized to zero, then (i) the latent model almost surely mimics the original environment, and (ii) our belief representation almost surely captures the value function.
\begin{theorem}[Model quality]\label{thm:value-diff-bounds} %
Let $\Rmax = \norm{\latentrewards}_{\infty}$ and 
$\KV= \nicefrac{\Rmax}{1 - \discount}$,
then for any latent policy $\latentpolicy \colon \latentbeliefs \to \distributions{\actions}$, 
the values of $\,\pomdp$ and $\latentpomdp_{\decoderparameter}$ are guaranteed to be bounded by the local and belief losses in average when $\latentpolicy$ is executed in $\pomdp$ via $\action \sim \latentpolicy\fun{\sampledot \mid \beliefencoder^{*}\fun{\history}}$:
\begin{align}
    \expectedsymbol{\history \sim \historydistribution_{\latentpolicy}}\abs{\values{\latentpolicy}{}{\history} - \latentvalues{\latentpolicy}{}{\history}} \leq \frac{\localrewardloss{\historydistribution_{\latentpolicy}} + \onpolicyrewardloss{\historydistribution_{\latentpolicy}} + \Rmax \beliefloss{\historydistribution_{\latentpolicy}} + \discount \KV \cdot \fun{ \localtransitionloss{\historydistribution_{\latentpolicy}} +  \onpolicytransitionloss{\historydistribution_{\latentpolicy}} + \beliefloss{\historydistribution_{\latentpolicy}} + \observationloss{\historydistribution_{\latentpolicy}}}}{1 - \discount}. \label{eq:value-diff-bound-model}
\end{align}
\end{theorem}
\iftrue
\begin{theorem}[Representation quality]\label{thm:value-diff-bound-representation}
Assume the variance of $\latentobservationfn$ goes to zero, let $\latentpolicy^{\star}$ be an optimal policy of  $\latentpomdp_{\decoderparameter}$, 
then for any $\epsilon > 0$, there is a $K \geq 0$ so that for any histories $\history_1, \history_2$ measurable under $\pomdp$ and $\latentpomdp$ with $ \beliefencoder^{*}\fun{\history_1} = \latentbelief_1$ and $\beliefencoder^{*}\fun{\history_2} = \latentbelief_2$,
the representation induced by $\beliefencoder$ almost surely yields:
\begin{multline}
    \abs{\values{\latentpolicy^{\star}}{}{\history_1} - \values{\latentpolicy^{\star}}{}{\history_2}} \leq
    K \wassersteindist{\latentdistance}{\latentbelief_1}{\latentbelief_2} +\, \epsilon \, + \\
    \frac{\localrewardloss{\historydistribution_{\latentpolicy^{\star}}} + \onpolicyrewardloss{\historydistribution_{\latentpolicy^{\star}}} + \fun{K + \discount\KV + \Rmax} \beliefloss{\historydistribution_{\latentpolicy^{\star}}} + \discount \KV \cdot \fun{ \localtransitionloss{\historydistribution_{\latentpolicy^{\star}}} +  \onpolicytransitionloss{\historydistribution_{\latentpolicy^{\star}}} + \observationloss{\historydistribution_{\latentpolicy^{\star}}}}}{1 - \discount} \fun{\frac{1}{\historydistribution_{\latentpolicy^{\star}}\fun{\history_1}} + \frac{1}{\historydistribution_{\latentpolicy^{\star}}\fun{\history_2}}}.
    \label{eq:value-diff-bound-representation} 
\end{multline}
\end{theorem}
While Thm.~\ref{thm:value-diff-bounds} asserts that training a WAE-MDP as a latent space model of the environment results in similar behaviors (i.e., close expected returns) compared to the original environment when they are measured under the agent policy --- which justifies the usage of $\latentpomdp$ as model of the environment --- Thm.~\ref{thm:value-diff-bound-representation} states that our learned update procedure yields a belief representation which is well-suited to optimize the policy: execution traces leading to close latent beliefs (via our learned updater $\beliefencoder$) are guaranteed to yield close expected returns as well (proofs in
Appendix~\ref{appendix:value-diff-bounds}).

\section{Learning to Believe}\label{sec:latent-belief-learner}
In the following, we assume that we have access to the latent model learned by the WAE-MDP.

\smallparagraph{Architecture.}~%
Our latent belief encoder $\beliefencoder$ aims at generalizing to \emph{any} POMDP.
Therefore, \emph{we do not make any assumption about the underlying belief distribution}.
To accommodate complex belief distributions, we use a \emph{Masked Auto-Regressive Flows} (MAF) \citep{NIPS2017_6c1da886}, a type of normalizing flow built on the auto-regressive property.
Precisely, to fit with the WAE-MDP framework and leverage the guarantees presented in Sect.~\ref{sec:guarantees}, we use the MAF of \citet{delgrange2023wasserstein} that learns relaxed multivariate latent distributions.

The \emph{sub-belief} $\beta_t$ is the vector that embeds the parameters of the belief distribution, which is converted into a belief  via the MAF $\mathbb{M}_{\encoderparameter}\fun{\beta_{t}} = \latentbelief_{t}$.
We use a \emph{sub-belief encoder} $\beliefencoder^{\text{sub}}$ to recursively update $\beta_t$ via $\beliefencoder^{\scriptscriptstyle\text{sub}}\fun{\beta_t, \action_t, \observation_{t + 1}} = \beta_{t + 1}$, so that $\beliefencoder\fun{\latentbelief_t, \action_t, \observation_{t + 1}} = \mathbb{M} \circ \beliefencoder^{\scriptscriptstyle\text{sub}}\fun{\beta_t, \action_t, \observation_{t + 1}}$.
RNNs are trained via back-propagation through time (BPTT), which is challenging %
\citep{DBLP:conf/icml/PascanuMB13}.
In contrast, albeit sub-beliefs are updated recursively in the same spirit as RNN hidden states, \emph{we do not need to use BPTT and use a simple feed-forward network for $\beliefencoder^{\text{sub}}$}, as illustrated in Fig.~\ref{fig:a2c}.
In R-A2C, RNN hidden states serve as compact representations of histories for the policy.
Since values of {time-steps closer to the end of an episode are easier to learn}, the gradients of future time-steps tend to be more accurate; thus BPTT helps learning.
This is in stark contrast with learning the belief update rule: the beliefs of {early time-steps are easier to infer}, 
so BPTT is unnecessary and might even be armful.

\smallparagraph{Training.}~%
We aim to train $\beliefencoder^{\text{sub}}$ and $\mathbb{M}$ to approximate the update rule by minimizing the Wasserstein between the belief update rule $\latentbeliefupdate$ of the latent POMDP, and the belief encoder $\beliefencoder$ (Eq.~\ref{eq:on-policy-losses}), to leverage the theoretical learning guarantees of Thm.~\ref{thm:value-diff-bounds} and \ref{thm:value-diff-bound-representation}.
However, Wasserstein optimization is known to be challenging, often requiring the use of additional networks, Lipschitz constraints, and a min-max optimization procedure \citep{DBLP:conf/icml/ArjovskyCB17}, similar to how WAE-MDPs are trained.
Also, sampling from both distributions is necessary for optimizing Wasserstein and, while sampling from our belief approximation is straightforward, sampling from the update rule (Eq.~\ref{eq:latent-belief-update}) is non-trivial.

\begin{figure}
    \centering
    \begin{subfigure}{.50\textwidth}
        \includegraphics[width=\linewidth]{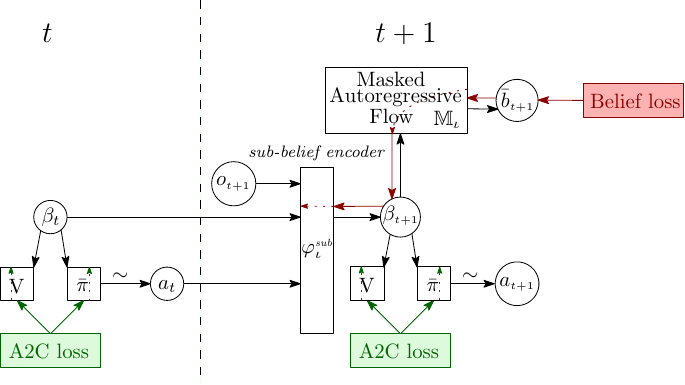}
        \label{fig:belief_a2c}
    \end{subfigure}\hfill
    \begin{subfigure}{.50\textwidth}
        \centering
        \includegraphics[width=\linewidth]{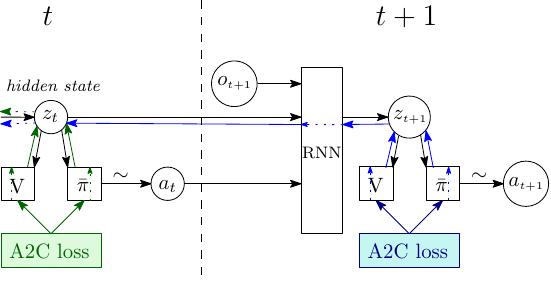}
        \label{fig:rnn_a2c}
    \end{subfigure}
    \caption{
    WBU (\emph{left}) %
    learns to encode the history into a sub-belief solely by optimizing the belief loss.
    The policy, being conditioned on the sub-belief, is learned via A2C and does not back-propagate through the sub-belief encoder.
    The R-A2C agent (\emph{right}) uses BPTT: the RNN leverages  gradients from future time-steps to improve its compression of the history for learning a policy and value function. In both plots, the colored arrows represent the gradient flows of the different losses.
    }\label{fig:a2c}
\end{figure}

As an alternative to the Wasserstein optimization, we minimize the KL divergence between the two distributions. $\dklsymbol$ is easier to optimize and only requires sampling from one of the two distributions (in our case, the belief encoder).
However, unlike the Wasserstein distance, guarantees can only be derived when the divergence approaches zero.
Nonetheless, in the WAE-MDP zero-temperature limit, $\dklsymbol$ bounds Wasserstein by the Pinsker's inequality \citep{borwein2005convex}.

\smallparagraph{On-policy KL divergence.}~Using $\dklsymbol$ as a proxy  for $\wassersteinsymbol{\latentdistance}$ allows to narrow the gap between $\beliefencoder$ and $\latentbeliefupdate$. %
\ifarxiv
\else
We train $\beliefencoder$ \emph{on-policy}, with the same samples as used for $\latentpolicy$, which aids learning despite gradients are not allowed to flow between the networks.
\fi
At any time-step $t \geq 0$, given the current belief $\latentbelief_t$, the action $\action_t$ played by the agent, and the next perceived observation $\observation_{t + 1}$, the belief proxy loss is: %
\vspace{-.01em}
\begin{multline}
    {\dkl{\beliefencoder\fun{\latentbelief_t, \action_t, \observation_{t + 1}}}{
    \latentbeliefupdate\fun{ \latentbelief_t, \action_t, \observation_{t + 1}}}}
    = \log\fun{\expectedsymbol{\latentstate \sim \latentbelief_t} \expectedsymbol{\latentstate' \sim \latentprobtransitions_{\decoderparameter}\fun{\sampledot \mid \latentstate, \action_t}} \latentobservationfn_{\decoderparameter}\fun{\observation_{t + 1} \mid \latentstate'}} +\\
    \!\!
    \expectedsymbol{\latentstate_{t+1} \sim \beliefencoder\fun{ \latentbelief_t, \action_t, \observation_{t + 1}}}\! \left[ \log{\beliefencoder\fun{\latentstate_{t+1} \mid \latentbelief_t, \action_t, \observation_{t + 1}}} \! - 
    \!\log  \!\expectedsymbol{\latentstate \sim \latentbelief_t}\!\latentprobtransitions_{\decoderparameter} \fun{\latentstate_{t+1} \mid \latentstate, \action_t}  \vphantom{\log  \expectedsymbol{\latentstate \sim \belief_t}\latentprobtransitions} - \log \latentobservationfn_{\decoderparameter}\fun{\observation_{t + 1}\mid \latentstate_{t + 1}} \right]. \label{eq:dkl}
\end{multline}
Eq.~\ref{eq:dkl} consists of $4$ terms: a normalization factor, negative entropy of $\beliefencoder$, belief update conformity with the latent MDP's state transition function, and filtration of latent states unrelated to $\observation_{t+1}$.
\begin{remark}[Variance of the latent observations]\label{rmk:variance}%
The WAE-MDP learns from the augmented POMDP $\augmentedpomdp$ (Sect.~\ref{sec:refinement-model}) %
 which is equivalent to the original environment and possesses a deterministic observation function. 
Therefore, the WAE-MDP also learns to deterministically map latent states to their observation through $\augmentedobservationfn$.
Still, we introduce a variance parameter to enable learning $\latentbeliefupdate$:
when deterministic, the observation terms of Eq.~\ref{eq:latent-belief-update} and~\ref{eq:dkl} are Dirac, which prevents learning;
deterministically filtering out states from the next belief that do not share the next observation would require constructing the full belief distribution, which is usually intractable.
To alleviate that, $\augmentedobservationfn$, coupled with the variance learned from the observation loss $\observationloss{}$, serves as a smooth version of the Dirac function.
\end{remark}

\smallparagraph{Policy learning}~is enabled by inputting the sub-belief into the policy, while the optimization of the belief encoder parameters by the RL agent is not allowed. Our method is applicable to \emph{any} on-policy algorithm, and we employ A2C in our experiments. We provide the final algorithm in Appendix~\ref{appendix:algorithm}.

\section{Experiments}
\begin{figure}
    \centering
    \includegraphics[width=.925\linewidth]{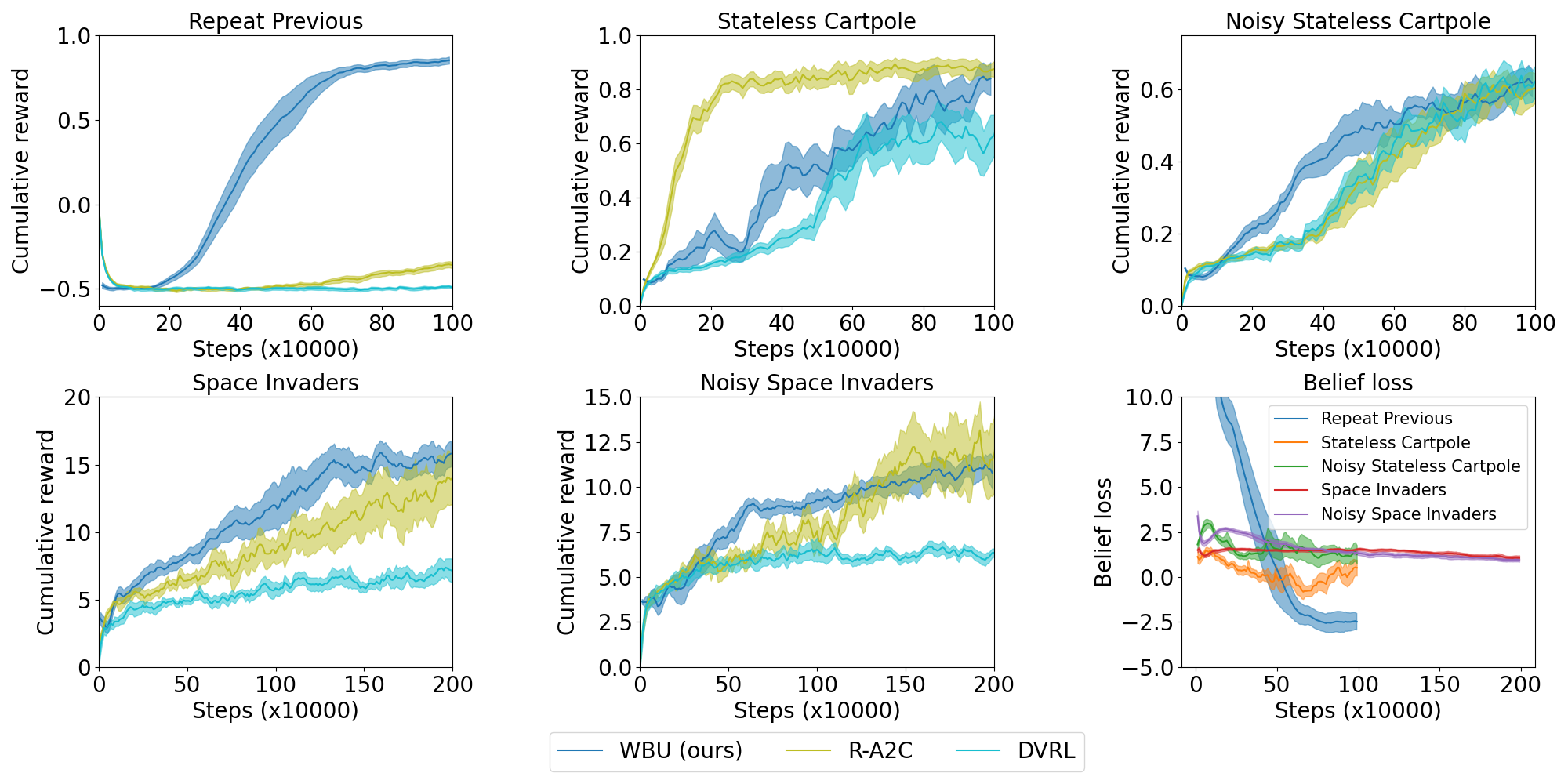}
    \caption{Evolution of the (i) undiscounted cumulative return for WBU, R-A2C and DVRL, and (ii) belief loss during learning for WBU (mean and standard error).
    We report $5$ instances of each algorithm.
    Appendix \ref{appendix:hyper} details the hyperparameter search performed. 
    }
    \label{fig:experiment}
\end{figure}

To evaluate our approach, we identify three types of POMDPs: those requiring \emph{long-term memory}, those where \emph{features of the state space are hidden} (and may be inferred from short-term memory), and those with \emph{noisy} observations. %
Notably, we stress that long-term memory is crucial in POMDPs, whereas short-term memory could be mitigated by stacking frames (e.g., \citealt{Mnih2015Human-levelLearning}).
We compare our agent to R-A2C and DVRL (Fig.~\ref{fig:experiment}), trained in environments from \textsc{POPGym} \citep{morad2023popgym} and our own partially observable version of \textsc{MinAtar} \citep{young19minatar}.

\smallparagraph{Memorization.}~%
The \textsc{RepeatPrevious} environment involves shuffling two decks of cards at the start of each episode and presenting the agent with a card at each time step.
The goal is to identify the suit of the card seen 8 time steps earlier.
\emph{Our algorithm stands out as the sole method demonstrating mid- to long-term memorization capabilities.}
Unlike other methods, notably DVRL which also attempts to learn a belief distribution, WBU provably acquires a suitable representation of the history by learning to maintain a sufficient statistic, thereby explaining its ability to retain past information.

\smallparagraph{Hidden features.}~%
We employ a cart pole scenario (\textsc{StatelessCartPole}) where velocity components of the system are hidden.
R-A2C excels rapidly here, capitalizing on short-term memory to infer velocities from the preceding observation, while DVRL is overtaken.
Still, WBU eventually reaches R-A2C final performance.
We also explore the \textsc{SpaceInvaders} environment, where the agent takes command of a cannon with the objective of shooting at groups of moving aliens. In the observation, we intentionally concealed the direction of alien movement and confounded friendly and enemy fires.
In this more challenging setting, WBU excels by earning the highest rewards.

\smallparagraph{Noise.}~%
We explore two types of noise. First, we introduce \emph{Gaussian noise} to the observations of \textsc{StatelessCartPole}. Second, for \textsc{SpaceInvaders}, \emph{binary noise} is injected via a radar-like mask obscuring the position of each alien with high probability. Hence, the agent must infer their positions based on previous observations.
By leveraging its ability to maintain a belief over the noiseless (latent) state space, WBU demonstrates its resilience to noise and swiftly provides superior solutions, whereas R-A2C eventually achieves comparable performance but with more variance.

\begin{wrapfigure}{r}{0.45\textwidth}
    \centering
    \includegraphics[width=0.445\textwidth]{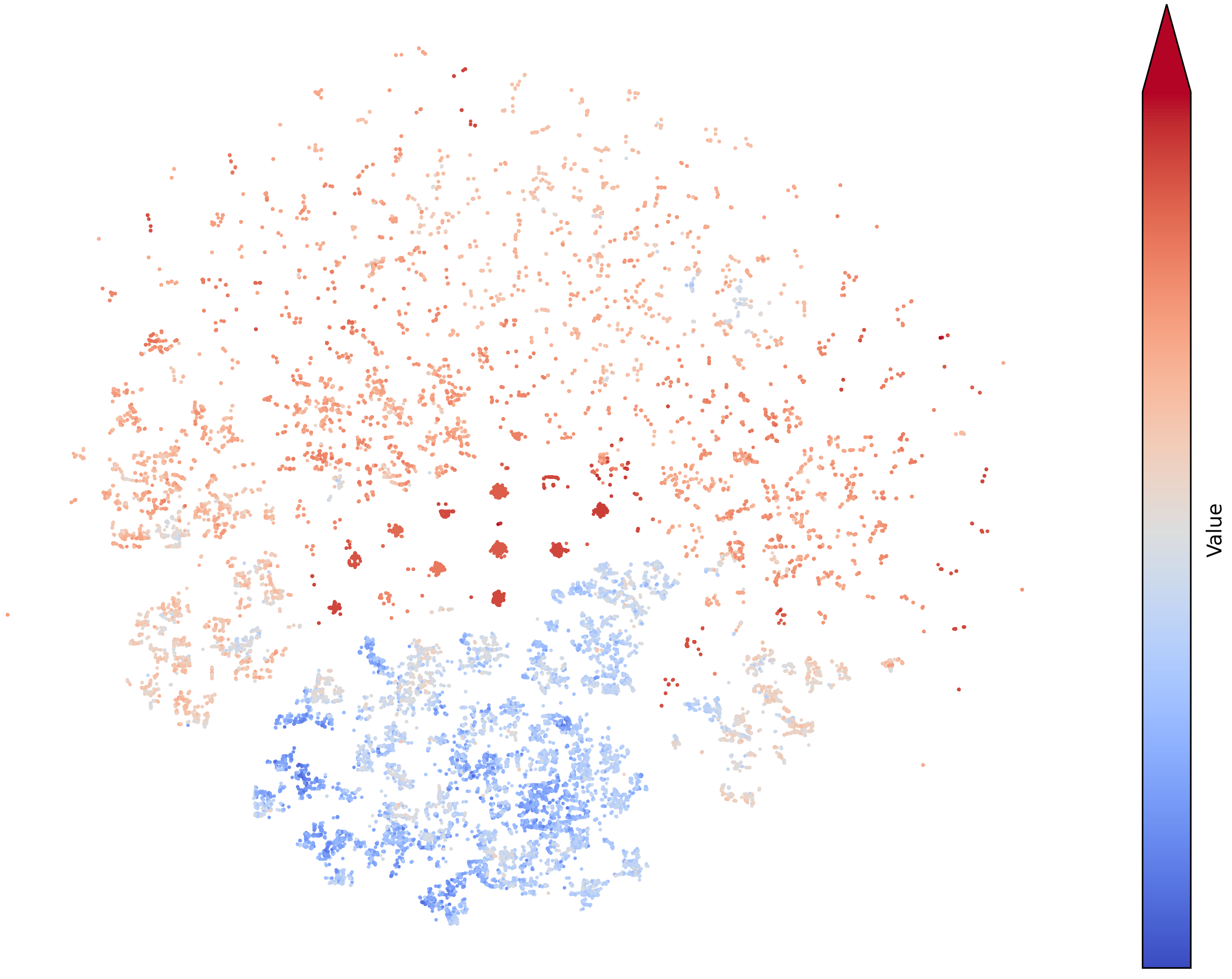}
    \caption{Two-dimensional t-SNE of our belief representation for \textsc{SpaceInvaders}.}
    \label{fig:tsne}
\end{wrapfigure}
\smallparagraph{Belief representation.}~%
Ideally, close policy inputs should lead to close values, which would ease its optimization.
Thm.~\ref{thm:value-diff-bound-representation} provides such a representation guarantee and ensures that the representation induced by $\beliefencoder$ captures the value function.
To support this, we performed a t-SNE \citep{JMLR:v9:vandermaaten08a} on our belief representation at the late stage of training, which projects latent beliefs on a 2D space (Fig.~\ref{fig:tsne}).
Interestingly, \emph{latent beliefs clustered together have indeed close values}, in line with Thm.~\ref{thm:value-diff-bound-representation}.
We reported the belief loss throughout the training phases (Fig.~\ref{fig:experiment}).
Importantly, unlike other baselines, \emph{our approach distinctly separates the optimization of $\beliefencoder$ and $\latentpolicy$}. Consequently, \emph{the policy optimization does not influence the representation} which is solely learned via $\beliefencoder$.
The decrease in this loss thus relates to improved representation quality for RL.

\section{Conclusion}
WBU provides a novel approach that approximates directly the belief update for POMDPs, in contrast to SOTA methods that uses the RL objective and regularization to attempt to turn the history into a sufficient statistic.
We believe that formulating a theoretically sound framework which allows reasoning on the solutions learned is equally important and parallel to works focusing primarily on agent performance.
By learning the belief and its update rule, we provide strong guarantees on the quality of the belief, its ability to condition the optimal value function, and ultimately, the effectiveness of our algorithm. Our theoretical analysis and experimental results demonstrate the potential of our approach.
Overall, our WBU algorithm provides a promising new direction for RL in POMDPs, with potential applications in a wide range of settings where decision-making is complicated by uncertainty and partial observability, or when guarantees on the agent behaviors are required. %

\smallparagraph{Future work.}~%
The theory we developed is not limited to the algorithm proposed in our paper. It opens diverse avenues for future work, e.g., on formally verifiable policies for POMDPs, by leveraging the guarantees presented in our framework.
We also leave the study and the adaptation of our framework under relaxed assumptions (e.g., in settings akin to the work of \citealt{lambrechts2022recurrent}) to future work.
In addition, Thm~\ref{thm:value-diff-bounds} enables policy optimization through planning in the learned model, as demonstrated in successful model-based RL methods (e.g., \citealt{hafner2021mastering}).
Scaling to high-dimensional observations (e.g., images) may potentially be computation-intensive due to observation filtering.
For this further challenge, we suggest either to modify the WAE-MDP framework by using a stochastic decoder (see \citealt{DBLP:conf/iclr/TolstikhinBGS18}, e.g., via PixelCNN, \citealt{DBLP:conf/nips/OordKEKVG16}), or learning a lower-dimensional latent observation space synced with the policy with a normalized or (relaxed) discrete prior (e.g., via a WAE-GAN, \citealt{DBLP:conf/iclr/TolstikhinBGS18}).
Finally, incorporating bisimulation metrics \citep{DBLP:journals/tcs/DesharnaisGJP04,DBLP:journals/siamcomp/FernsPP11} will strengthen guarantees for belief learning, even though bisimulation is challenging in POMDPs \citep{DBLP:conf/ijcai/CastroPP09}.

\ificlrfinal
\subsection*{Acknowledgements}
This research was supported by funding from the Flemish Government under the ``Onderzoeksprogramma Artifici\"{e}le Intelligentie (AI) Vlaanderen'' program and was supported by the DESCARTES iBOF project.
R. Avalos is supported by the Research Foundation – Flanders (FWO), under grant number 11F5721N. 
G.A. Perez is also supported by the Belgian FWO “SAILor” project (G030020N).
We thank Mathieu Reymond, Denis Steckelmacher, and Mustafa Mert Çelikok for their valuable feedback.
\fi

\subsection*{Reproducibility Statement}
We referenced in the main text the parts of the Appendix presenting the proofs of our Lemma (Appendix~\ref{appendix:stationary-histories}) and Theorems (Appendix~\ref{appendix:value-diff-bounds}).
We also provide the pseudo-code of our algorithm (Appendix~\ref{appendix:algorithm}) as well as extra details required to compute our losses (Appendix~\ref{appendix:belief-update}, \ref{appendix:temperature}, and~\ref{appendix:losses}). 
We included the code in Supplementary Material where we detailed installation instructions and a script to rerun all the experiments presented in the paper. Additionally, we provide the details of our hyperparameter search (Appendix~\ref{appendix:hyper}).

\bibliography{references}
\bibliographystyle{iclr2024_conference}

\appendix
\newpage
\begin{appendices}
\section*{Appendix}

\section{The Belief Update Rule}\label{appendix:belief-update}

Let $\latentpomdp = \tuple{\latentmdp, \observations, \latentobservationfn}$ be a latent POMDP with underlying MDP $\latentmdp = \latentmdptuple$. 
At step $t \geq 0$, assume that the current latent belief is $\latentbelief_t \in \latentbeliefs$. 
Then, when $\action_t \in \actions$ is executed and $\observation_{t + 1} \in \observations$ is observed, $\latentbelief_t$ is updated according to $\latentbeliefupdate\fun{\latentbelief_t, \action_t, \observation_{t + 1}}$ as follows, so that for any (believed) state $\latentstate_{t + 1} \in \latentstates$,

\begin{equation}
    \latentbelief_{t + 1}\fun{\latentstate_{t + 1}} = \frac{\expectedsymbol{\latentstate_t \sim \latentbelief_t} \latentprobtransitions_{\decoderparameter}\fun{\latentstate_{t + 1} \mid \latentstate_t, \action_t} \cdot \latentobservationfn_{\decoderparameter}\fun{\observation_{t + 1}\mid\latentstate_{t + 1}}}{\E_{\latentstate_{t} \sim \latentbelief_t} \E_{\latentstate' \sim \latentprobtransitions_{\decoderparameter}(\cdot \mid \latentstate_t, \latentaction_t)} \latentobservation_{\decoderparameter}\fun{\observation_{t + 1} \mid \latentstate'}}. \label{eq:latent-belief-update-appendix}
\end{equation}

\begin{figure}[h]
    \centering
    \includegraphics[width=\textwidth]{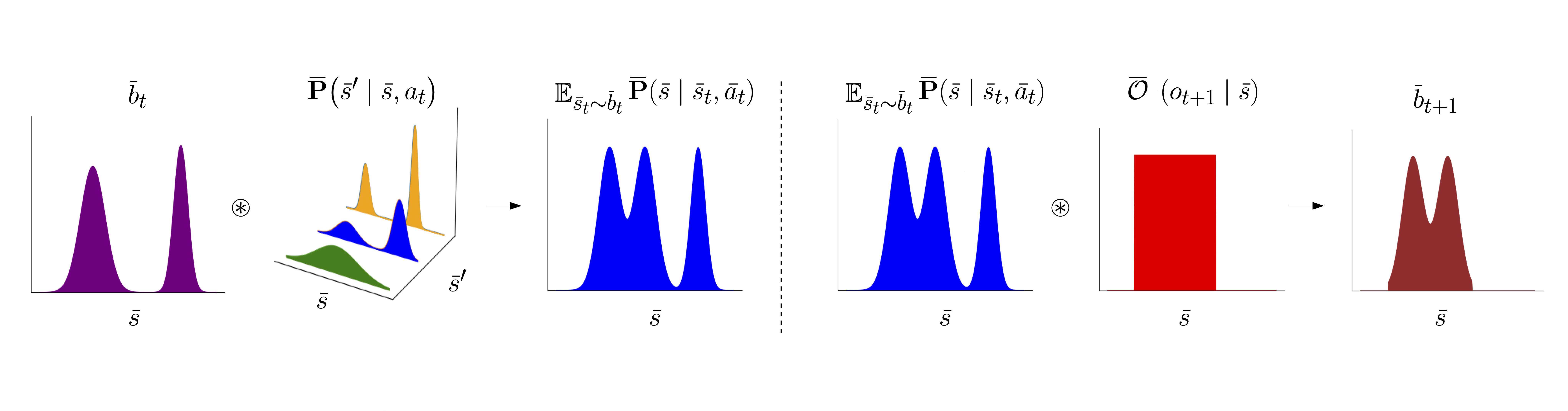}
    \caption{The belief update rule: (\emph{left}) transformation of the current belief $\latentbelief_t$ with the transition probability function $\latentprobtransitions$, evaluated on the current action $\action_t$, into the next state probability density; (\emph{right}) filtering out the next states that could not have produced the next observation $\observation_{t+1}$.}
    \label{fig:belief_update}
\end{figure}

The belief update rule $\latentbeliefupdate$ %
is divided in two steps (cf. Fig.~\ref{fig:belief_update}).
First, the current belief distribution $\latentbelief_t$ is used to marginalize the latent transition function  $\latentprobtransitions_{\decoderparameter}$ over the believed latent states, to further infer the distribution over the possible next states.
This first part corresponds to computing the probability of going to the next believed state $\latentstate_{t + 1}$ from the possible states $\latentstate_t$ witnessed by the current belief $\latentbelief_t$. 
Second, the next observation $\observation_{t+1}$ is used to filter the resulting density based on the observation function: believed latent states that do not share the observation $\observation_{t + 1}$ perceived are filtered out according to $\latentobservationfn\fun{\observation_{t + 1} \mid \latentstate_{t + 1}}$.
It is worth noting that the latent model is learned from $\augmentedpomdp$, whose observation function is deterministic.
Without modelling the latent observation function $\latentobservationfn_{\decoderparameter}$ as a normal distribution, the second part of the belief update would need to eliminate all next states with different observations --- which is not gradient descent friendly.
The third operation (not present in Fig.~\ref{fig:belief_update}) normalizes the output of the observation filtering to obtain a probability density.

\section{Dirac Measures}

In this work, we consider the \emph{Dirac delta function} $\delta$ as a \emph{measure}.
Specifically, this means that for any complete, separable space $\measurableset$ and point $a \in \measurableset$, the \emph{Dirac measure} with impulse $a$ is $\diracimpulsesymbol_a\in \distributions{\measurableset}$ and satisfies $\diracimpulse{A}{a} = 1$ if $a \in A$ and $\diracimpulse{A}{a} = 0$ otherwise, for any $A \in \borel{\measurableset}$.
Interesting properties of the Dirac measure include
$\diracimpulsesymbol_{a} = \lim_{\sigma \to 0} \normal{a}{\sigma^2}{}$, where $\normal{a}{\sigma^2}{}$ is the normal distribution with mean $a$ and variance $\sigma^2$, and
$\int_{\measurableset} \diracimpulse{x}{a} f\fun{x} \, dx = f\fun{a}$ for any compactly supported function $f$.

\section{Proof of Lemma~\ref{lem:stationary-histories}: Stationarity over Histories}\label{appendix:stationary-histories}
Let us formally restate the Lemma:
\begin{lemma}~\label{lem:extended-stationary-histories}
Let $\pomdp$ be an episodic POMDP with action space $\actions$ and observation space $\observations$.
There is a well defined probability distribution $\historydistribution_{\latentpolicy} \in \distributions{\histories}$ over histories drawn at the limit from the interaction of the RL agent with $\pomdp$, when it operates under a latent policy $\latentpolicy$ conditioned over the beliefs of a latent POMDP $\latentpomdp$ that shares the action and observation spaces of $\pomdp$ and is executed via the belief encoder, i.e., $\action \sim \latentpolicy\fun{\sampledot \mid \beliefencoder^{*}\fun{\history}}$ for any $\history \in \histories$.
\end{lemma}
\begin{proofsketch}
Build a \emph{history unfolding} as the MDP whose state space consists of all histories and keeps track of the current history of $\pomdp$ at any time of the interaction. 
The resulting MDP remains episodic since it is equivalent to $\pomdp$: the former mimics the behaviors of the latter under $\latentpolicy$.
All episodic processes are \emph{ergodic} \citep{DBLP:conf/nips/Huang20}, which guarantees the existence of such a distribution.
\end{proofsketch}
We dedicate this Section to formally detailing and proving every claim of this proof sketch.
Before going further, we formally define the notion of \emph{episodic process}, and we further introduce the notions of \emph{memory-based policies}, \emph{Markov Chains}, and \emph{limiting distributions in Markov Chains}.

\subsection{Preliminaries}
Recall by Assumption~\ref{assumption:episodic} that the environment $\pomdp$ is an episodic process.
We \emph{formally} recall the notion of \emph{episodic process}:
\begin{definition}[Episodic RL process]\label{def:episodic}
The RL procedure is \emph{episodic} iff the environment $\pomdp$ embeds a special \emph{reset state} $\resetstate \in \states$ so that (i) under any policy $\policy$, the environment is almost surely eventually reset: $\Prob_{\policy}^{\mdp}\fun{\set{\seq{\state}{\infty}, \seq{\action}{\infty} \mid \exists t > 0, \state_t = \resetstate}} = 1$;
(ii) when reset, the environment transitions to the initial state: $\probtransitions\fun{\sinit \mid \sreset, \action} > 0$ and $\probtransitions\fun{\states \setminus \set{\sinit, \sreset} \mid \sreset, \action} = 0$ for all $\action \in \actions$;
and (iii) \emph{the reset state is observable}: there is an observation $\observation^{\star} \in \observations$ so that $\observationfn\fun{\observation^{\star} \mid \state', \action} = 0$ when $\state' \neq \sreset$, and $\observationfn\fun{\sampledot \mid \sreset, \action} = \diracimpulsesymbol_{\observation^\star}$ for $\action \in \actions$.
An \emph{episode} is a history $\historytuple{\action}{\observation}{T}$
where $\observationfn\fun{\observation_1 \mid \sinit, \action_0} > 0 \text{ and } 
\observation_T = \observation^{\star}$.
\end{definition}

\smallparagraph{Encoding memory through Mealy machines.}~\emph{Policies} are building blocks to define the probability space of any MDP.
To deal with policies whose decisions are based on unrolling histories, we formally define the notion of \emph{memory} in policies. 

\begin{definition}[Policy as Mealy Machine]
Given an MDP $\mdp = \mdptuple$, any policy for $\mdp$ can be encoded as a \emph{stochastic Mealy machine} $\policy = \mealymachine{\policy}$ as follows:
$\policystates$ is a set of \emph{memory states};
$\mealyaction{\policy} \colon \states \times \policystates \to \distributions{\actions}$ is the \emph{next action function};
$\mealyupdate{\policy} \colon \states \times \policystates \times \actions \times \states \to \distributions{\policystates}$ is the \emph{memory update function}; and
$\qinit$ is the initial memory state.
\end{definition}
\begin{example}[Stationary policy]
A stationary policy $\policy$ can be encoded as any Mealy machine $\policy$ with memory space $\policystates$ where $\left|\policystates\right| = 1$.
\end{example}
\begin{example}[Latent policy]\label{ex:memory-latent-policy}
Let  $\pomdp = \pomdptuple$ with underlying MDP $\mdp = \mdptuple$ and the latent space model $\latentpomdp$ with initial state $\zinit$ be the POMDPs of Lemma~\ref{lem:extended-stationary-histories}.
Then, any latent (stationary)  policy $\latentpolicy \colon \latentbeliefs \to \distributions{\actions}$ conditioned on the belief space $\latentbeliefs$ of $\latentpomdp$ can be executed in the belief MDP $\beliefmdp$ of $\pomdp$ via the Mealy machine $\latentpolicy' = \langle \latentbeliefs, {\mealyaction{\latentpolicy}}, {\mealyupdate{\latentpolicy}}, \diracimpulsesymbol_{\zinit} \rangle$, keeping track in its memory of the current latent belief $\latentbelief \in \latentbeliefs$ inferred by our belief encoder $\beliefencoder$.
This enables the agent to take its decisions solely based on the latter: $\mealyaction{\latentpolicy}(\sampledot \mid \belief, \latentbelief) = \latentpolicy(\sampledot \mid \latentbelief)$.
When the belief MDP transitions to the next belief $\belief'$, the memory is then updated according to the observation dynamics:
\begin{align*}
\hfill \mealyupdate{\latentpolicy}(\latentbelief' \mid \belief, \latentbelief, \action,  \belief') &= 
\frac{\expectedsymbol{\state \sim \belief}\expectedsymbol{\state' \sim \probtransitions\fun{\sampledot \mid \state, \action}} \expectedsymbol{\observation' \sim \observationfn\fun{\sampledot \mid \state', \action}} \diracimpulse{\latentbelief'}{\beliefencoder(\latentbelief, \action, \observation')} \cdot \diracimpulse{\belief'}{\beliefupdate(\belief, \action, \observation')}}
{\expectedsymbol{\state \sim \belief}\expectedsymbol{\state' \sim \probtransitions\fun{\sampledot \mid \state, \action}} \expectedsymbol{\observation' \sim \observationfn\fun{\sampledot \mid \state', \action}} \diracimpulse{\belief'}{\beliefupdate(\belief, \action, \observation')}} \;\, \text{if $\belief' \neq \diracimpulsesymbol_{\sreset}$}, \\
\mealyupdate{\latentpolicy}(\sampledot \mid \belief, \latentbelief, \action,  \diracimpulsesymbol_{\sreset}) &= \diracimpulsesymbol_{\latentstate_{\mathsf{reset}}} \qquad\qquad\qquad\qquad\qquad \text{otherwise (to fulfil the episodic constraint).}
\end{align*}
Note that $\mealyupdate{\latentpolicy}$ is simply  obtained by applying the usual conditional probability rule:
$
\mealyupdate{\latentpolicy}\fun{\latentbelief' \mid \belief, \latentbelief, \action, \belief'} = \nicefrac{\text{Pr}\fun{\belief', \latentbelief' \mid \belief, \latentbelief, \action}}{\text{Pr}\fun{\belief' \mid \belief, \latentbelief, \action}}, \text{ where } \text{Pr}\fun{\belief', \latentbelief' \mid \belief, \latentbelief, \action} = \expectedsymbol{\state \sim \belief}\expectedsymbol{\state' \sim \probtransitions\fun{\sampledot \mid \state, \action}} \expectedsymbol{\observation' \sim \observationfn\fun{\sampledot \mid \state', \action}} \diracimpulse{\latentbelief'}{\beliefencoder(\latentbelief, \action, \observation')} \cdot \diracimpulse{\belief'}{\beliefupdate(\belief, \action, \observation')}
$
and
$\text{Pr}\fun{\belief' \mid \belief, \latentbelief, \action} = \probtransitions_{\beliefs}\fun{\belief' \mid \belief, \action}$ since the next \emph{original} belief state is independent of the {current} \emph{latent} belief state.
\end{example}

\begin{definition}[Markov Chain]\label{def:markov-chain}
A \emph{Markov Chain} (MC) is an MDP whose action space $\actions$ consists of a singleton, i.e., $\left|\actions\right| = 1$.
Any MDP $\mdp = \mdptuple$ and memory-based policy $\policy = \mealymachine{\policy}$ induces a Markov Chain \[\mdp^\policy = \langle \states \times \policystates, \probtransitions_{\policy}, \rewards_{\policy}, \tuple{\sinit, \qinit}, \discount \rangle,\] where:
\begin{itemize}
    \item the state space consists of the product of the original state space and the memory of $\policy$;
    \item the transition function embeds the next action and the policy update functions from the policy, i.e.,
    \begin{equation*}
    \probtransitions_{\policy}\fun{\tuple{\state', \policystate'} \mid \tuple{\state, \policystate}} = \expectedsymbol{\action \sim \mealyaction{\policy}\fun{\sampledot \mid \state, \policystate}} \mealyupdate{\policy}\fun{\policystate' \mid \state, \policystate, \action, \state'} \cdot \probtransitions\fun{\state' \mid \state, \action}, \text{ and}
    \end{equation*}
    \item the rewards are averaged over the possible actions produced by the next action function, i.e., $\rewards_{\policy}\fun{\tuple{\state, \policystate}} = \expectedsymbol{\action \sim \mealyaction{\policy}\fun{\sampledot \mid \state, \policystate}} \rewards\fun{\state, \action}$.
\end{itemize}
Furthermore, \emph{the probability measure $\Prob_{\policy}^{\mdp}$ is actually the unique probability measure defined over the measurable infinite trajectories of the MC} $\mdp^{\policy}$ \citep{DBLP:books/wi/Puterman94}.
\end{definition}
We now formally define the distribution over states encountered at the limit when an agent operates in an MDP under a given policy, as well as the conditions of existence of such a distribution.

\begin{definition}[Bottom strongly connected components and limiting distributions]\label{def:stationary-distr}
Let $\mdp$ be an MDP with state space $\mdp$ and $\policy$ be a policy for $\mdp$.
Write $\mdp\left[\state\right]$ for the MDP where we change the initial state $\sinit$ of $\mdp$ by $\state \in \states$.
The measure $\stationary{\policy}^{t}: \states \to \distributions{\states}$ with $ \stationary{\policy}^{t}\fun{\state' \mid \state} =  \Prob^{\mdp\left[\state\right]}_\policy\fun{\set{{\seq{\state}{\infty}, \seq{\action}{\infty}}
\mid \state_t = \state'}}$ is the distribution giving the probability for the agent of being in each state of $\mdp\left[\state\right]$ after exactly $t$ steps.
The subset $B \subseteq \states$ is a \emph{strongly connected component} (SCC)
of $\mdp^\policy$ if for any pair of states $\state, \state' \in B$, $\stationary{\policy}^t\fun{\state' \mid \state} > 0$ for some $t \in \N$. 
It is a \emph{bottom SCC} (BSCC) if (i) $B$ is a maximal SCC, and (ii) for each $\state \in B$, $\probtransitions_\policy\fun{B \mid \state} = 1$.
The unique \emph{stationary distribution} of $B$ is
$\stationary{\policy} \in \distributions{B}$, defined as $\stationary{\policy}\fun{\state} = \expectedsymbol{\dot{\state} \sim \stationary{\policy}}{\probtransitions_{\policy}\fun{\state \mid \dot{\state}}} = \lim_{T \to \infty} \frac{1}{T} \sum_{t = 0}^{T} \stationary{\policy}^{t}\fun{\state \mid \state_{\bot}}$ for any $\state_{\bot} \in B$ \citepAR{BK08}.
An MDP $\mdp$ is \emph{ergodic} under the policy $\policy$ if the state space of $\mdp^\policy$ consists of a unique aperiodic BSCC. 
In that case, $\stationary{\policy} = \lim_{t \to \infty} \stationary{\policy}^t\fun{\sampledot \mid \state}$ for all $\state \in \states$.
\end{definition}

To provide such a stationary distribution over histories, we define a \emph{history unfolding} MDP, where the state space keeps track of the current history of $\pomdp$ during the interaction.
We then show that this history MDP is \emph{equivalent} to $\pomdp$ under $\latentpolicy$.

\subsection{History Unfolding}\label{appendix:unfolding}

Let us define the \emph{history unfolding} MDP $\indexof{\mdp}{\historydistribution}$, which consists of the tuple $\tuple{\indexof{\states}{\historydistribution}, \actions, \indexof{\probtransitions}{\historydistribution}, \indexof{\rewards}{\historydistribution}, \star, \discount}$, where:
\begin{itemize}
    \item the state space consists of the set of all the possible histories (i.e., sequence of actions and observations) that can be encountered in $\pomdp$, i.e., $\indexof{\states}{\historydistribution} = \histories \cup \set{\star, \hreset}$, which additionally embeds a special symbol $\star$ indicating that no observation has been perceived yet %
    with
    $\beliefupdate^{*}\fun{\star} = \diracimpulsesymbol_{\sinit}$, as well as a special reset state $\hreset$;
    \item the transition function maps the current history to the belief space to infer the distribution over the next possible observations, i.e.,
    \begin{align*}
    \indexof{\probtransitions}{\historydistribution}\fun{\history' \mid \history, \action} &= \expectedsymbol{\state \sim \beliefupdate^{*}\fun{\history}}\expectedsymbol{\state'\sim\probtransitions\fun{\sampledot \mid \state, \action}}\expectedsymbol{\observation' \sim \observationfn\fun{\sampledot \mid \state', \action}}\diracimpulse{\history'}{\history \cdot \action \cdot \observation'} \quad\quad \text{if } \beliefupdate^{*}\fun{\history} \neq \diracimpulsesymbol_{\sreset}, \text{ and}\\
    \indexof{\probtransitions}{\historydistribution}\fun{\history' \mid \history, \action} &= \beliefprobtransitions\fun{\diracimpulsesymbol_{\sreset} \mid \diracimpulsesymbol_{\sreset}, \action} \cdot \diracimpulsesymbol_{\hreset}\fun{\history'} + \beliefprobtransitions\fun{\diracimpulsesymbol_{\sinit} \mid \diracimpulsesymbol_{\sreset}, \action} \cdot \diracimpulse{\history'}{\star} \quad \; \text{otherwise},
    \end{align*}
    where
    $\history \cdot \action \cdot \observation'$ is the concatenation of $\action$, $\observation'$ with the history $\history = \historytuple{\action}{\observation}{T}$, resulting in the history $\tuple{\seq{\action}{T}, \indexof{\observation}{1:T+1}}$ so that $\action_T = \action$ and $\observation_{T+1} = \observation'$; and
    \item the reward function maps the history to the belief space as well, which enables to infer the expected rewards obtained in the states over the this belief, i.e., $\indexof{\rewards}{\historydistribution}\fun{\history, \action} = \expectedsymbol{\state \sim \beliefupdate^{*}\fun{\history}} \rewards\fun{\state, \action}$.
\end{itemize}

We now aim at showing that, under the latent policy $\latentpolicy$, the POMDP $\pomdp$ and the MDP $\indexof{\mdp}{\historydistribution}$ are \emph{equivalent}. 
More formally, we are looking for an equivalence relation between two probabilistic models, so that the latter induce the same behaviors, or in other words, the same expected return.
We formalize this equivalence relation as a \emph{stochastic bisimulation} between $\beliefmdp$ (that we know being an MDP formulation of $\pomdp$) and $\indexof{\mdp}{\historydistribution}$.

\begin{definition}[Bisimulation]
Let $\mdp = \mdptuple$ be an MDP.
A stochastic \emph{bisimulation} $\equiv$
on $\mdp$ is a behavioral equivalence between states $\state_1, \state_2 \in \states$ 
so that, $\state_1 \equiv \state_2$ iff
\begin{enumerate}
    \item $\rewards(\state_1, \action) = \rewards(\state_2, \action)$, and
    \item $\probtransitions(T \mid \state_1, \action) = \probtransitions(T \mid \state_2, \action)$,
\end{enumerate}
for each action $\action \in \actions$ and equivalence class $T \in \states / \equiv$.
\end{definition}
Properties of bisimulation include trajectory equivalence and the equality of their optimal expected return \citep{DBLP:conf/popl/LarsenS89,DBLP:journals/ai/GivanDG03}.
The relation can be extended to compare two MDPs by considering the disjoint union of their state space.
\begin{lemma}\label{lem:bisim}
Let $\pomdp$ be the POMDP of Lemma~\ref{lem:extended-stationary-histories}, and $\latentpolicy \colon \latentbeliefs \to \distributions{\actions}$ be a latent policy conditioned on the beliefs of a latent space model of $\pomdp$. 
Define the stationary policy $\latentpolicy^{\clubsuit} \colon \indexof{\states}{\historydistribution} \to \distributions{\actions}$ for $\indexof{\mdp}{\historydistribution}$ as ${\latentpolicy}^{\clubsuit}\fun{\sampledot \mid \history} = \latentpolicy\fun{\sampledot \mid \beliefencoder^{*}\fun{\history}}$, 
and the policy ${\latentpolicy}^{\diamondsuit}$ for $\beliefmdp$ encoded by the Mealy machine detailed in Example~\ref{ex:memory-latent-policy}.
Then, $\indexof{\mdp}{\historydistribution}^{{\latentpolicy}^{\clubsuit}}$ and $\beliefmdp^{{\latentpolicy}^{\diamondsuit}}$ are in stochastic bisimulation.
\end{lemma}
\begin{proof}
First, note that the MC $\beliefmdp^{\latentpolicy^\diamondsuit}$ is defined as the tuple $\tuple{\beliefs \times \latentbeliefs, \probtransitions_{\latentpolicy^{\diamondsuit}}, \rewards_{\latentpolicy^{\diamondsuit}}, \tuple{\binit, \latentbelief_I}, \discount}$ so that 
\begin{align*}
\probtransitions_{\latentpolicy^{\diamondsuit}}\fun{\belief', \latentbelief' \mid \belief, \latentbelief} &= \expectedsymbol{\action \sim \latentpolicy\fun{\sampledot \mid \latentbelief}} \mealyupdate{\latentpolicy}\fun{\latentbelief' \mid \belief, \latentbelief, \action, \belief'} \cdot \beliefprobtransitions\fun{\belief' \mid \belief, \action} \\
&= \expectedsymbol{\action \sim \latentpolicy\fun{\sampledot \mid \latentbelief}}\expectedsymbol{\state \sim \belief}\expectedsymbol{\state' \sim \probtransitions\fun{\sampledot \mid \state, \action}}\expectedsymbol{\observation' \sim \observationfn\fun{\sampledot \mid \state, \action}} 
\diracimpulse{\latentbelief'}{\beliefencoder\fun{\latentbelief, \observation, \action}} \cdot
\diracimpulse{\belief'}{\beliefupdate\fun{\belief, \observation, \action}}, \text{and} \\
\rewards_{\latentpolicy^{\diamondsuit}}({{\belief, \latentbelief}}) &= \expectedsymbol{\action \sim \latentpolicy\fun{\sampledot \mid \latentbelief}}\expectedsymbol{\state \sim \belief}\rewards\fun{\state, \action}. \tag{cf. Definition~\ref{def:markov-chain}}
\end{align*}
Define the relation
$\pomdpbisim$ as the set $\set{\big\langle{\history, \langle\belief, \latentbelief\rangle} \big\rangle\mid \beliefupdate^{*}\fun{\history} = \belief \text{ and } \beliefencoder^{*}\fun{\history} = \latentbelief} \subseteq \indexof{\states}{\historydistribution} \times \beliefs \times \latentbeliefs$.
We show that $\pomdpbisim$ is a bisimulation relation between the states of $\indexof{\mdp}{\historydistribution}^{\latentpolicy^{\clubsuit}}$ and $\beliefmdp^{{\latentpolicy}^{\diamondsuit}}$. 
Let $\history \in \indexof{\states}{\historydistribution}$, $\belief \in \beliefs$, and $\latentbelief \in \latentbeliefs$ so that $\history \pomdpbisim \langle \belief, \latentbelief \rangle$:
\begin{enumerate}
    \item \label{enum:reward-bisim}$\rewards_{\latentpolicy^{\clubsuit}}\fun{\history} = \expectedsymbol{\action \sim \latentpolicy\fun{\sampledot \mid \beliefencoder^{*}\fun{\history}}}\expectedsymbol{\state \sim \beliefupdate^{*}\fun{\history}}\rewards\fun{\state, \action} = \expectedsymbol{\action \sim \latentpolicy\fun{\sampledot \mid \latentbelief}}\expectedsymbol{\state \sim \belief}\rewards\fun{\state, \action} = \rewards_{\latentpolicy^{\diamondsuit}}({\belief, \latentbelief})$;
    \item \label{enum:transition-bisim}Each equivalence class $T \in \fun{\indexof{\states}{\historydistribution} \times \beliefs \times \latentbeliefs} / \pomdpbisim$ consists of histories sharing the same belief and latent beliefs.
    Then, each equivalence class $T$ can be associated to a single belief and latent belief pair.
    Concretely, let $\belief' \in \beliefs$, $\latentbelief' \in \latentbeliefs$, an equivalence class of $\pomdpbisim$ has the form $T = \left[\langle\belief', \latentbelief' \rangle\right]_{\pomdpbisim}$ so that
    \begin{enumerate}
        \item the projection of $\left[\langle\belief', \latentbelief' \rangle\right]_{\pomdpbisim}$ on $\indexof{\states}{\historydistribution}$ is the set $\set{\history' \in \indexof{\states}{\historydistribution} \mid \beliefupdate^{*}\fun{\history'} = \belief' \text{ and } \beliefencoder^{*}\fun{\history'} = \latentbelief'}$, and
        \item the projection of $\left[\langle\belief', \latentbelief' \rangle\right]_{\pomdpbisim}$ on the state space of $\beliefmdp^{\latentpolicy^{\diamondsuit}}$ is merely the pair $\langle\belief', \latentbelief'\rangle$.
    \end{enumerate}
    Therefore, 
    \begin{align*}
        &{\probtransitions}_{\latentpolicy^{\clubsuit}}\fun{\left[ \langle \belief', \latentbelief' \rangle \right]_{\pomdpbisim} \mid \history}\\
        =& \int_{\left[ \langle \belief', \latentbelief' \rangle \right]_{\pomdpbisim}}\, \expectedsymbol{\action \sim \latentpolicy\fun{\sampledot \mid \beliefencoder^{*}\fun{\history}}} \expectedsymbol{\state \sim \beliefupdate^{*}} \expectedsymbol{\state' \sim \probtransitions\fun{\sampledot \mid \state, \action}} \expectedsymbol{\observation' \sim \observationfn\fun{\sampledot \mid \state', \action}} \diracimpulse{\history'}{\history \cdot \action \cdot \observation'} \; d\history' \\
        =& 
        \int_{\indexof{\states}{\historydistribution}}\, \expectedsymbol{\action \sim \latentpolicy\fun{\sampledot \mid \beliefencoder^{*}\fun{\history}}} \expectedsymbol{\state \sim \beliefupdate^{*}\fun{\history}} \expectedsymbol{\state' \sim \probtransitions\fun{\sampledot \mid \state, \action}} \expectedsymbol{\observation' \sim \observationfn\fun{\sampledot \mid \state', \action}} \diracimpulse{\history'}{\history \cdot \action \cdot \observation'} \cdot \diracimpulse{\belief'}{\beliefupdate^{*}\fun{\history'}} \cdot \diracimpulse{\latentbelief'}{\beliefencoder^{*}\fun{\history'}} \; d\history' \tag{by definition of $\left[ \langle \belief', \latentbelief' \rangle \right]_{\pomdpbisim}$} \\
        =& 
        \expectedsymbol{\action \sim \latentpolicy\fun{\sampledot \mid \beliefencoder^{*}\fun{\history}}} \expectedsymbol{\state \sim \beliefupdate^{*}\fun{\history}} \expectedsymbol{\state' \sim \probtransitions\fun{\sampledot \mid \state, \action}} \expectedsymbol{\observation' \sim \observationfn\fun{\sampledot \mid \state', \action}}  \diracimpulse{\belief'}{\beliefupdate^{*}\fun{\history \cdot \action \cdot \observation'}} \cdot \diracimpulse{\latentbelief'}{\beliefencoder^{*}\fun{\history \cdot \action \cdot \observation'}} \\
        =& 
        \expectedsymbol{\action \sim \latentpolicy\fun{\sampledot \mid \latentbelief}} \expectedsymbol{\state \sim \belief} \expectedsymbol{\state' \sim \probtransitions\fun{\sampledot \mid \state, \action}} \expectedsymbol{\observation' \sim \observationfn\fun{\sampledot \mid \state', \action}}  \diracimpulse{\belief'}{\beliefupdate\fun{\belief, \action, \observation'}} \cdot \diracimpulse{\latentbelief'}{\beliefencoder\fun{\latentbelief, \action, \observation'}} \tag{since $\history \pomdpbisim \langle \belief, \latentbelief \rangle$} \\
        =& \probtransitions_{\latentpolicy^{\diamondsuit}}\fun{\belief', \latentbelief' \mid \belief, \latentbelief} \\
        =& \probtransitions_{\latentpolicy^{\diamondsuit}}\fun{\left[ \langle \belief', \latentbelief' \rangle \right]_{\pomdpbisim} \mid \belief, \latentbelief}
    \end{align*}
\end{enumerate}
By \ref{enum:reward-bisim} and \ref{enum:transition-bisim}, we have that $\indexof{\mdp}{\historydistribution}$ and $\beliefmdp$ are in bisimulation under the equivalence relation $\pomdpbisim$, when the policies $\latentpolicy^{\clubsuit}$ and $\latentpolicy^{\diamondsuit}$ are respectively executed in the two models.
\end{proof}
\begin{corollary}\label{corr:expected-return-bisim}
The agent behaviors, formulated through the expected return, that are obtained by executing the policies respectively in the two models are the same:
$\expectedsymbol{\latentpolicy^{\clubsuit}}^{\indexof{\mdp}{\historydistribution}}\left[ \sum_{t = 0}^{\infty} \discount^{t} \cdot \indexof{\rewards}{\historydistribution}\fun{\indexof{\action}{0:t}, \indexof{\observation}{1: t}} \right] = 
\expectedsymbol{\latentpolicy^{\diamondsuit}}^{\beliefmdp}\left[ \sum_{t = 0}^{\infty} \discount^{t} \cdot \beliefrewards\fun{\belief_t, \action_t} \right]$.
\end{corollary}
\begin{proof}
Follows directly from \citep{DBLP:conf/popl/LarsenS89,DBLP:journals/ai/GivanDG03}: the bisimulation relation implies 
the equality of the maximum expected return in the two models, where the maximum is taken over the set of all stationary policies of the two models.
Since we consider MCs and not MDPs, the models are purely stochastic, so there is no nondeterministic choice linked to the choice of action, and then there is only one possible expected return, which yields the result.
\end{proof}
Note that we omitted the super script of $\latentpolicy^{\clubsuit}$ in the main text; we directly considered $\latentpolicy$ as a policy conditioned over histories, by using the exact same definition.

\subsection{Existence of a Stationary Distribution over Histories}\label{appendix:proof-stationary}
Now that we have proven that the history unfolding is equivalent to the belief MDP, we now have all the ingredients to prove Lemma~\ref{lem:extended-stationary-histories}.
\begin{proof}
By definition of $\indexof{\mdp}{\historydistribution}$, the execution of $\latentpolicy^{\clubsuit}$ is guaranteed to remain an episodic process. 
Every episodic process is ergodic \citep{DBLP:conf/nips/Huang20}, there is thus a unique stationary distribution $\historydistribution_{\latentpolicy^{\clubsuit}} = \lim_{t \to \infty} \stationary{\latentpolicy^{\clubsuit}}^{t}\fun{\sampledot \mid \star}$ defined over the state space of $\indexof{\mdp}{\historydistribution}$.
This stationary distribution is thus the limiting distribution over the histories of $\pomdp$ when the latter operates under $\latentpolicy$, or equivalently, the limiting distribution of the MC $\beliefmdp^{\latentpolicy^{\diamondsuit}}$ by Corollary~\ref{corr:expected-return-bisim}.
\end{proof}

\section{Discrete Latent Variables and Temperatures}\label{appendix:temperature}
As mentioned in Section~\ref{sec:latent-space-modeling}, the optimization process of the WAE-MDP relies on a temperature parameter, $\temperature \in \mathopen(0, 1\mathclose)$.
The latter controls the continuity of the latent space learned.
The purpose of the parameter is primarily to learn a discrete latent space model: precisely, we use continuous relaxation of discrete random variables \citepAR{DBLP:conf/iclr/MaddisonMT17}.
This is essentially the Bernoulli version of the Gumbel softmax trick \citepAR{Jang2017CategoricalGumbel-softmax}.
The technique yields re-parameterizable and convex densities, which is compliant with stochastic gradient descent. The zero-temperature limit (i.e., passing from the continuous to the discrete setting) is theoretically enabled via simple rounding of continuous random variables. Furthermore, the logits of discrete densities are guaranteed to be identical to those of their relaxed counterparts.

Alternatively, one could use the \emph{straight-through gradients estimator} \citepAR{DBLP:journals/corr/BengioLC13}, as used by \citetAR{DBLP:conf/nips/OordVK17,DBLP:conf/icml/FajtlAMR20} and \citet{hafner2021mastering}.
This consists in using discrete variables and non-differentiable functions in the forward pass of the input through the neural networks, while continuous variables and surrogate functions are used in the backward pass, i.e., during the backpropagation of the gradients.
This yields low-variance, but biased gradients.
In contrast, continuous relaxations allow to interpolate between these phenomena: a higher temperature produces low-variance but biased gradients, while lowering the temperature to zero increases the variance but ends up producing unbiased gradients. 

Therefore, at any time, the discrete densities (zero-temperature limit) of the WAE-MDP are used, except when the gradients of the objective are computed, where their relaxed counterpart are used $(\temperature > 0)$.
As such, the WBU training procedure follows the same principle as WAE-MDPs: using continuous relaxation of the discrete random variables (precisely, multivariate Bernoullis) when the belief loss (Eq.~\ref{eq:dkl}) is minimized while the actual discrete (autoregressive) density is used otherwise.
We chose temperature values by following the guidelines from the original paper \citepAR{DBLP:conf/iclr/MaddisonMT17}. In the latter, it is mentioned that setting up annealing schemes (to the zero temperature limit) is a good practice but is not necessary for obtaining good results, which is confirmed experimentally in the experimental evaluation of WAE-MDPs \citep[Appendix~B.8]{delgrange2023wasserstein}.

\section{Value Difference Bounds}~\label{appendix:value-diff-bounds}
This section is dedicated to proving Theorems~\ref{thm:value-diff-bounds} and~\ref{thm:value-diff-bound-representation}.
Both Theorems bound the value difference of histories, in the original and latent space models via our local and belief losses, to provide model and representation quality guarantees.
Before proving the Theorems, we first formally define the \emph{value function} of any POMDP, and then illustrate intuitively the meaning of each loss used to bound the value differences.
\subsection{Value Functions}
We start by formally defining the value function of any MDP.
\begin{definition}[Value function]\label{def:values}
Let $\mdp = \mdptuple$ be an MDP, and $\policy$ be a policy for $\mdp$.
Write $\mdp[\state]$ for the MDP obtained by replacing $\sinit$ by $\state \in \states$.
Then, the value of the state $\state \in \states$ is defined as the expected return obtained from that state by running $\policy$, i.e., $\values{\policy}{}{\state} = \expectedsymbol{\policy}^{\mdp[\state]}\left[\sum_{t = 0}^\infty \discount^t \cdot \rewards\fun{\state_t, \action_t}\right]$.
Let $\mdp^{\policy} = \langle \states_{\policy}, \probtransitions_{\policy}, \rewards_{\policy}, \sinit, \discount \rangle$ be the Markov Chain induced by $\policy$ (cf. Definition~\ref{def:markov-chain}).
Then, the value function can be defined as the unique solution of the Bellman's equation \citep{DBLP:books/wi/Puterman94}:
$
    \values{\policy}{}{\state} = \rewards_{\policy}\fun{\state} + \expected{\state' \sim \probtransitions_{\policy}\fun{\state}} {\discount \cdot \values{\policy}{}{\state'}}.
$
The typical goal of an RL agent is to learn a policy $\policy^{\star}$ that maximizes the value of the initial state of $\mdp$: $\max_{\policy^{\star}} \values{\policy^{\star}}{}{\sinit}$.
\end{definition}

\begin{property}[POMDP values]\label{prop:pomdp-values}
We obtain the value function of any POMDP $\pomdp = \pomdptuple$ by considering the values obtained in its belief MDP $\beliefmdp = \beliefmdptuple$.
Therefore, the value of any history $\history \in \histories$ is obtained by mapping $\history$ to the belief space: let $\policy$ be a policy conditioned on the beliefs of $\pomdp$, then we write $\values{\policy}{}{\history}$ for $\values{\policy}{}{\beliefupdate^{*}\fun{\history}}$.
Therefore, we have in particular for any latent policy $\latentpolicy \colon \latentbeliefs \to \distributions{\actions}$:
\begin{align*}
\values{\latentpolicy}{}{\history}
&= \E{}_{\latentpolicy^{\diamondsuit}}^{\beliefmdp[\beliefupdate^{*}\fun{\history}]}\left[ \sum_{t = 0}^{\infty} \discount^{t} \cdot \beliefrewards\fun{\belief_t, \action_t}\right] \tag{cf. Lemma~\ref{lem:bisim}  for definitions of $\latentpolicy^{\diamondsuit}$ and $\latentpolicy^{\clubsuit}$}\\
&= \E{}_{\latentpolicy^{\clubsuit}}^{\indexof{\mdp}{\historydistribution}[\history]}\left[ \sum_{t=0}^{\infty} \discount^{t} \cdot \indexof{\rewards}{\historydistribution}\fun{\history_t, \action_t}\right] \tag{cf. Corollary~\ref{corr:expected-return-bisim}}\\
&=\expected{\action \sim \latentpolicy^{\clubsuit}\fun{\sampledot \mid \history}}{\indexof{\rewards}{\historydistribution}\fun{\history, \action} + \expected{\history' \sim \indexof{\probtransitions}{\historydistribution}\fun{\sampledot \mid \history, \action}}{\discount \cdot \values{\latentpolicy}{}{\history'}}} \tag{by Definition~\ref{def:values}} \\
&=\expected{\action \sim \latentpolicy\fun{\sampledot \mid \beliefencoder^{*}\fun{\history}}}{\indexof{\rewards}{\historydistribution}\fun{\history, \action} + \expected{\history' \sim \indexof{\probtransitions}{\historydistribution}\fun{\sampledot \mid \history, \action}}{\discount \cdot \values{\latentpolicy}{}{\history'}}} \tag{by definition of $\latentpolicy^{\clubsuit}$} \\
&=\expectedsymbol{\action \sim \latentpolicy\fun{\sampledot \mid \beliefencoder^{*}\fun{\history}}} \, \expected{\state \sim \beliefupdate^{*}\fun{\history}}{\rewards\fun{\state, \action} + \expectedsymbol{\state' \sim \probtransitions\fun{\sampledot \mid \state, \action}}\expected{\observation' \sim \observationfn\fun{\sampledot \mid \state', \action}}{\discount \cdot \values{\latentpolicy}{}{\history \cdot \action \cdot \observation'}}}. \tag{by definition of $\indexof{\mdp}{\historydistribution}$}
\end{align*}
Similarly, we write $\latentvaluessymbol{\latentpolicy}{}\,$ for the values of a latent POMDP $\latentpomdp$.
\end{property}

\subsection{Local and Belief Losses}\label{appendix:losses}

\begin{figure}
\begin{subfigure}[c]{.45\textwidth}
    \centering
    \includegraphics[width=.875\textwidth]{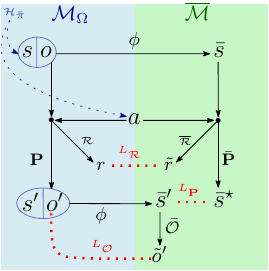}
    \caption{%
    Optimization of the latent space model parameters (i.e., $\latentrewards, \latentprobtransitions$, and $\latentobservationfn$) by minimizing local losses.
    }
    \label{subfig:latent-flow-local-losses}
\end{subfigure}
\hfill
\begin{subfigure}[c]{.45\textwidth}
    \centering
    \includegraphics[width=.875\textwidth]{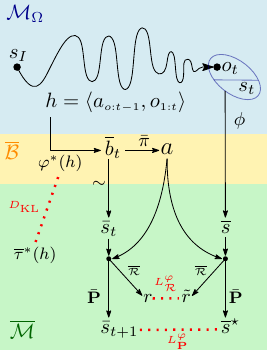}
    \caption{Optimization of the belief encoder $\varphi$ by minimizing the (proxy) belief loss, as well as the reward and transition regularizers.}
    \label{subfig:latent-flow-belief-losses}
\end{subfigure}
\caption{%
Latent flows used to compute the local and belief losses. Arrows represent (stochastic) mappings, the original state-observation (resp. latent state) space is spread along the blue (resp. green) area, and the latent belief space is spread along the yellow area.
Distances (and discrepancies) are depicted in red.
Notice that the blue area corresponds to the state-observation space $\states_{\observations}$, which is accessible during training (Assumption~\ref{assumption:access-state}).
}
\label{fig:latent-flow}
\end{figure}

Theorems~\ref{thm:value-diff-bounds} and~\ref{thm:value-diff-bound-representation} involve the minimization of \emph{local} ($\localrewardloss{}, \localtransitionloss{}, \observationloss{}$) and \emph{belief} ($\beliefloss{}, \onpolicytransitionloss{}, \onpolicyrewardloss{}$) losses.
We intuitively describe how these losses are minimized via the latent flows depicted in Fig.~\ref{fig:latent-flow}.

The procedure allowing to minimize the local losses is depicted in Fig.~\ref{subfig:latent-flow-local-losses}.
At each step, a state $\state$, an observation $\observation$ of $\state$, and an action $\action$ are drawn from the distribution $\historydistribution_{\latentpolicy}$ of experiences encountered while executing $\latentpolicy$.
First, $\tuple{\state, \observation}$ is mapped to the latent space via the state embedding function of the WAE-MDP: $\embed\fun{\state, \observation} = \latentstate$.
Then, the action $\action$ is executed both in the original and latent space models (respectively from $\tuple{\state, \observation}$ and $\latentstate$), which allows to quantify the distance between the next reward and transition produced in the two models. %
Finally, the original model transitions to the next state-observation pair $\tuple{\state', \observation'}$, and mapping it again to the latent space through $\embed\fun{\state', \observation'} = \latentstate'$ allows to quantify the distance between the original observation $\observation'$ and the one that is produced in the latent space model, from $\latentstate'$ via $\tilde{\observation}' \sim \latentobservationfn\fun{\sampledot \mid \latentstate'}$.

The procedure allowing to minimize the belief losses is depicted in Fig.~\ref{subfig:latent-flow-belief-losses}.
This time, the optimization is performed \emph{on-policy}, which means that it is performed while executing the policy in the original environment.
At time step $t \geq 1$, the current history $\history$ ends up in the observation $\observation_t$ of state $\state_t$.
First, the discrepancy between the latent belief obtained via our belief encoder $\latentbelief_t = \beliefencoder^{*}\fun{\history}$ and the one obtained via the true latent belief update function $\latentbelief' = \latentbeliefupdate^{*}\fun{\history}$ is evaluated.
Second, we compute the reward and transition regularizers by minimizing the distance between rewards and transitions produced from believed states $\latentstate_t \sim \latentbelief_t$ (i.e., states expected from the current belief $\latentbelief_t$) and those produced by mapping the current state-observation pair into the latent space, via $\embed\fun{\state_t, \observation_t} = \latentstate$, when the action $\action\sim\latentpolicy\fun{\sampledot\mid\latentbelief_t}$ is produced.
Finally, $\action$ is executed in the original environment and the process is repeated until the end of the episode.

\subsection{Warm Up: Some Wasserstein Properties}
In the following, we elaborate on properties and definitions related to the Wasserstein metrics that will be useful to prove the main claims.
In particular, Wasserstein can be reformulated as the maximum mean discrepancy of \emph{$1$-Lipschitz functions}.
The claim holds when the temperature is at the zero-limit, i.e., when the discrete variables and densities of the models are used (cf. Appendix~\ref{appendix:temperature}).
This further makes the distance metric $\latentdistance$ associated with the latent state space go to the discrete metric $\condition{\neq} \colon \measurableset \to \set{1, 0}$  \citep{delgrange2023wasserstein}, formally defined as $\condition{\neq}\fun{x1, x2} = 1$ iff $x_1 \neq x_2$.
\begin{definition}[Lipschitz continuity]
Let $\measurableset, \varmeasurableset$ be two measurable set and $f \colon \measurableset \to \varmeasurableset$ be a function mapping elements from $\measurableset$ to $\varmeasurableset$.
If otherwise specified, we consider that $f$ is real-valued function, i.e., $\varmeasurableset = \R$.
Assume that $\measurableset$ is equipped with a metric $\distance \colon \measurableset \to \mathopen[0, \infty\mathclose)$.
Then, given a constant $K \geq 0$, we say that $f$ is $K$-Lipschitz iff, for any $x_1, x_2 \in \measurableset$, $\abs{f\fun{x_1} - f\fun{x_2}} \leq K \cdot \distance\fun{x_1, x_2}$.
We write $\Lipschf{d}^K$ for the set of $K$-Lipschitz functions.
\end{definition}
\begin{definition}[Wasserstein dual]
The Kantorovich-Rubinstein duality \citep{Villani2009} allows formulating the Wasserstein distance between $P$ and $Q$ as
    $
        \wassersteindist{\distance}{P}{Q} = \sup_{f \in \Lipschf{\distance}^{\,1}} \left|{\expectedsymbol{x \sim P} f\fun{x} - \expectedsymbol{y \sim Q} f\fun{y}}\right|.
    $
\end{definition}

\begin{property}[Lipschitz constant]\label{prop:lipschitz-constant}
Let $f \colon \measurableset \to \R$, so that $\distance$ is a metric on $\measurableset$.
Assume that $f$ is $K$-Lipschitz, i.e., $f \in \Lipschf{\distance}^{K}$, then for any two distributions $P, Q \in \distributions{\measurableset}$, $\abs{\expectedsymbol{x_1 \sim P}f\fun{x_1} - \expectedsymbol{x_2 \sim Q} f\fun{x_2}} \leq K \cdot \wassersteindist{\distance}{P}{Q}$. %

In particular, for \textbf{any} bounded function $g \colon \measurableset \to Y$ with $Y \subseteq \R$, when the distance metric associated with $\measurableset$ is the discrete metric, i.e., $\distance = \condition{\neq}$, we have $\abs{\expectedsymbol{x_1 \sim P}g\fun{x_1} - \expectedsymbol{x_2 \sim Q} g\fun{x_2}} \leq K_Y \cdot \wassersteindist{\condition{\neq}}{P}{Q} = K_Y \cdot \dtv{P}{Q}$, where $K_Y \geq \sup_{x \in \measurableset} \abs{g\fun{x}}$ (see, e.g., \citealt[Sect.~6]{DBLP:conf/icml/GeladaKBNB19}, for a discussion).
\end{property}
Property~\ref{prop:lipschitz-constant} intuitively implies the emergence of the $\KV$ constant in the Theorem's inequality: we know that the latent value function is bounded by $\nicefrac{\sup_{\latentstate, a}\abs{\latentrewards_{\decoderparameter}\fun{\latentstate, \action}}}{1 - \discount}$, so given two distributions $P, Q$ over $\latentstates$, the maximum mean discrepancy of the latent value function is bounded by $\KV \cdot \wassersteindist{\latentdistance}{P}{Q}$ when the temperature goes to zero.

Finally, since the value difference is computed in expectation, we introduce the following useful property:

\begin{lemma}[Wasserstein in expectation]\label{lem:wasserstein-expecation}
For any $f \colon \varmeasurableset \times \measurableset \to \R$ so that $\measurableset$ is equipped with the metric $\distance$, consider the function $g_y \colon \measurableset \to \R$ defined as $g_y\fun{x} = f\fun{y, x}$.
Assume that for any $y \in \varmeasurableset$, $g_y$ is $K$-Lipschitz, i.e., $g_y \in \Lipschf{\distance}^K$. 
Then, let $\replaybuffer \in \distributions{\varmeasurableset}$ be a distribution over $\varmeasurableset$ and $P, Q \in \distributions{\measurableset}$ be two distributions over $\measurableset$, we have 
$\expectedsymbol{y \sim \replaybuffer} \abs{\expectedsymbol{ x_1 \sim P}{f\fun{y, x_1}} - \expectedsymbol{x_2 \sim Q}{f\fun{y, x_2}}} \leq K \cdot \wassersteindist{\distance}{P}{Q}$.
\end{lemma}
\begin{proof}
The proof is straightforward by construction of $g_y$:
\begin{align*}
    &\expectedsymbol{y \sim \replaybuffer} \abs{\expectedsymbol{x_1 \sim P}{f\fun{y, x_1}} - \expectedsymbol{x_2 \sim Q}{f\fun{y, x_2}}} \\
    = & \expectedsymbol{y \sim \replaybuffer} \abs{\expectedsymbol{x_1 \sim P}{g_y\fun{x_1}} - \expectedsymbol{x_2 \sim Q}{g_y\fun{x_2}}} \\
    \leq & \expected{y \sim \replaybuffer}{K \cdot \wassersteindist{\distance}{P}{Q}} \tag{by Property~\ref{prop:lipschitz-constant}, since $g_y$ is $K$-Lipschitz}\\
    = & K \cdot \wassersteindist{\distance}{P}{Q}
\end{align*}
\end{proof}

\subsection{Model Quality Bound: Time to Raise your Expectations}
Let us restate Theorem~\ref{thm:value-diff-bounds}:
\begin{theorem}\label{thm:value-diff-bounds-extended}
Let $\pomdp$, $\latentpomdp_{\decoderparameter}$, and $\latentpolicy \colon \latentbeliefs \to \distributions{\actions}$ be respectively the original and the latent POMDP, as well as the latent policy of Lemma~\ref{lem:extended-stationary-histories}, so that the latent POMDP is learned through a WAE-MDP, via the minimization of the local losses  $\localrewardloss{\historydistribution_{\latentpolicy}}, \localtransitionloss{\history_{\latentpolicy}}$. %
Assume that the WAE-MDP is at the zero-temperature limit (i.e., ${\temperature \to 0}$, see Appendix~\ref{appendix:temperature}) and let $K_{\overbar{\mathrm{V}}} = \nicefrac{\parallel{\latentrewards}\parallel_\infty}{1 - \discount}$, then for any such latent policy $\latentpolicy$, 
the values of $\,\pomdp$ and $\latentpomdp_{\decoderparameter}$ are guaranteed to be bounded by the local losses in average:
\begin{equation}
    \expectedsymbol{\history \sim \historydistribution_{\latentpolicy}}\abs{\values{\latentpolicy}{}{\history} - \latentvalues{\latentpolicy}{}{\history}} \leq \frac{\localrewardloss{\historydistribution_{\latentpolicy}} + \onpolicyrewardloss{\historydistribution_{\latentpolicy}} + \Rmax \beliefloss{\historydistribution_{\latentpolicy}} + \discount \KV \cdot \fun{ \localtransitionloss{\historydistribution_{\latentpolicy}} +  \onpolicytransitionloss{\historydistribution_{\latentpolicy}} + \beliefloss{\historydistribution_{\latentpolicy}} + \observationloss{\historydistribution_{\latentpolicy}}}}{1 - \discount}.
\end{equation}
\end{theorem}
\begin{proof}
The plan of the proof is as follows:
\begin{enumerate}
    \item We exploit the fact that the value function can be defined as the fixed point of the Bellman's equations;
    \item We repeatedly apply the triangular and the Jenson's inequalities to end up with inequalities which reveal mean discrepancies for either rewards or value functions;
    \item We exploit the fact that the temperature goes to zero to bound those discrepancies by Wasserstein (see Porperty~\ref{prop:lipschitz-constant} and the related discussion);
    \item The last two points allow highlighting the $L_1$ norm and Wasserstein terms in the local and belief losses;
    \item Finally, we set up the inequalities to obtain a discounted next value difference term, and we exploit the stationary property of $\historydistribution_{\latentpolicy}$ to fall back on the original, discounted, absolute value difference term;
    \item Putting all together, we end up with an inequality only composed of constants, multiplied by losses that we aim at minimizing.
\end{enumerate}

Concretely, the absolute value difference can be bounded by:
\begin{align*}
    & \expectedsymbol{\history \sim \historydistribution_{\latentpolicy}}{\left|\values{\latentpolicy}{}{\history} - \latentvalues{\latentpolicy}{}{\history}\right|}\\
    = & 
    \expectedsymbol{\history \sim \historydistribution_{\latentpolicy}}\Bigg|\expectedsymbol{\state \sim \beliefupdate^*\fun{\history}} \expected{\action \sim \latentpolicy\fun{\sampledot \mid \beliefencoder^*\fun{\history}}}{\rewards\fun{\state, \action} + \discount \expectedsymbol{\state' \sim \probtransitions\fun{\sampledot \mid \state, \action}}\expectedsymbol{\observation' \sim \observationfn\fun{\sampledot \mid \state', \action}} \values{\latentpolicy}{}{\history \cdot \action \cdot \observation'}}
    \\
    & \quad \quad \quad - \expectedsymbol{\latentstate \sim \latentbeliefupdate^{*}\fun{\history}} \expected{\action \sim \latentpolicy\fun{\sampledot \mid \beliefencoder^*\fun{\history}}}{
    \latentrewards_{\decoderparameter}\fun{\latentstate, \action} + \discount \expectedsymbol{\latentstate' \sim \latentprobtransitions_{\decoderparameter}\fun{\sampledot \mid \latentstate, \action}}
    \expectedsymbol{\observation' \sim \latentobservationfn_{\decoderparameter}\fun{\sampledot \mid \latentstate'}}\latentvalues{\latentpolicy}{}{\history \cdot \action \cdot \observation'
    }}\Bigg|
    \tag{see Property~\ref{prop:pomdp-values}}
    \\
    = & 
    \expectedsymbol{\history \sim \historydistribution_{\latentpolicy}}  \color{blue}\expectedsymbol{\action \sim \latentpolicy\fun{\sampledot \mid \beliefencoder^*\fun{\history}}}\Bigg|\expected{\state \sim \beliefupdate^*\fun{\history}} {\rewards\fun{\state, \action} + \discount \expectedsymbol{\state' \sim \probtransitions\fun{\sampledot \mid \state, \action}}\expectedsymbol{\observation' \sim \observationfn\fun{\sampledot \mid \state', \action}} \values{\latentpolicy}{}{\history \cdot \action \cdot \observation'}}
    \\
    & \qquad \qquad \qquad \qquad \quad -  \color{blue}\expected{\latentstate \sim \latentbeliefupdate^{*}\fun{\history}}{
    \latentrewards_{\decoderparameter}\fun{\latentstate, \action} + \discount \expectedsymbol{\latentstate' \sim \latentprobtransitions_{\decoderparameter}\fun{\sampledot \mid \latentstate, \action}}
    \expectedsymbol{\observation' \sim \latentobservationfn_{\decoderparameter}\fun{\sampledot \mid \latentstate'}}\latentvalues{\latentpolicy}{}{\history \cdot \action \cdot \observation'
    }}\Bigg|
    \tag{ \color{blue}Jensen's inequality}
    \\
    \leq & \expectedsymbol{\history \sim \historydistribution_{\latentpolicy}}\expectedsymbol{\action \sim \latentpolicy\fun{\sampledot \mid \beliefencoder^*\fun{\history}}}
    \Bigg[ \color{blue}\;
    \abs{\expectedsymbol{\state \sim \beliefupdate^{*}\fun{\history}} \rewards\fun{\state, \action} - \expectedsymbol{\latentstate \sim \latentbeliefupdate^{*}\fun{\history}} \latentrewards_{\decoderparameter}\fun{\latentstate, \action}} \\
    &  \color{blue}\quad \quad + \discount \abs{\expectedsymbol{\state \sim \beliefupdate^{*}\fun{\history}}
    \expectedsymbol{\state' \sim \probtransitions\fun{\sampledot \mid \state, \action}}\expectedsymbol{\observation' \sim \observationfn\fun{\sampledot \mid \state', \action}} \values{\policy}{}{\history \cdot \action \cdot \observation'} - \expectedsymbol{\latentstate \sim \latentbeliefupdate^{*}\fun{\history}}\expectedsymbol{\latentstate' \sim \latentprobtransitions_{\decoderparameter}\fun{\sampledot \mid \latentstate, \action}}\expectedsymbol{\observation' \sim \latentobservationfn_{\decoderparameter}\fun{\sampledot \mid \latentstate'}} \latentvalues{\latentpolicy}{}{\history\cdot\action\cdot\observation'}}
    \;\Bigg] \tag{ \color{blue}Triangular inequality}
\end{align*}
For the sake of clarity, we split the inequality in two parts.

\smallparagraph{Part 1: Reward bounds}%
\begin{align*}
 & \expectedsymbol{\history \sim \historydistribution_{\latentpolicy}} \expectedsymbol{\action \sim \latentpolicy\fun{\sampledot \mid \beliefencoder^{*}\fun{\history}}}
 \abs{\expectedsymbol{\state \sim \beliefupdate^{*}\fun{\history}}\rewards\fun{\state, \action} - \expectedsymbol{\latentstate \sim \latentbeliefupdate^{*}\fun{\history}} \latentrewards_{\decoderparameter}\fun{\latentstate, \action}} \\
 = & \expectedsymbol{\history, {\observation} \sim \historydistribution_{\latentpolicy}} \expectedsymbol{\action \sim \latentpolicy\fun{\sampledot \mid \beliefencoder^{*}\fun{\history}}}
 \abs{\expected{\state \sim \beliefupdate^{*}\fun{\history}}{\rewards\fun{\state, \action} - \latentrewards_{\decoderparameter}\fun{\embed_{\encoderparameter}\fun{\state, \observation}, \action}} + \expectedsymbol{\state \sim \beliefupdate^{*}\fun{\history}}\expected{\latentstate \sim \latentbeliefupdate^{*}\fun{\history}}{\latentrewards_{\decoderparameter}\fun{\embed_{\encoderparameter}\fun{\state, \observation}, \action} - \latentrewards_{\decoderparameter}\fun{\latentstate, \action}}} \\
 &\begin{aligned}
     &\text{($\observation$ is the last observation of $\history$;}\\
     &\text{the state embedding function $\embed_{\encoderparameter}$ that links the original and latent state spaces comes into play)}
 \end{aligned} \\
 \leq & \expectedsymbol{\history, \observation \sim \historydistribution_{\latentpolicy}} \expected{\action \sim \latentpolicy\fun{\sampledot \mid \beliefencoder^{*}\fun{\history}}}{
 { \color{blue}\abs{\expected{\state \sim \beliefupdate^{*}\fun{\history}}{\rewards\fun{\state, \action} - \latentrewards_{\decoderparameter}\fun{\embed_{\encoderparameter}\fun{\state, \observation}, \action}}} + \abs{\expectedsymbol{\state \sim \beliefupdate^{*}\fun{\history}}\expected{\latentstate \sim \latentbeliefupdate^{*}\fun{\history}}{\latentrewards_{\decoderparameter}\fun{\embed_{\encoderparameter}\fun{\state, \observation}, \action} - \latentrewards_{\decoderparameter}\fun{\latentstate, \action}}}}} \tag{ \color{blue}Triangular inequality}\\
 \leq & \expectedsymbol{\history, \observation \sim \historydistribution_{\latentpolicy}} \expected{\action \sim \latentpolicy\fun{\sampledot \mid \beliefencoder^{*}\fun{\history}}}
 {
 {\color{blue}\expectedsymbol{\state \sim \beliefupdate^{*}\fun{\history}}
 \abs{{\rewards\fun{\state, \action} - \latentrewards_{\decoderparameter}\fun{\embed_{\encoderparameter}\fun{\state, \observation}, \action}}} }+ \abs{\expectedsymbol{\state \sim \beliefupdate^{*}\fun{\history}}\expected{\latentstate \sim \latentbeliefupdate^{*}\fun{\history}} {\latentrewards_{\decoderparameter}\fun{\embed_{\encoderparameter}\fun{\state, \observation}, \action} - \latentrewards_{\decoderparameter}\fun{\latentstate, \action}}}}
 \tag{ \color{blue}Jensen's inequality}\\
 = & \expectedsymbol{\history, \observation \sim \historydistribution_{\latentpolicy}} \expectedsymbol{\action \sim \latentpolicy\fun{\sampledot \mid \beliefencoder^{*}\fun{\history}}}
 \expectedsymbol{\state \sim \beliefupdate^{*}\fun{\history}}
 \abs{{\rewards\fun{\state, \action} - \latentrewards_{\decoderparameter}\fun{\embed_{\encoderparameter}\fun{\state, \observation}, \action}}} \\
 & + \expectedsymbol{\history, \observation \sim \historydistribution_{\latentpolicy}} \expectedsymbol{\action \sim \latentpolicy\fun{\sampledot \mid \beliefencoder^{*}\fun{\history}}}
 \abs{\expectedsymbol{\state \sim \beliefupdate^{*}\fun{\history}}\expected{\latentstate \sim \latentbeliefupdate^{*}\fun{\history}} {\latentrewards_{\decoderparameter}\fun{\embed_{\encoderparameter}\fun{\state, \observation}, \action} - \latentrewards_{\decoderparameter}\fun{\latentstate, \action}}} \\
 =& { \color{blue}\localrewardloss{\historydistribution_{\latentpolicy}}} + \expectedsymbol{\history, \observation \sim \historydistribution_{\latentpolicy}}\expectedsymbol{\action \sim \latentpolicy\fun{\sampledot \mid \beliefencoder^*\fun{\history}}}\abs{
 \expectedsymbol{\state \sim \beliefupdate^{*}\fun{\history}}\expected{\latentstate \sim \latentbeliefupdate^{*}\fun{\history}}{\latentrewards_{\decoderparameter}\fun{\embed_{\encoderparameter}\fun{\state, \observation}, \action} - \latentrewards_{\decoderparameter}\fun{\latentstate, \action}}}
 \tag{ \color{blue}by definition of $\localrewardloss{\historydistribution_{\latentpolicy}}$}\\
 = & \localrewardloss{\historydistribution_{\latentpolicy}} + \expectedsymbol{\history, \observation \sim \historydistribution_{\latentpolicy}}\expectedsymbol{\action \sim \latentpolicy\fun{\sampledot \mid \beliefencoder^*\fun{\history}}}\abs{\expectedsymbol{\state \sim \beliefupdate^{*}\fun{\history}}\expectedsymbol{\latentstate \sim \latentbeliefupdate^{*}\fun{\history}}{ \color{blue} \expected{\latentstate_{\bot} \sim \beliefencoder^{*}\fun{\history}}{\left[\latentrewards_{\decoderparameter}\fun{\embed_{\encoderparameter}\fun{\state, \observation}, \action} - \latentrewards_{\decoderparameter}\fun{\latentstate_{\bot}, \action}\right] + \left[\latentrewards_{\decoderparameter}\fun{\latentstate_{\bot}, \action}- \latentrewards_{\decoderparameter}\fun{\latentstate, \action}\right]}}}
 \tag{ \color{blue}the belief encoder $\beliefencoder$ comes into play}\\
 \end{align*}
 \begin{align*}
 = & \localrewardloss{\historydistribution_{\latentpolicy}} + \expectedsymbol{\history, \observation \sim \historydistribution_{\latentpolicy}}\expectedsymbol{\action \sim \latentpolicy\fun{\sampledot \mid \beliefencoder^*\fun{\history}}}\left|\expectedsymbol{\state \sim \beliefupdate^{*}\fun{\history}}\expectedsymbol{\latentstate \sim \beliefencoder^{*}\fun{\history}}\left[\latentrewards_{\decoderparameter}\fun{\embed_{\encoderparameter}\fun{\state, \observation}, \action} - \latentrewards_{\decoderparameter}\fun{\latentstate, \action}\right] \right.\\
 &\left.\qquad\qquad\qquad\qquad\qquad\qquad+  \expectedsymbol{\latentstate \sim \latentbeliefupdate^{*}\fun{\history}}\expectedsymbol{\latentstate_{\bot} \sim \beliefencoder^{*}\fun{\history}}\left[\latentrewards_{\decoderparameter}\fun{\latentstate_{\bot}, \action}- \latentrewards_{\decoderparameter}\fun{\latentstate, \action}\right]\right|\\
 \leq & \localrewardloss{\historydistribution_{\latentpolicy}} + \expectedsymbol{\history, \observation \sim \historydistribution_{\latentpolicy}}\expectedsymbol{\action \sim \latentpolicy\fun{\sampledot \mid \beliefencoder^*\fun{\history}}}  \color{blue}\left[\abs{\expectedsymbol{\state \sim \beliefupdate^{*}\fun{\history}}\expectedsymbol{\latentstate \sim \beliefencoder^{*}\fun{\history}}{\left[\latentrewards_{\decoderparameter}\fun{\embed_{\encoderparameter}\fun{\state, \observation}, \action} - \latentrewards_{\decoderparameter}\fun{\latentstate, \action}\right]}} \right. \\
 &\left. { \color{blue}\qquad\qquad\qquad\qquad\qquad\qquad+ \abs{  \expectedsymbol{\latentstate \sim \latentbeliefupdate^{*}\fun{\history}}\expectedsymbol{\latentstate_{\bot} \sim \beliefencoder^{*}\fun{\history}}\left[\latentrewards_{\decoderparameter}\fun{\latentstate_{\bot}, \action}- \latentrewards_{\decoderparameter}\fun{\latentstate, \action}\right]} }\right]
 \tag{ \color{blue}Triangular inequality} \\
 \leq & \localrewardloss{\historydistribution_{\latentpolicy}} + \expectedsymbol{\history, \observation \sim \historydistribution_{\latentpolicy}}\expectedsymbol{\action \sim \latentpolicy\fun{\sampledot \mid \beliefencoder^*\fun{\history}}}\left[ \color{blue}\expectedsymbol{\state \sim \beliefupdate^{*}\fun{\history}}\expectedsymbol{\latentstate \sim \beliefencoder^{*}\fun{\history}}{\abs{\latentrewards_{\decoderparameter}\fun{\embed_{\encoderparameter}\fun{\state, \observation}, \action} - \latentrewards_{\decoderparameter}\fun{\latentstate, \action}}} \right.\\
 & \left.  \qquad\qquad\qquad\qquad\qquad\qquad+ \abs{  \expectedsymbol{\latentstate \sim \latentbeliefupdate^{*}\fun{\history}}\expectedsymbol{\latentstate_{\bot} \sim \beliefencoder^{*}\fun{\history}}\left[\latentrewards_{\decoderparameter}\fun{\latentstate_{\bot}, \action}- \latentrewards_{\decoderparameter}\fun{\latentstate, \action}\right]}\right]
 \tag{ \color{blue}Jensen's inequality}\\
 = & \localrewardloss{\historydistribution_{\latentpolicy}} + \expectedsymbol{\history, \observation \sim \historydistribution_{\latentpolicy}}\expectedsymbol{\action \sim \latentpolicy\fun{\sampledot \mid \beliefencoder^*\fun{\history}}}\expectedsymbol{\state \sim \beliefupdate^{*}\fun{\history}}\expectedsymbol{\latentstate \sim \beliefencoder^{*}\fun{\history}}\abs{\latentrewards_{\decoderparameter}\fun{\embed_{\encoderparameter}\fun{\state, \observation}, \action} - \latentrewards_{\decoderparameter}\fun{\latentstate, \action}} \\
 & \quad+ \expectedsymbol{\history, \observation \sim \historydistribution_{\latentpolicy}}\expectedsymbol{\action \sim \latentpolicy\fun{\sampledot \mid \beliefencoder^*\fun{\history}}}\abs{\expectedsymbol{\latentstate \sim \latentbeliefupdate^{*}\fun{\history}}\expectedsymbol{\latentstate_{\bot} \sim \beliefencoder^{*}\fun{\history}}\latentrewards_{\decoderparameter}\fun{\latentstate_{\bot}, \action} - \latentrewards_{\decoderparameter}\fun{\latentstate, \action}}\\
 = & \localrewardloss{\historydistribution_{\latentpolicy}} + { \color{blue}\onpolicyrewardloss{\historydistribution_{\latentpolicy}} }+ \expectedsymbol{\history, \observation \sim \historydistribution_{\latentpolicy}}\expectedsymbol{\action \sim \latentpolicy\fun{\sampledot \mid \beliefencoder^*\fun{\history}}}\abs{\expectedsymbol{\latentstate \sim \latentbeliefupdate^{*}\fun{\history}}\expected{\latentstate_{\bot} \sim \beliefencoder^{*}\fun{\history}}{\latentrewards_{\decoderparameter}\fun{\latentstate_{\bot}, \action} - \latentrewards_{\decoderparameter}\fun{\latentstate, \action}}}
 \tag{ \color{blue}by definition of \onpolicyrewardloss{\historydistribution_{\latentpolicy}}, Eq.~\ref{eq:on-policy-losses}}\\
 \leq & \localrewardloss{\historydistribution_{\latentpolicy}} + \onpolicyrewardloss{\historydistribution_{\latentpolicy}} + \expectedsymbol{\history \sim \historydistribution_{\latentpolicy}}  \color{blue}\Rmax \wassersteindist{\latentdistance}{\latentbeliefupdate^{*}\fun{\history}}{\beliefencoder^{*}\fun{\history}}
 \tag{ \color{blue}as $\temperature \to 0$, by Lem.~\ref{lem:wasserstein-expecation} and Prop.~\ref{prop:lipschitz-constant}}\\
 = & \localrewardloss{\historydistribution_{\latentpolicy}} + \onpolicyrewardloss{\historydistribution_{\latentpolicy}} +   \color{blue}\Rmax \beliefloss{\historydistribution_{\latentpolicy}};
\end{align*}

\smallparagraph{Part 2: Next value bounds}
\begin{align*}
    &&&\discount \cdot \expectedsymbol{\history \sim \historydistribution_{\latentpolicy}}\expectedsymbol{\action \sim \latentpolicy\fun{\sampledot \mid \beliefencoder^{*}\fun{\history}}} \Big|\expectedsymbol{\state \sim \beliefupdate^{*}\fun{\history}}\expectedsymbol{\state' \sim \probtransitions\fun{\sampledot \mid \state, \action}}\expectedsymbol{\observation' \sim \observationfn\fun{\sampledot \mid \state', \action}} \values{\latentpolicy}{}{\history \cdot \action \cdot \observation'} \\&&&
    \qquad\qquad\qquad\qquad\qquad - \expectedsymbol{\latentstate \sim \latentbeliefupdate^{*}\fun{\history}}\expectedsymbol{\latentstate' \sim \latentprobtransitions_{\decoderparameter}\fun{\sampledot \mid \latentstate, \action}}\expectedsymbol{\observation' \sim \latentobservationfn_{\decoderparameter}\fun{\sampledot \mid \latentstate'}}\latentvalues{\latentpolicy}{}{\history \cdot \action \cdot \observation'}\Big|\\
    &=& & \discount \cdot \expectedsymbol{\history \sim \historydistribution_{\latentpolicy}}\expectedsymbol{\action \sim \latentpolicy\fun{\sampledot \mid \beliefencoder^{*}\fun{\history}}}
    \Bigg|
    \expectedsymbol{\state \sim \beliefupdate^{*}\fun{\history}}\expectedsymbol{\state' \sim \probtransitions\fun{\sampledot \mid \state, \action}}\expected{\observation' \sim \observationfn\fun{\sampledot \mid \state', \action}}{\values{\latentpolicy}{}{\history \cdot \action \cdot \observation'} {\color{blue}- \expectedsymbol{\hat{\observation}' \sim \latentobservationfn_{\decoderparameter}\fun{\sampledot \mid \embed_{\encoderparameter}\fun{\state', {\observation}'}}}\latentvalues{\latentpolicy}{}{\history \cdot \action \cdot {\hat{\observation}}'}}}\\
    &&& {\color{blue}+}  \Bigg[\color{blue}\expectedsymbol{\state \sim \beliefupdate^{*}\fun{\history}} \expectedsymbol{\state' \sim \probtransitions\fun{\sampledot \mid \state, \action}}\expectedsymbol{\observation' \sim \observationfn\fun{\sampledot \mid \state', \action}}\expectedsymbol{\hat{\observation}' \sim \latentobservationfn_{\decoderparameter}\fun{\sampledot \mid \embed_{\encoderparameter}\fun{\state', \observation'}}}\latentvalues{\latentpolicy}{}{\history \cdot \action \cdot \hat{\observation}'}  \\ &&& \qquad\qquad\qquad\qquad\qquad\qquad\qquad\qquad\quad-\expectedsymbol{\latentstate \sim \latentbeliefupdate^{*}\fun{\history}}\expectedsymbol{\latentstate' \sim \latentprobtransitions_{\decoderparameter}\fun{\sampledot \mid \latentstate, \action}}\expectedsymbol{\observation' \sim \latentobservationfn_{\decoderparameter}\fun{\sampledot \mid \latentstate'}}\latentvalues{\latentpolicy}{}{\history \cdot \action \cdot \observation'}\Bigg] \Bigg|
     \tag{{\color{blue}the state embedding function $\embed_{\encoderparameter}$ comes into play}, {\color{blue}as well as the latent observation function $\latentobservationfn_{\decoderparameter}$}}\\
     \end{align*}
     \begin{align*}
    &{\color{blue}\leq}&&  \discount \cdot \expectedsymbol{\history \sim \historydistribution_{\latentpolicy}}\expectedsymbol{\action \sim \latentpolicy\fun{\sampledot \mid \beliefencoder^{*}\fun{\history}}}
    {\color{blue}\Bigg|}{\expectedsymbol{\state \sim \beliefupdate^{*}\fun{\history}}\expectedsymbol{\state' \sim \probtransitions\fun{\sampledot \mid \state, \action}}\expected{\observation' \sim \observationfn\fun{\sampledot \mid \state', \action}}{\values{\latentpolicy}{}{\history \cdot \action \cdot \observation'} - \expectedsymbol{\hat{\observation}' \sim \latentobservationfn_{\decoderparameter}\fun{\sampledot \mid \embed_{\encoderparameter}\fun{\state', \observation'}}}\latentvalues{\latentpolicy}{}{\history \cdot \action \cdot \hat{\observation}'}}{\color{blue}\Bigg|}}\\
    &&& \quad \quad +  \discount \cdot \expectedsymbol{\history \sim \historydistribution_{\latentpolicy}}\expectedsymbol{\action \sim \latentpolicy\fun{\sampledot \mid \beliefencoder^{*}\fun{\history}}} {\color{blue}\Bigg|}\expectedsymbol{\state \sim \beliefupdate^{*}\fun{\history}} \expectedsymbol{\state' \sim \probtransitions\fun{\sampledot \mid \state, \action}}\expectedsymbol{\observation' \sim \observationfn\fun{\sampledot \mid \state', \action}}\expectedsymbol{\hat{\observation}' \sim \latentobservationfn_{\decoderparameter}\fun{\sampledot \mid \embed_{\encoderparameter}\fun{\state', \observation'}}}\latentvalues{\latentpolicy}{}{\history \cdot \action \cdot \hat{\observation}'} \\
    &&& \qquad\qquad\qquad\qquad\qquad\qquad- \expectedsymbol{\latentstate \sim \latentbeliefupdate^{*}\fun{\history}}\expectedsymbol{\latentstate' \sim \latentprobtransitions_{\decoderparameter}\fun{\sampledot \mid \latentstate, \action}}\expectedsymbol{\observation' \sim \latentobservationfn_{\decoderparameter}\fun{\sampledot \mid \latentstate'}}\latentvalues{\latentpolicy}{}{\history \cdot \action \cdot \observation'}{\color{blue}\Bigg|}
    \tag{\color{blue}Triangular inequality}\\
    &=&& \discount \cdot \expectedsymbol{\history \sim \historydistribution_{\latentpolicy}}\expectedsymbol{\action \sim \latentpolicy\fun{\sampledot \mid \beliefencoder^{*}\fun{\history}}}
    \abs{\expectedsymbol{\state \sim \beliefupdate^{*}\fun{\history}}\expectedsymbol{\state' \sim \probtransitions\fun{\sampledot \mid \state, \action}}\expected{\observation' \sim \observationfn\fun{\sampledot \mid \state', \action}}{\values{\latentpolicy}{}{\history \cdot \action \cdot \observation'} - \expectedsymbol{\hat{\observation}' \sim \latentobservationfn_{\decoderparameter}\fun{\sampledot \mid \embed_{\encoderparameter}\fun{\state', \observation'}}}\latentvalues{\latentpolicy}{}{\history \cdot \action \cdot \hat{\observation}'}}}\\
    &&&  +  \discount \cdot \expectedsymbol{\history, {\color{darkgreen}\observation} \sim \historydistribution_{\latentpolicy}}\expectedsymbol{\action \sim \latentpolicy\fun{\sampledot \mid \beliefencoder^{*}\fun{\history}}}\left|\expectedsymbol{\state \sim \beliefupdate^{*}\fun{\history}} \Bigg[\expectedsymbol{\state', \observation' \sim \probtransitions_{\observations}\fun{\sampledot \mid \state, {\color{darkgreen}\observation}, \action}}{\expectedsymbol{\hat{\observation}' \sim \latentobservationfn_{\decoderparameter}\fun{\sampledot \mid \embed_{\encoderparameter}\fun{\state', \observation'}}}\latentvalues{\latentpolicy}{}{\history \cdot \action \cdot \hat{\observation}'}} \right. \\
    &&& \left. \qquad\qquad\qquad\qquad\qquad\quad\qquad\qquad {\color{blue}- \expectedsymbol{\latentstate' \sim \latentprobtransitions_{\decoderparameter}\fun{\sampledot \mid \embed_{\encoderparameter}\fun{ \state, {\color{darkgreen}\observation}}, \action}}{\expectedsymbol{\hat{\observation}' \sim \latentobservationfn_{\decoderparameter}\fun{\sampledot \mid \latentstate'}}\latentvalues{\latentpolicy}{}{\history \cdot \action \cdot \hat{\observation}'}}}\Bigg] \right.\\
    &&& \left.\qquad\qquad\qquad\qquad\qquad\quad {\color{blue}+} \left[{\color{blue}\expectedsymbol{\state \sim \beliefupdate^{*}\fun{\history}} \expectedsymbol{\latentstate' \sim \latentprobtransitions_{\decoderparameter}\fun{\sampledot \mid \embed_{\encoderparameter}\fun{\state, {\color{darkgreen}\observation}}, \action}}\expectedsymbol{{\observation}' \sim \latentobservationfn_{\decoderparameter}\fun{\sampledot \mid \latentstate'}}\latentvalues{\latentpolicy}{}{\history \cdot \action \cdot {\observation}'}} \right.\right. \\
    &&& \left.\left.\qquad\qquad\qquad\qquad\qquad\qquad\;\;- \expectedsymbol{\latentstate \sim \latentbeliefupdate^{*}\fun{\history}}\expectedsymbol{\latentstate' \sim \latentprobtransitions_{\decoderparameter}\fun{\sampledot \mid \latentstate, \action}}\expectedsymbol{\observation' \sim \latentobservationfn_{\decoderparameter}\fun{\sampledot \mid \latentstate'}}\latentvalues{\latentpolicy}{}{\history \cdot \action \cdot \observation'}\right]\right|
    \tag{{\color{darkgreen}$\observation$ is the last observation of $\history$}; {\color{blue}the latent MDP dynamics, modeled by $\latentprobtransitions_{\decoderparameter}$, come into play}}\\
    &{\color{blue}\leq}&&  \discount \cdot \expectedsymbol{\history \sim \historydistribution_{\latentpolicy}}\expectedsymbol{\action \sim \latentpolicy\fun{\sampledot \mid \beliefencoder^{*}\fun{\history}}}
    \abs{\expectedsymbol{\state \sim \beliefupdate^{*}\fun{\history}}\expectedsymbol{\state' \sim \probtransitions\fun{\sampledot \mid \state, \action}}\expected{\observation' \sim \observationfn\fun{\sampledot \mid \state', \action}}{\values{\latentpolicy}{}{\history \cdot \action \cdot \observation'} - \expectedsymbol{\hat{\observation}' \sim \latentobservationfn_{\decoderparameter}\fun{\sampledot \mid \embed_{\encoderparameter}\fun{\state', \observation'}}}\latentvalues{\latentpolicy}{}{\history \cdot \action \cdot \hat{\observation}'}}}\\
    &&& \color{blue} +  \discount \cdot \expectedsymbol{\history, \observation \sim \historydistribution_{\latentpolicy}}\expectedsymbol{\action \sim \latentpolicy\fun{\sampledot \mid \beliefencoder^{*}\fun{\history}}} \left|\expectedsymbol{\state \sim \beliefupdate^{*}\fun{\history}}\Bigg[ \expectedsymbol{\state', \observation' \sim \probtransitions_{\observations}\fun{\sampledot \mid \state, \observation, \action}}{\expectedsymbol{\hat{\observation}' \sim \latentobservationfn_{\decoderparameter}\fun{\sampledot \mid \embed_{\encoderparameter}\fun{\state', \observation'}}}\latentvalues{\latentpolicy}{}{\history \cdot \action \cdot \hat{\observation}'}} \right. \\
    &&& \color{blue} \left. \qquad\qquad\qquad\qquad\qquad\qquad\qquad\qquad\;- \expectedsymbol{\latentstate' \sim \latentprobtransitions_{\decoderparameter}\fun{\sampledot \mid \embed_{\encoderparameter}\fun{ \state, \observation}, \action}}{\expectedsymbol{{\hat{\observation}}' \sim \latentobservationfn_{\decoderparameter}\fun{\sampledot \mid \latentstate'}}\latentvalues{\latentpolicy}{}{\history \cdot \action \cdot \hat{\observation}'}}\Bigg]\right| \\
    &&& + \color{blue} \discount \cdot \expectedsymbol{\history, \observation \sim \historydistribution_{\latentpolicy}}\expectedsymbol{\action \sim \latentpolicy\fun{\sampledot \mid \beliefencoder^{*}\fun{\history}}} \left|\expectedsymbol{\state \sim \beliefupdate^{*}\fun{\history}} \expectedsymbol{\latentstate' \sim \latentprobtransitions_{\decoderparameter}\fun{\sampledot \mid \embed_{\encoderparameter}\fun{\state, \observation}, \action}}\expectedsymbol{{\observation}' \sim \latentobservationfn_{\decoderparameter}\fun{\sampledot \mid \latentstate'}}\latentvalues{\latentpolicy}{}{\history \cdot \action \cdot {\observation}'} \right.\\
    &&& \color{blue} \left. \qquad\qquad\qquad\qquad\qquad\quad- \expectedsymbol{\latentstate \sim \latentbeliefupdate^{*}\fun{\history}}\expectedsymbol{\latentstate' \sim \latentprobtransitions_{\decoderparameter}\fun{\sampledot \mid \latentstate, \action}}\expectedsymbol{\observation' \sim \latentobservationfn_{\decoderparameter}\fun{\sampledot \mid \latentstate'}}\latentvalues{\latentpolicy}{}{\history \cdot \action \cdot \observation'}\right|
    \tag{\color{blue}Triangular inequality}
\end{align*}
\begin{align*}
    &{\color{blue}\leq}&&  \discount \cdot \expectedsymbol{\history \sim \historydistribution_{\latentpolicy}}\expectedsymbol{\action \sim \latentpolicy\fun{\sampledot \mid \beliefencoder^{*}\fun{\history}}}
    \abs{\expectedsymbol{\state \sim \beliefupdate^{*}\fun{\history}}\expectedsymbol{\state' \sim \probtransitions\fun{\sampledot \mid \state, \action}}\expected{\observation' \sim \observationfn\fun{\sampledot \mid \state', \action}}{\values{\latentpolicy}{}{\history \cdot \action \cdot \observation'} - \expectedsymbol{\hat{\observation}' \sim \latentobservationfn_{\decoderparameter}\fun{\sampledot \mid \embed_{\encoderparameter}\fun{\state', \observation'}}}\latentvalues{\latentpolicy}{}{\history \cdot \action \cdot \hat{\observation}'}}}\\
    &&& +  \discount \cdot \expectedsymbol{\history, \observation \sim \historydistribution_{\latentpolicy}}\expectedsymbol{\action \sim \latentpolicy\fun{\sampledot \mid \beliefencoder^{*}\fun{\history}}} {\color{blue}\expectedsymbol{\state \sim \beliefupdate^{*}\fun{\history}}} {\color{blue}\left|\expectedsymbol{\latentstate' \sim \embed_{\encoderparameter}\probtransitions_{\observations}\fun{\sampledot \mid \state, \observation, \action}}{\expectedsymbol{{\observation}' \sim \latentobservationfn_{\decoderparameter}\fun{\sampledot \mid \latentstate'}}\latentvalues{\latentpolicy}{}{\history \cdot \action \cdot {\observation}'}} \right.}\\
    &&& \color{blue}\left. \qquad\qquad\qquad\qquad\qquad\qquad\qquad- \expectedsymbol{\latentstate' \sim \latentprobtransitions_{\decoderparameter}\fun{\sampledot \mid \embed_{\encoderparameter}\fun{ \state, \observation}, \action}}{\expectedsymbol{{\observation}' \sim \latentobservationfn_{\decoderparameter}\fun{\sampledot \mid \latentstate'}}\latentvalues{\latentpolicy}{}{\history \cdot \action \cdot {\observation}'}}\right| \\
    &&&  + \discount \cdot \expectedsymbol{\history, \observation \sim \historydistribution_{\latentpolicy}}\expectedsymbol{\action \sim \latentpolicy\fun{\sampledot \mid \beliefencoder^{*}\fun{\history}}} \left|\expectedsymbol{\state \sim \beliefupdate^{*}\fun{\history}} \expectedsymbol{\latentstate' \sim \latentprobtransitions_{\decoderparameter}\fun{\sampledot \mid \embed_{\encoderparameter}\fun{\state, \observation}, \action}}\expectedsymbol{{\observation}' \sim \latentobservationfn_{\decoderparameter}\fun{\sampledot \mid \latentstate'}}\latentvalues{\latentpolicy}{}{\history \cdot \action \cdot {\observation}'} \right.\\
    &&& \qquad\qquad\qquad\qquad\qquad\quad-\left. \expectedsymbol{\latentstate \sim \latentbeliefupdate^{*}\fun{\history}}\expectedsymbol{\latentstate' \sim \latentprobtransitions_{\decoderparameter}\fun{\sampledot \mid \latentstate, \action}}\expectedsymbol{\observation' \sim \latentobservationfn_{\decoderparameter}\fun{\sampledot \mid \latentstate'}}\latentvalues{\latentpolicy}{}{\history \cdot \action \cdot \observation'}\right|
    \tag{\color{blue} by definition of $\embed\probtransitions_{\observations}$ and the Jensen's inequality}\\
    &\leq&&  \discount \cdot \expectedsymbol{\history \sim \historydistribution_{\latentpolicy}}\expectedsymbol{\action \sim \latentpolicy\fun{\sampledot \mid \beliefencoder^{*}\fun{\history}}}
    \abs{\expectedsymbol{\state \sim \beliefupdate^{*}\fun{\history}}\expectedsymbol{\state' \sim \probtransitions\fun{\sampledot \mid \state, \action}}\expected{\observation' \sim \observationfn\fun{\sampledot \mid \state', \action}}{\values{\latentpolicy}{}{\history \cdot \action \cdot \observation'} - \expectedsymbol{\hat{\observation}' \sim \latentobservationfn_{\decoderparameter}\fun{\sampledot \mid \embed_{\encoderparameter}\fun{\state', \observation'}}}\latentvalues{\latentpolicy}{}{\history \cdot \action \cdot \hat{\observation}'}}}\\
    &&&  +  \discount \cdot \expectedsymbol{\history, \observation \sim \historydistribution_{\latentpolicy}}\expectedsymbol{\action \sim \latentpolicy\fun{\sampledot \mid \beliefencoder^{*}\fun{\history}}} \expectedsymbol{\state \sim \beliefupdate^{*}\fun{\history}} {\color{blue}\KV \cdot \wassersteindist{\latentdistance}{\embed_{\encoderparameter}\probtransitions_{\observations}\fun{\sampledot \mid \state, \action}}{\latentprobtransitions_{\decoderparameter}\fun{\sampledot \mid \embed_{\encoderparameter}\fun{\state, \observation}, \action}}}\\
    &&& +  \discount \cdot \expectedsymbol{\history, \observation \sim \historydistribution_{\latentpolicy}}\expectedsymbol{\action \sim \latentpolicy\fun{\sampledot \mid \beliefencoder^{*}\fun{\history}}} \left|\expectedsymbol{\state \sim \beliefupdate^{*}\fun{\history}} \expectedsymbol{\latentstate' \sim \latentprobtransitions_{\decoderparameter}\fun{\sampledot \mid \embed_{\encoderparameter}\fun{\state, \observation}, \action}}\expectedsymbol{{\observation}' \sim \latentobservationfn_{\decoderparameter}\fun{\sampledot \mid \latentstate'}}\latentvalues{\latentpolicy}{}{\history \cdot \action \cdot {\observation}'} \right. \\
    &&& \left. \qquad\qquad\qquad\qquad\qquad\quad- \expectedsymbol{\latentstate \sim \latentbeliefupdate^{*}\fun{\history}}\expectedsymbol{\latentstate' \sim \latentprobtransitions_{\decoderparameter}\fun{\sampledot \mid \latentstate, \action}}\expectedsymbol{\observation' \sim \latentobservationfn_{\decoderparameter}\fun{\sampledot \mid \latentstate'}}\latentvalues{\latentpolicy}{}{\history \cdot \action \cdot \observation'}\right|
    \tag{\color{blue}as $\temperature \to 0$, by Lem.~\ref{lem:wasserstein-expecation}}\\
    &\leq&&  \discount \cdot \expectedsymbol{\history \sim \historydistribution_{\latentpolicy}}\expectedsymbol{\action \sim \latentpolicy\fun{\sampledot \mid \beliefencoder^{*}\fun{\history}}}
    \abs{\expectedsymbol{\state \sim \beliefupdate^{*}\fun{\history}}\expectedsymbol{\state' \sim \probtransitions\fun{\sampledot \mid \state, \action}}\expected{\observation' \sim \observationfn\fun{\sampledot \mid \state', \action}}{\values{\latentpolicy}{}{\history \cdot \action \cdot \observation'} - \expectedsymbol{\hat{\observation}' \sim \latentobservationfn_{\decoderparameter}\fun{\sampledot \mid \embed_{\encoderparameter}\fun{\state', \observation'}}}\latentvalues{\latentpolicy}{}{\history \cdot \action \cdot \hat{\observation}'}}}\\
    &&&  +  \discount \KV {\color{blue}\localtransitionloss{\historydistribution_{\latentpolicy}}}\\
    &&& +  \discount \cdot \expectedsymbol{\history, \observation \sim \historydistribution_{\latentpolicy}}\expectedsymbol{\action \sim \latentpolicy\fun{\sampledot \mid \beliefencoder^{*}\fun{\history}}} \left|\expectedsymbol{\state \sim \beliefupdate^{*}\fun{\history}} \expectedsymbol{\latentstate' \sim \latentprobtransitions_{\decoderparameter}\fun{\sampledot \mid \embed_{\encoderparameter}\fun{\state, \observation}, \action}}\expectedsymbol{{\observation}' \sim \latentobservationfn_{\decoderparameter}\fun{\sampledot \mid \latentstate'}}\latentvalues{\latentpolicy}{}{\history \cdot \action \cdot {\observation}'} \right.\\
    &&& \left. \qquad\qquad\qquad\qquad\qquad\quad- \expectedsymbol{\latentstate \sim \latentbeliefupdate^{*}\fun{\history}}\expectedsymbol{\latentstate' \sim \latentprobtransitions_{\decoderparameter}\fun{\sampledot \mid \latentstate, \action}}\expectedsymbol{\observation' \sim \latentobservationfn_{\decoderparameter}\fun{\sampledot \mid \latentstate'}}\latentvalues{\latentpolicy}{}{\history \cdot \action \cdot \observation'}\right|
    \tag{\color{blue}by definition of $\localtransitionloss{\historydistribution_{\latentpolicy}}$}\\
    &=&&  \discount \cdot \expectedsymbol{\history \sim \historydistribution_{\latentpolicy}}\expectedsymbol{\action \sim \latentpolicy\fun{\sampledot \mid \beliefencoder^{*}\fun{\history}}}
    \abs{\expectedsymbol{\state \sim \beliefupdate^{*}\fun{\history}}\expectedsymbol{\state' \sim \probtransitions\fun{\sampledot \mid \state, \action}}\expected{\observation' \sim \observationfn\fun{\sampledot \mid \state', \action}}{\values{\latentpolicy}{}{\history \cdot \action \cdot \observation'} - \expectedsymbol{\hat{\observation}' \sim \latentobservationfn_{\decoderparameter}\fun{\sampledot \mid \embed_{\encoderparameter}\fun{\state', \observation'}}}\latentvalues{\latentpolicy}{}{\history \cdot \action \cdot \hat{\observation}'}}}\\
    &&&  +  \discount \KV \localtransitionloss{\historydistribution_{\latentpolicy}}\\
    &&& +  \discount \cdot \expectedsymbol{\history, \observation \sim \historydistribution_{\latentpolicy}}\expectedsymbol{\action \sim \latentpolicy\fun{\sampledot \mid \beliefencoder^{*}\fun{\history}}} \Bigg|\left[\expectedsymbol{\state \sim \beliefupdate^{*}\fun{\history}} \expectedsymbol{\latentstate' \sim \latentprobtransitions_{\decoderparameter}\fun{\sampledot \mid \embed_{\encoderparameter}\fun{\state, \observation}, \action}}\expectedsymbol{{\observation}' \sim \latentobservationfn_{\decoderparameter}\fun{\sampledot \mid \latentstate'}}\latentvalues{\latentpolicy}{}{\history \cdot \action \cdot {\observation}'} \right.\\
    &&&\left. {\color{blue}\qquad\qquad\qquad\qquad\qquad\qquad - \expectedsymbol{\latentstate \sim \beliefencoder^{*}\fun{\history}} \expectedsymbol{\latentstate' \sim \latentprobtransitions_{\decoderparameter}\fun{\sampledot \mid \latentstate, \action}}\expectedsymbol{\observation' \sim \latentobservationfn_{\decoderparameter}\fun{\sampledot \mid \latentstate'}} \latentvalues{\latentpolicy}{}{\history \cdot \action \cdot \observation'}} \right]\\
    &&& \quad \quad \quad \quad + \left[
    {\color{blue}\expectedsymbol{\latentstate \sim \beliefencoder^{*}\fun{\history}} \expectedsymbol{\latentstate' \sim \latentprobtransitions_{\decoderparameter}\fun{\sampledot \mid \latentstate, \action}}\expectedsymbol{\observation' \sim \latentobservationfn_{\decoderparameter}\fun{\sampledot \mid \latentstate'}} \latentvalues{\latentpolicy}{}{\history \cdot \action \cdot \observation'}} - 
    \expectedsymbol{\latentstate \sim \latentbeliefupdate^{*}\fun{\history}}\expectedsymbol{\latentstate' \sim \latentprobtransitions_{\decoderparameter}\fun{\sampledot \mid \latentstate, \action}}\expectedsymbol{\observation' \sim \latentobservationfn_{\decoderparameter}\fun{\sampledot \mid \latentstate'}}\latentvalues{\latentpolicy}{}{\history \cdot \action \cdot \observation'}
    \right]\Bigg|
    \tag{\color{blue}the belief encoder $\beliefencoder$ comes into play}
\end{align*}
\begin{align*}
    &\leq&&  \discount \cdot \expectedsymbol{\history \sim \historydistribution_{\latentpolicy}}\expectedsymbol{\action \sim \latentpolicy\fun{\sampledot \mid \beliefencoder^{*}\fun{\history}}}
    \abs{\expectedsymbol{\state \sim \beliefupdate^{*}\fun{\history}}\expectedsymbol{\state' \sim \probtransitions\fun{\sampledot \mid \state, \action}}\expected{\observation' \sim \observationfn\fun{\sampledot \mid \state', \action}}{\values{\latentpolicy}{}{\history \cdot \action \cdot \observation'} - \expectedsymbol{\hat{\observation}' \sim \latentobservationfn_{\decoderparameter}\fun{\sampledot \mid \embed_{\encoderparameter}\fun{\state', \observation'}}}\latentvalues{\latentpolicy}{}{\history \cdot \action \cdot \hat{\observation}'}}}\\
    &&&  +  \discount \KV \localtransitionloss{\historydistribution_{\latentpolicy}}\\
    &&& \color{blue}+  \discount \cdot \expectedsymbol{\history, \observation \sim \historydistribution_{\latentpolicy}}\expectedsymbol{\action \sim \latentpolicy\fun{\sampledot \mid \beliefencoder^{*}\fun{\history}}} \left|\expectedsymbol{\state \sim \beliefupdate^{*}\fun{\history}} \expectedsymbol{\latentstate' \sim \latentprobtransitions_{\decoderparameter}\fun{\sampledot \mid \embed_{\encoderparameter}\fun{\state, \observation}, \action}}\expectedsymbol{{\observation}' \sim \latentobservationfn_{\decoderparameter}\fun{\sampledot \mid \latentstate'}}\latentvalues{\latentpolicy}{}{\history \cdot \action \cdot {\observation}'} \right. \\
    &&& \color{blue}\left. \qquad\qquad\qquad\qquad\qquad- \expectedsymbol{\latentstate \sim \beliefencoder^{*}\fun{\history}} \expectedsymbol{\latentstate' \sim \latentprobtransitions_{\decoderparameter}\fun{\sampledot \mid \latentstate, \action}}\expectedsymbol{\observation' \sim \latentobservationfn_{\decoderparameter}\fun{\sampledot \mid \latentstate'}} \latentvalues{\latentpolicy}{}{\history \cdot \action \cdot \observation'} \right|\\
    &&&  \color{blue}+  \discount \cdot \expectedsymbol{\history \sim \historydistribution_{\latentpolicy}}\expectedsymbol{\action \sim \latentpolicy\fun{\sampledot \mid \beliefencoder^{*}\fun{\history}}} \left|
    \expectedsymbol{\latentstate \sim \beliefencoder^{*}\fun{\history}} \expectedsymbol{\latentstate' \sim \latentprobtransitions_{\decoderparameter}\fun{\sampledot \mid \latentstate, \action}}\expectedsymbol{\observation' \sim \latentobservationfn_{\decoderparameter}\fun{\sampledot \mid \latentstate'}} \latentvalues{\latentpolicy}{}{\history \cdot \action \cdot \observation'} \right.\\
    &&& \color{blue} \left. \qquad\qquad\qquad\qquad\qquad-
    \expectedsymbol{\latentstate \sim \latentbeliefupdate^{*}\fun{\history}}\expectedsymbol{\latentstate' \sim \latentprobtransitions_{\decoderparameter}\fun{\sampledot \mid \latentstate, \action}}\expectedsymbol{\observation' \sim \latentobservationfn_{\decoderparameter}\fun{\sampledot \mid \latentstate'}}\latentvalues{\latentpolicy}{}{\history \cdot \action \cdot \observation'}\right|
    \tag{ \color{blue}triangular inequality}\\
    &\leq&&  \discount \cdot \expectedsymbol{\history \sim \historydistribution_{\latentpolicy}}\expectedsymbol{\action \sim \latentpolicy\fun{\sampledot \mid \beliefencoder^{*}\fun{\history}}}
    \abs{\expectedsymbol{\state \sim \beliefupdate^{*}\fun{\history}}\expectedsymbol{\state' \sim \probtransitions\fun{\sampledot \mid \state, \action}}\expected{\observation' \sim \observationfn\fun{\sampledot \mid \state', \action}}{\values{\latentpolicy}{}{\history \cdot \action \cdot \observation'} - \expectedsymbol{\hat{\observation}' \sim \latentobservationfn_{\decoderparameter}\fun{\sampledot \mid \embed_{\encoderparameter}\fun{\state', \observation'}}}\latentvalues{\latentpolicy}{}{\history \cdot \action \cdot \hat{\observation}'}}}\\
    &&&  +  \discount \KV \localtransitionloss{\historydistribution_{\latentpolicy}}\\
    &&& +  \discount \cdot \expectedsymbol{\history, \observation \sim \historydistribution_{\latentpolicy}}\expectedsymbol{\action \sim \latentpolicy\fun{\sampledot \mid \beliefencoder^{*}\fun{\history}}}  \color{blue}\expectedsymbol{\state \sim \beliefupdate^{*}\fun{\history}} \expectedsymbol{\latentstate \sim \beliefencoder^{*}\fun{\history}} \left| \expected{\latentstate' \sim \latentprobtransitions_{\decoderparameter}\fun{\sampledot \mid \embed_{\encoderparameter}\fun{\state, \observation}, \action}}{\expectedsymbol{{\observation}' \sim \latentobservationfn_{\decoderparameter}\fun{\sampledot \mid \latentstate'}}\latentvalues{\latentpolicy}{}{\history \cdot \action \cdot {\observation}'}} \right. \\
    &&&  \color{blue}\qquad\qquad\qquad\qquad\qquad\qquad\qquad\qquad\qquad- \left.  \expected{\latentstate' \sim \latentprobtransitions_{\decoderparameter}\fun{\sampledot \mid \latentstate, \action}}{\expectedsymbol{\observation' \sim \latentobservationfn_{\decoderparameter}\fun{\sampledot \mid \latentstate'}} \latentvalues{\latentpolicy}{}{\history \cdot \action \cdot \observation'}}\right|\\
    &&& +  \discount \cdot \expectedsymbol{\history \sim \historydistribution_{\latentpolicy}}\expectedsymbol{\action \sim \latentpolicy\fun{\sampledot \mid \beliefencoder^{*}\fun{\history}}} \left|
    \expectedsymbol{\latentstate \sim \beliefencoder^{*}\fun{\history}} \expectedsymbol{\latentstate' \sim \latentprobtransitions_{\decoderparameter}\fun{\sampledot \mid \latentstate, \action}}\expectedsymbol{\observation' \sim \latentobservationfn_{\decoderparameter}\fun{\sampledot \mid \latentstate'}} \latentvalues{\latentpolicy}{}{\history \cdot \action \cdot \observation'} \right.\\
    &&& \left. \qquad\qquad\qquad\qquad\qquad - 
    \expectedsymbol{\latentstate \sim \latentbeliefupdate^{*}\fun{\history}}\expectedsymbol{\latentstate' \sim \latentprobtransitions_{\decoderparameter}\fun{\sampledot \mid \latentstate, \action}}\expectedsymbol{\observation' \sim \latentobservationfn_{\decoderparameter}\fun{\sampledot \mid \latentstate'}}\latentvalues{\latentpolicy}{}{\history \cdot \action \cdot \observation'}\right|
    \tag{ \color{blue}Jensen's inequality}\\
    &\leq&&  \discount \cdot \expectedsymbol{\history \sim \historydistribution_{\latentpolicy}}\expectedsymbol{\action \sim \latentpolicy\fun{\sampledot \mid \beliefencoder^{*}\fun{\history}}}
    \abs{\expectedsymbol{\state \sim \beliefupdate^{*}\fun{\history}}\expectedsymbol{\state' \sim \probtransitions\fun{\sampledot \mid \state, \action}}\expected{\observation' \sim \observationfn\fun{\sampledot \mid \state', \action}}{\values{\latentpolicy}{}{\history \cdot \action \cdot \observation'} - \expectedsymbol{\hat{\observation}' \sim \latentobservationfn_{\decoderparameter}\fun{\sampledot \mid \embed_{\encoderparameter}\fun{\state', \observation'}}}\latentvalues{\latentpolicy}{}{\history \cdot \action \cdot \hat{\observation}'}}}\\
    &&&  +  \discount \KV \localtransitionloss{\historydistribution_{\latentpolicy}}\\
    &&& +  \discount \cdot \expectedsymbol{\history, \observation \sim \historydistribution_{\latentpolicy}}\expectedsymbol{\action \sim \latentpolicy\fun{\sampledot \mid \beliefencoder^{*}\fun{\history}}} \expectedsymbol{\state \sim \beliefupdate^{*}\fun{\history}} \expectedsymbol{\latentstate \sim \beliefencoder^{*}\fun{\history}}  \color{blue}\KV \wassersteindist{\latentdistance}{\latentprobtransitions_{\decoderparameter}\fun{\sampledot \mid \embed_{\encoderparameter}\fun{\state, \observation}, \action}}{\latentprobtransitions_{\decoderparameter}\fun{\sampledot \mid \latentstate, \action}}\\
    &&& +  \discount \cdot \expectedsymbol{\history \sim \historydistribution_{\latentpolicy}}\expectedsymbol{\action \sim \latentpolicy\fun{\sampledot \mid \beliefencoder^{*}\fun{\history}}} \left|
    \expectedsymbol{\latentstate \sim \beliefencoder^{*}\fun{\history}} \expectedsymbol{\latentstate' \sim \latentprobtransitions_{\decoderparameter}\fun{\sampledot \mid \latentstate, \action}}\expectedsymbol{\observation' \sim \latentobservationfn_{\decoderparameter}\fun{\sampledot \mid \latentstate'}} \latentvalues{\latentpolicy}{}{\history \cdot \action \cdot \observation'} \right. \\
    &&& \left. \qquad\qquad\qquad\qquad\qquad\quad- 
    \expectedsymbol{\latentstate \sim \latentbeliefupdate^{*}\fun{\history}}\expectedsymbol{\latentstate' \sim \latentprobtransitions_{\decoderparameter}\fun{\sampledot \mid \latentstate, \action}}\expectedsymbol{\observation' \sim \latentobservationfn_{\decoderparameter}\fun{\sampledot \mid \latentstate'}}\latentvalues{\latentpolicy}{}{\history \cdot \action \cdot \observation'}\right|
    \tag{ \color{blue}as $\temperature \to 0$, by Lem.~\ref{lem:wasserstein-expecation}}\\
    &=&&  \discount \cdot \expectedsymbol{\history \sim \historydistribution_{\latentpolicy}}\expectedsymbol{\action \sim \latentpolicy\fun{\sampledot \mid \beliefencoder^{*}\fun{\history}}}
    \abs{\expectedsymbol{\state \sim \beliefupdate^{*}\fun{\history}}\expectedsymbol{\state' \sim \probtransitions\fun{\sampledot \mid \state, \action}}\expected{\observation' \sim \observationfn\fun{\sampledot \mid \state', \action}}{\values{\latentpolicy}{}{\history \cdot \action \cdot \observation'} - \expectedsymbol{\hat{\observation}' \sim \latentobservationfn_{\decoderparameter}\fun{\sampledot \mid \embed_{\encoderparameter}\fun{\state', \observation'}}}\latentvalues{\latentpolicy}{}{\history \cdot \action \cdot \hat{\observation}'}}}\\
    &&&  +  \discount \KV \cdot \fun{ \localtransitionloss{\historydistribution_{\latentpolicy}} +  { \color{blue}\onpolicytransitionloss{\historydistribution_{\latentpolicy}}}} \\
    &&& +  \discount \cdot \expectedsymbol{\history \sim \historydistribution_{\latentpolicy}}\expectedsymbol{\action \sim \latentpolicy\fun{\sampledot \mid \beliefencoder^{*}\fun{\history}}} \left|
    \expected{\latentstate \sim \beliefencoder^{*}\fun{\history}}{ \expectedsymbol{\latentstate' \sim \latentprobtransitions_{\decoderparameter}\fun{\sampledot \mid \latentstate, \action}}\expectedsymbol{\observation' \sim \latentobservationfn_{\decoderparameter}\fun{\sampledot \mid \latentstate'}} \latentvalues{\latentpolicy}{}{\history \cdot \action \cdot \observation'}} \right.\\
    &&& \left. \qquad\qquad\qquad\qquad\qquad-  \expected{\latentstate \sim \latentbeliefupdate^{*}\fun{\history}}{\expectedsymbol{\latentstate' \sim \latentprobtransitions_{\decoderparameter}\fun{\sampledot \mid \latentstate, \action}}\expectedsymbol{\observation' \sim \latentobservationfn_{\decoderparameter}\fun{\sampledot \mid \latentstate'}}\latentvalues{\latentpolicy}{}{\history \cdot \action \cdot \observation'}}\right|
    \tag{ \color{blue}by definition of $\onpolicytransitionloss{\historydistribution_{\latentpolicy}}$, Eq.~\ref{eq:on-policy-losses}}
\end{align*}
\begin{align*}
    &\leq&&  \discount \cdot \expectedsymbol{\history \sim \historydistribution_{\latentpolicy}}\expectedsymbol{\action \sim \latentpolicy\fun{\sampledot \mid \beliefencoder^{*}\fun{\history}}}
    \abs{\expectedsymbol{\state \sim \beliefupdate^{*}\fun{\history}}\expectedsymbol{\state' \sim \probtransitions\fun{\sampledot \mid \state, \action}}\expected{\observation' \sim \observationfn\fun{\sampledot \mid \state', \action}}{\values{\latentpolicy}{}{\history \cdot \action \cdot \observation'} - \expectedsymbol{\hat{\observation}' \sim \latentobservationfn_{\decoderparameter}\fun{\sampledot \mid \embed_{\encoderparameter}\fun{\state', \observation'}}}\latentvalues{\latentpolicy}{}{\history \cdot \action \cdot \hat{\observation}'}}}\\
    &&&  +  \discount \KV \cdot \fun{ \localtransitionloss{\historydistribution_{\latentpolicy}} +  \onpolicytransitionloss{\historydistribution_{\latentpolicy}}} \\
    &&& \color{blue} +  \discount \cdot \expectedsymbol{\history \sim \historydistribution_{\latentpolicy}}\expectedsymbol{\action \sim \latentpolicy\fun{\sampledot \mid \beliefencoder^{*}\fun{\history}}} \KV \wassersteindist{\latentdistance}{\latentbeliefupdate^{*}\fun{\history}}{\beliefencoder^{*}\fun{\history}}
    \tag{\color{blue}as $\temperature \to 0$, by Lem.~\ref{lem:wasserstein-expecation}; note that Wasserstein is symmetric since it is a distance metric \citep{Villani2009}}\\
    &\leq&&  \discount \cdot \expectedsymbol{\history \sim \historydistribution_{\latentpolicy}}\expectedsymbol{\action \sim \latentpolicy\fun{\sampledot \mid \beliefencoder^{*}\fun{\history}}}
    \abs{\expectedsymbol{\state \sim \beliefupdate^{*}\fun{\history}}\expectedsymbol{\state' \sim \probtransitions\fun{\sampledot \mid \state, \action}}\expected{\observation' \sim \observationfn\fun{\sampledot \mid \state', \action}}{\values{\latentpolicy}{}{\history \cdot \action \cdot \observation'} - \expectedsymbol{\hat{\observation}' \sim \latentobservationfn_{\decoderparameter}\fun{\sampledot \mid \embed_{\encoderparameter}\fun{\state', \observation'}}}\latentvalues{\latentpolicy}{}{\history \cdot \action \cdot \hat{\observation}'}}}\\
    &&&  +  \discount \KV \cdot \fun{ \localtransitionloss{\historydistribution_{\latentpolicy}} +  \onpolicytransitionloss{\historydistribution_{\latentpolicy}} + {\color{blue}\beliefloss{\historydistribution_{\latentpolicy}}}}
    \tag{\color{blue}by definition of $\beliefloss{\historydistribution_{\latentpolicy}}$, Eq.~\ref{eq:on-policy-losses}}\\
    &=&&  \discount \cdot \expectedsymbol{\history, \observation \sim \historydistribution_{\latentpolicy}}\expectedsymbol{\action \sim \latentpolicy\fun{\sampledot \mid \beliefencoder^{*}\fun{\history}}}
    \left|\expectedsymbol{\state \sim \beliefupdate^{*}\fun{\history}}\expectedsymbol{\state', \observation' \sim \probtransitions_{\observations}\fun{\sampledot \mid \state, \observation, \action}}\left[\fun{\values{\latentpolicy}{}{\history \cdot \action \cdot \observation'} -
    \latentvalues{\latentpolicy}{}{\history \cdot \action \cdot \observation'}} \vphantom{\fun{\latentvalues{\latentpolicy}{}{\history \cdot \action \cdot \observation'} - 
    \expectedsymbol{\hat{\observation}' \sim \latentobservationfn_{\decoderparameter}\fun{\sampledot \mid \embed_{\encoderparameter}\fun{\state', \observation'}}}\latentvalues{\latentpolicy}{}{\history \cdot \action \cdot \hat{\observation}'}}}\right.\right. \\
    &&&\left.\left. \qquad\qquad\qquad\qquad\qquad\qquad\qquad\qquad\qquad\quad+ %
    \fun{\latentvalues{\latentpolicy}{}{\history \cdot \action \cdot \observation'} - 
    \expectedsymbol{\hat{\observation}' \sim \latentobservationfn_{\decoderparameter}\fun{\sampledot \mid \embed_{\encoderparameter}\fun{\state', \observation'}}}\latentvalues{\latentpolicy}{}{\history \cdot \action \cdot \hat{\observation}'}}\right]\right|\\
    &&&  +  \discount \KV \cdot \fun{ \localtransitionloss{\historydistribution_{\latentpolicy}} +  \onpolicytransitionloss{\historydistribution_{\latentpolicy}} + \beliefloss{\historydistribution_{\latentpolicy}}}\\
    &\leq&&  \discount \color{blue}\cdot \expectedsymbol{\history, \observation \sim \historydistribution_{\latentpolicy}}\expectedsymbol{\action \sim \latentpolicy\fun{\sampledot \mid \beliefencoder^{*}\fun{\history}}}
    \abs{\expectedsymbol{\state \sim \beliefupdate^{*}\fun{\history}}\expected{\state', \observation' \sim \probtransitions_{\observations}\fun{\sampledot \mid \state, \observation, \action}}{\values{\latentpolicy}{}{\history \cdot \action \cdot \observation'} -
    \latentvalues{\latentpolicy}{}{\history \cdot \action \cdot \observation'}}} \\
    &&& + \discount \cdot \color{blue}
    \expectedsymbol{\history, \observation \sim \historydistribution_{\latentpolicy}}\expectedsymbol{\action \sim \latentpolicy\fun{\sampledot \mid \beliefencoder^{*}\fun{\history}}}
    \abs{\expectedsymbol{\state \sim \beliefupdate^{*}\fun{\history}}\expected{\state', \observation' \sim \probtransitions_{\observations}\fun{\sampledot \mid \state, \observation, \action}}
    {\latentvalues{\latentpolicy}{}{\history \cdot \action \cdot \observation'} - 
    \expectedsymbol{\hat{\observation}' \sim \latentobservationfn_{\decoderparameter}\fun{\sampledot \mid \embed_{\encoderparameter}\fun{\state', \observation'}}}\latentvalues{\latentpolicy}{}{\history \cdot \action \cdot \hat{\observation}'}}}\\
    &&&  +  \discount \KV \cdot \fun{ \localtransitionloss{\historydistribution_{\latentpolicy}} +  \onpolicytransitionloss{\historydistribution_{\latentpolicy}} + \beliefloss{\historydistribution_{\latentpolicy}}}
    \tag{\color{blue}triangular inequality}
\end{align*}
\begin{align*}
    &\leq&&  \discount \cdot \expectedsymbol{\history, \observation \sim \historydistribution_{\latentpolicy}}\expectedsymbol{\action \sim \latentpolicy\fun{\sampledot \mid \beliefencoder^{*}\fun{\history}}}
    \abs{\expectedsymbol{\state \sim \beliefupdate^{*}\fun{\history}}\expected{\state', \observation' \sim \probtransitions_{\observations}\fun{\sampledot \mid \state, \observation, \action}}{\values{\latentpolicy}{}{\history \cdot \action \cdot \observation'} -
    \latentvalues{\latentpolicy}{}{\history \cdot \action \cdot \observation'}}} \\
    &&& + \discount \cdot
    \expectedsymbol{\history, \observation \sim \historydistribution_{\latentpolicy}}\expectedsymbol{\action \sim \latentpolicy\fun{\sampledot \mid \beliefencoder^{*}\fun{\history}}} \color{blue}
    \expectedsymbol{\state \sim \beliefupdate^{*}\fun{\history}}
    \expectedsymbol{\state' \sim \probtransitions\fun{\sampledot \mid \state, \action}}
    \abs{\expected{\observation' \sim \observationfn\fun{\sampledot \mid \state', \action}}
    {\latentvalues{\latentpolicy}{}{\history \cdot \action \cdot \observation'} - 
    \expectedsymbol{\hat{\observation}' \sim \latentobservationfn_{\decoderparameter}\fun{\sampledot \mid \embed_{\encoderparameter}\fun{\state', \observation'}}}\latentvalues{\latentpolicy}{}{\history \cdot \action \cdot \hat{\observation}'}}}\\
    &&&  +  \discount \KV \cdot \fun{ \localtransitionloss{\historydistribution_{\latentpolicy}} +  \onpolicytransitionloss{\historydistribution_{\latentpolicy}} + \beliefloss{\historydistribution_{\latentpolicy}}}
    \tag{\color{blue}by definition of $\probtransitions_{\observations}$ and the Jensen's inequality}\\
    &\leq&&  \discount \cdot \expectedsymbol{\history, \observation \sim \historydistribution_{\latentpolicy}}\expectedsymbol{\action \sim \latentpolicy\fun{\sampledot \mid \beliefencoder^{*}\fun{\history}}}
    \abs{\expectedsymbol{\state \sim \beliefupdate^{*}\fun{\history}}\expected{\state', \observation' \sim \probtransitions_{\observations}\fun{\sampledot \mid \state, \observation, \action}}{\values{\latentpolicy}{}{\history \cdot \action \cdot \observation'} -
    \latentvalues{\latentpolicy}{}{\history \cdot \action \cdot \observation'}}} \\
    &&& + \discount \cdot
    \expectedsymbol{\history, \observation \sim \historydistribution_{\latentpolicy}}\expectedsymbol{\action \sim \latentpolicy\fun{\sampledot \mid \beliefencoder^{*}\fun{\history}}}
    \expectedsymbol{\state \sim \beliefupdate^{*}\fun{\history}}
    \expectedsymbol{\state' \sim \probtransitions\fun{\sampledot \mid \state, \action}} \color{blue}
    \KV \dtv{\observationfn\fun{\sampledot \mid \state', \action}}{\expectedsymbol{\observation' \sim \state', \action}\latentobservationfn_{\decoderparameter}\fun{\sampledot \mid \embed_{\encoderparameter}\fun{\state', \observation'}}}\\
    &&&  +  \discount \KV \cdot \fun{ \localtransitionloss{\historydistribution_{\latentpolicy}} +  \onpolicytransitionloss{\historydistribution_{\latentpolicy}} + \beliefloss{\historydistribution_{\latentpolicy}}}
    \tag{\color{blue}cf. Prop.~\ref{prop:lipschitz-constant} and Lem~\ref{lem:wasserstein-expecation}}\\
    &=&&  \discount \cdot \expectedsymbol{\history, \observation \sim \historydistribution_{\latentpolicy}}\expectedsymbol{\action \sim \latentpolicy\fun{\sampledot \mid \beliefencoder^{*}\fun{\history}}}
    \abs{\expectedsymbol{\state \sim \beliefupdate^{*}\fun{\history}}\expected{\state', \observation' \sim \probtransitions_{\observations}\fun{\sampledot \mid \state, \observation, \action}}{\values{\latentpolicy}{}{\history \cdot \action \cdot \observation'} -
    \latentvalues{\latentpolicy}{}{\history \cdot \action \cdot \observation'}}} \\
    &&&  +  \discount \KV \cdot \fun{ \localtransitionloss{\historydistribution_{\latentpolicy}} +  \onpolicytransitionloss{\historydistribution_{\latentpolicy}} + \beliefloss{\historydistribution_{\latentpolicy}} + {\color{blue}\observationloss{\historydistribution_{\latentpolicy}}}}
    \tag{\color{blue}by definition of $\observationloss{\historydistribution_{\latentpolicy}}$, Eq.~\ref{eq:observation-loss}}\\
    &\leq&&  \discount \cdot \expectedsymbol{\history, \observation \sim \historydistribution_{\latentpolicy}}\expectedsymbol{\action \sim \latentpolicy\fun{\sampledot \mid \beliefencoder^{*}\fun{\history}}} \color{blue}
    \expectedsymbol{\state \sim \beliefupdate^{*}\fun{\history}}\expectedsymbol{\state', \observation' \sim \probtransitions_{\observations}\fun{\sampledot \mid \state, \observation, \action}}
    \abs{{\values{\latentpolicy}{}{\history \cdot \action \cdot \observation'} -
    \latentvalues{\latentpolicy}{}{\history \cdot \action \cdot \observation'}}}  \\
    &&& +  \discount \KV \cdot \fun{ \localtransitionloss{\historydistribution_{\latentpolicy}} +  \onpolicytransitionloss{\historydistribution_{\latentpolicy}} + \beliefloss{\historydistribution_{\latentpolicy}} + \observationloss{\historydistribution_{\latentpolicy}}}
    \tag{\color{blue}Jensen's inequality}\\
    &=&&  \discount \cdot {\color{blue} \expectedsymbol{\history, \observation \sim \historydistribution_{\latentpolicy}}
    \abs{{\values{\latentpolicy}{}{\history} -
    \latentvalues{\latentpolicy}{}{\history}}}} + \discount \KV \cdot \fun{ \localtransitionloss{\historydistribution_{\latentpolicy}} +  \onpolicytransitionloss{\historydistribution_{\latentpolicy}} + \beliefloss{\historydistribution_{\latentpolicy}} + \observationloss{\historydistribution_{\latentpolicy}}}
    \tag{\color{blue}$\historydistribution_{\latentpolicy}$ is a stationary distribution (Lem.~\ref{lem:extended-stationary-histories}) which allows us to apply the stationary property (Def.~\ref{def:stationary-distr})}
\end{align*}

\smallparagraph{Putting all together.}~%
To recap, by Part 1 and 2, we have:
\begin{align*}
    \expectedsymbol{\history \sim \historydistribution_{\latentpolicy}}\abs{\values{\latentpolicy}{}{\history} - \latentvalues{\latentpolicy}{}{\history}} &\leq \localrewardloss{\historydistribution_{\latentpolicy}} + \onpolicyrewardloss{\historydistribution_{\latentpolicy}} + \Rmax \beliefloss{\historydistribution_{\latentpolicy}} + \discount \cdot \expectedsymbol{\history \sim \historydistribution_{\latentpolicy}}\abs{\values{\latentpolicy}{}{\history} - \latentvalues{\latentpolicy}{}{\history}} \\
    & \qquad+ \discount \KV \cdot \fun{ \localtransitionloss{\historydistribution_{\latentpolicy}} +  \onpolicytransitionloss{\historydistribution_{\latentpolicy}} + \beliefloss{\historydistribution_{\latentpolicy}} + \observationloss{\historydistribution_{\latentpolicy}}} \\
    \expectedsymbol{\history \sim \historydistribution_{\latentpolicy}}\abs{\values{\latentpolicy}{}{\history} - \latentvalues{\latentpolicy}{}{\history}} \cdot \fun{1 - \discount} &\leq \localrewardloss{\historydistribution_{\latentpolicy}} + \onpolicyrewardloss{\historydistribution_{\latentpolicy}} + \Rmax \beliefloss{\historydistribution_{\latentpolicy}} + \discount \KV \cdot \fun{ \localtransitionloss{\historydistribution_{\latentpolicy}} +  \onpolicytransitionloss{\historydistribution_{\latentpolicy}} + \beliefloss{\historydistribution_{\latentpolicy}} + \observationloss{\historydistribution_{\latentpolicy}}} \\
    \expectedsymbol{\history \sim \historydistribution_{\latentpolicy}}\abs{\values{\latentpolicy}{}{\history} - \latentvalues{\latentpolicy}{}{\history}} &\leq \frac{\localrewardloss{\historydistribution_{\latentpolicy}} + \onpolicyrewardloss{\historydistribution_{\latentpolicy}} + \Rmax \beliefloss{\historydistribution_{\latentpolicy}} + \discount \KV \cdot \fun{ \localtransitionloss{\historydistribution_{\latentpolicy}} +  \onpolicytransitionloss{\historydistribution_{\latentpolicy}} + \beliefloss{\historydistribution_{\latentpolicy}} + \observationloss{\historydistribution_{\latentpolicy}}}}{1 - \discount}
\end{align*}
which finally concludes the proof.
\end{proof}

\subsection{Representation Quality Bound}
We start by showing that the optimal \emph{latent} value function is \emph{almost} Lipschitz continuous in the latent belief space.
Coupled with Theorem~\ref{thm:value-diff-bounds-extended}, this result allows to show that whenever \emph{two pairs of histories are encoded to close representations, their values (i.e., the return obtained from that history points) are guaranteed to be close as well} whenever the losses introduced in Sec.~\ref{sec:guarantees} are minimized and go to zero.
Phrased differently, this Theorem ensures that the representation induced by our encoder is suitable to optimize the value function since the distance between beliefs in the latent space characterizes the distance of behaviors of the agent in the original environment. The latent belief space thus captures the necessary information to learn a policy that optimizes the expected return.

\begin{definition}[Almost Lipschitzness]
Let $\measurableset$ be a measurable set equipped with a metric $\distance \colon \measurableset \to \mathopen[0, \infty\mathclose)$ and $f \colon \measurableset \to \R$.
We say that $f$ is \emph{almost} Lipschitz continuous (e.g., \citepAR{note-vanderbei}) iff for all $\epsilon > 0$, there is a constant $K \geq 0$ so that $\abs{f\fun{x_1} - f\fun{x_2}} \leq K \distance\fun{x_1, x_2} + \epsilon$ for any $x_1, x_2 \in \measurableset$.
\end{definition}

\begin{notation}[Optimal value function]
For any MDP $\mdp$, let $\policy^{\star}$ be an optimal policy of $\mdp$, then we write $\valuessymbol{}{\star}$ for $\valuessymbol{\policy^{\star}}{}$ (see Property~\ref{prop:pomdp-values} for the values of a POMDP).
\end{notation}

\begin{lemma}\label{optimal-latent-values-almost-lipschitz}
Let $\pomdp = \pomdptuple$ be a POMDP with underlying MDP $\mdp = \mdptuple$.
Assume that $\pomdp$ is discrete, i.e., $\states$, $\actions$, and $\observations$ are finite sets.
Then, $\valuessymbol{}{\star}$ is almost Lipschitz continuous.
\end{lemma}
\begin{proof}
Define $\mathcal{V}$ as the set of real-valued bounded functions $V \colon \beliefs \to \R$ and $U \colon \beliefs \times \actions \times \mathcal{V} \to  \mathcal{V}$ as
\begin{equation*}
    U\fun{\belief, \action, V} = \rewards_\beliefs\fun{\belief, \action} + \expected{\belief' \sim \probtransitions_{\beliefs}\fun{\sampledot \mid \belief, \action}}{\discount V\fun{\belief'}}. 
\end{equation*}
The Bellman update operator is defined as $\mathcal{U} \colon \mathcal{V} \to \mathcal{V}$ as $\fun{\mathcal{U}V}\fun{\belief} = \max_{\action \in \actions} U\fun{\belief, \action, V}$ and is an isotone mapping that is a contraction under the supremum norm with fixed point $\valuessymbol{}{\star}$, i.e., $\valuessymbol{}{\star} = \mathcal{U}\valuessymbol{}{\star}$ \citep{DBLP:books/wi/Puterman94,DBLP:journals/jair/Hauskrecht00,Sutton1998ReinforcementIntroduction}.
Furthermore, for any initial value function $V_0 \in \mathcal{V}$, the sequence resulting from value iteration (VI), $V_{i + 1} = \mathcal{U}V_i$, converges to $\valuessymbol{}{\star}$ (with linear convergence rate $\discount$ \citet{DBLP:books/wi/Puterman94}): for any $\epsilon' > 0$, there is a $i \in \N$ so that for all $j \geq i$, $\norm{\valuessymbol{j}{} - \valuessymbol{}{\star}}_{\infty} \leq \epsilon'$.
Now, let $\epsilon > 0$; in particular, the latter statement holds for $\epsilon' = \nicefrac{\epsilon}{2}$.
Since the convergence of VI holds for any initial value, we assume that $V_0 \in \mathcal{V}$ has been chosen as a \emph{piecewise linear convex} (PWLC) function.
Then, for all $i \geq 0$, $V_i$ is also PWLC \citep{DBLP:journals/ior/Sondik78,DBLP:journals/ior/SmallwoodS73,DBLP:journals/jair/Hauskrecht00}.
Since $\states$ is discrete, the belief space $\beliefs$ is the standard $\abs{\states}$-dimensional simplex, so the domain of $V_i$ is compact, meaning that it is defined as a finite collection of linear functions.
Thus, $V_i$ is also $2K$-Lipschitz: one just need take $2K$ as the higher slope of these functions (in absolute value).
In consequence, for any pair of beliefs $\belief_1, \belief_2 \in \beliefs,$
\begin{align*}
    & \abs{\values{}{\star}{\belief_1} - \values{}{\star}{\belief_2}}\\
    = & \abs{\values{}{\star}{\belief_1} - \values{i}{}{\belief_1} + \values{i}{}{\belief_1} - \values{i}{}{\belief_2} + \values{i}{}{\belief_2} - \values{}{\star}{\belief_2}} \\
    \leq & \abs{\values{}{\star}{\belief_1} - \values{i}{}{\belief_1}} + \abs{ \values{i}{}{\belief_1} - \values{i}{}{\belief_2}} + \abs{ \values{i}{}{\belief_2} - \values{}{\star}{\belief_2}} \tag{Triangular inequality} \\
    \leq & 2 \epsilon' + \abs{\values{i}{}{\belief_1} - \values{i}{}{\belief_2}} \tag{by the convergence of VI} \\
    =&  \epsilon + \abs{\values{i}{}{\belief_1} - \values{i}{}{\belief_2}} \\
    \leq & \epsilon + {K} \cdot \dtv{\belief_1}{\belief_2} \tag{since $\dtv{\belief_1}{\belief_2} = \nicefrac{1}{2}\norm{b_1 - b_2}_1$},\\
\end{align*}
which means that $\valuessymbol{}{\star}$ is almost Lipschitz, by definition.
\end{proof}

\begin{corollary}\label{corr-latent-values-almost-lipschitz}
 When the temperature of the WAE-MDP (see Appendix~\ref{appendix:temperature}) and the variance of $\latentobservationfn$ go to zero (see Remark~\ref{rmk:variance}), the optimal latent value function of $\latentpomdp$ is almost Lipschitz-continuous.
\end{corollary}
\begin{proof}
Assuming the WAE-MDP temperature goes to zero, the state space of $\latentpomdp$ is discrete, $\latentdistance = \condition{\neq}$, and $\wassersteinsymbol{\latentdistance} = \dtvsymbol$.
Furthermore, $\latentobservationfn$ is deterministic as its variance goes to zero; therefore the set of observations of $\latentpomdp$ can be limited to the set of images of $\latentobservationfn_{\mu}$, which is finite since $\latentstates$ is finite.
Then Lemma~\ref{optimal-latent-values-almost-lipschitz} can be applied.
\end{proof}

\begin{theorem}\label{thm:representation-quality-appendix}
Let $\latentpolicy^{\star}$ be an optimal policy of the POMDP $\latentpomdp_{\decoderparameter}$.
Then, for any $\epsilon > 0$, there is a constant $K \geq 0$ so that for any pair of measurable histories under the original and latent models $\history_1, \history_2$ that are mapped to latent beliefs through $ \beliefencoder^{*}\fun{\history_1} = \latentbelief_1$ and $\beliefencoder^{*}\fun{\history_2} = \latentbelief_2$,
the belief representation induced by $\beliefencoder$ almost surely yields:
\begin{multline*}
    \abs{\values{\latentpolicy^{\star}}{}{\history_1} - \values{\latentpolicy^{\star}}{}{\history_2}} \leq
    K \wassersteindist{\latentdistance}{\latentbelief_1}{\latentbelief_2} +\, \epsilon \, + \\
    \frac{\localrewardloss{\historydistribution_{\latentpolicy^{\star}}} + \onpolicyrewardloss{\historydistribution_{\latentpolicy^{\star}}} + \fun{K + \discount\KV + \Rmax} \beliefloss{\historydistribution_{\latentpolicy^{\star}}} + \discount \KV \cdot \fun{ \localtransitionloss{\historydistribution_{\latentpolicy^{\star}}} +  \onpolicytransitionloss{\historydistribution_{\latentpolicy^{\star}}} + \observationloss{\historydistribution_{\latentpolicy^{\star}}}}}{1 - \discount} \fun{\historydistribution_{\latentpolicy^{\star}}\fun{\history_1}^{-1} + \historydistribution_{\latentpolicy^{\star}}\fun{\history_2}^{-1}}
\end{multline*}
when the WAE-MDP temperature as well as the variance of the observation decoder go to zero.
\end{theorem}
\begin{proof}
First, observe that for any measurable history $\history$,
$
\abs{\values{\latentpolicy^{\star}}{}{\history} - \latentvalues{\latentpolicy^{\star}}{}{\history}} \leq
\historydistribution_{\latentpolicy^{\star}}\fun{\history}^{-1} \cdot \expectedsymbol{\history' \sim \historydistribution_{\latentpolicy^{\star}}}\abs{\values{\latentpolicy^{\star}}{}{\history'} - \latentvalues{\latentpolicy^{\star}}{}{\history'}}
$ (cf. \citealt{DBLP:conf/icml/GeladaKBNB19}).
Therefore, we have:
\begin{align*}
    & \abs{\values{\latentpolicy^{\star}}{}{\history_1} - \values{\latentpolicy^{\star}}{}{\history_2}} \\
    = & \abs{\values{\latentpolicy^{\star}}{}{\history_1} - \latentvalues{\latentpolicy^{\star}}{}{\history_1} + \latentvalues{\latentpolicy^{\star}}{}{\history_1} -  \latentvalues{\latentpolicy^{\star}}{}{\history_2} + \latentvalues{\latentpolicy^{\star}}{}{\history_2} - \values{\latentpolicy^{\star}}{}{\history_2}} \\
    \leq & \abs{\values{\latentpolicy^{\star}}{}{\history_1} - \latentvalues{\latentpolicy^{\star}}{}{\history_1}} + \abs{ \latentvalues{\latentpolicy^{\star}}{}{\history_1} -  \latentvalues{\latentpolicy^{\star}}{}{\history_2}} + \abs{ \latentvalues{\latentpolicy^{\star}}{}{\history_2} - \values{\latentpolicy^{\star}}{}{\history_2}} \tag{Triangular inequality}\\
    \leq & \historydistribution_{\latentpolicy^{\star}}\fun{\history_1}^{-1} \expectedsymbol{\history \sim \historydistribution_{\latentpolicy^{\star}}} \abs{\values{\latentpolicy^{\star}}{}{\history} - \latentvalues{\latentpolicy^{\star}}{}{\history}} + \abs{ \latentvalues{\latentpolicy^{\star}}{}{\history_1} -  \latentvalues{\latentpolicy^{\star}}{}{\history_2}} \\
    &+ \historydistribution_{\latentpolicy^{\star}}\fun{\history_2}^{-1} \expectedsymbol{\history \sim \historydistribution_{\latentpolicy^{\star}}} \abs{\values{\latentpolicy^{\star}}{}{\history} - \latentvalues{\latentpolicy^{\star}}{}{\history}} \\
    \leq & \abs{ \latentvalues{\latentpolicy^{\star}}{}{\history_1} -  \latentvalues{\latentpolicy^{\star}}{}{\history_2}} \\
    &+ \frac{\localrewardloss{\historydistribution_{\latentpolicy^{\star}}} + \onpolicyrewardloss{\historydistribution_{\latentpolicy^{\star}}} + \Rmax \beliefloss{\historydistribution_{\latentpolicy^{\star}}} + \discount \KV \cdot \fun{ \localtransitionloss{\historydistribution_{\latentpolicy^{\star}}} +  \onpolicytransitionloss{\historydistribution_{\latentpolicy^{\star}}} + \beliefloss{\historydistribution_{\latentpolicy^{\star}}} + \observationloss{\historydistribution_{\latentpolicy^*}}}}{1 - \discount} \fun{\historydistribution_{\latentpolicy^{\star}}\fun{\history_1}^{-1} + \historydistribution_{\latentpolicy^{\star}}\fun{\history_2}^{-1}} \tag{Thm.~\ref{thm:value-diff-bounds}}
\end{align*}

Consider $\history_{\bot} = \arg\min \set{ \historydistribution_{\latentpolicy^{\star}}\fun{\history} \colon \history \in \support{\historydistribution_{\latentpolicy^{\star}}} \cap \mathbf{H}\fun{\pomdp, \latentpomdp}}$, where $\support{\historydistribution_{\latentpolicy^{\star}}}$ denotes the support of $\historydistribution_{\latentpolicy^{\star}}$ and $\mathbf{H}\fun{\pomdp, \latentpomdp}$ the set of histories compliant for $\pomdp$ and $\latentpomdp$ (i.e., histories whose observations are in the intersection of the reachable parts of the original and latent observation spaces). 
Notice that $\history_{\bot}$ exists almost surely, i.e., with probability one. 
First, $\historydistribution_{\latentpolicy}$, which is defined over histories from the history unfolding (cf. Section~\ref{appendix:proof-stationary}), is solely defined over episodes, i.e., histories for which the POMDP did not restart following a reset (Section~\ref{appendix:unfolding}), so we only need consider such histories.
By Assumption~\ref{assumption:episodic}, the reset state is almost surely triggered, which means that the length of histories ending in $\observation^{\star}$ (the ``reset" observation, cf. Definition~\ref{def:episodic}) is also almost surely finite (otherwise, the POMDP does not reset and this event has a zero probability). 
Second, we consider histories which are measurable in the two models (i.e., the original and latent POMDPs) and the observation space of the latent POMDP is finite (cf. the Proof of Corollary~\ref{corr-latent-values-almost-lipschitz}).
So, the set of histories from which we extract the minimum is almost surely finite.

Now, let $\epsilon > 0$.
Recall that the latent value function (defined over the latent belief space) is almost Lipschitz continuous (Corollary~\ref{corr-latent-values-almost-lipschitz}).
In particular, for $\delta = \frac{\epsilon}{\fun{1 + 2\historydistribution_{\latentpolicy^{\star}}\fun{\history_{\bot}}^{-1}}}$,
there is a $K \geq 0$ so that for any $\latentbelief, \latentbelief' \in \latentbeliefs$, $\abs{\latentvalues{}{\star}{\belief} - \latentvalues{}{\star}{\belief'}} \leq K \wassersteindist{\latentdistance}{\belief}{\belief'} + \delta$.
Then:
\begin{align*}
    & \abs{\latentvalues{}{\star}{\history_1} - \latentvalues{}{\star}{\history_2}} \\
    = & \abs{\latentvalues{}{\star}{\latentbeliefupdate^{*}\fun{\history_1}} - \latentvalues{}{\star}{\latentbeliefupdate^{*}\fun{\history_2}}} \\
    = & \abs{\latentvalues{}{\star}{\latentbeliefupdate^{*}\fun{\history_1}} - \latentvalues{}{\star}{\beliefencoder^{*}\fun{\history_1}} + \latentvalues{}{\star}{\beliefencoder^{*}\fun{\history_1}} - \latentvalues{}{\star}{\beliefencoder^{*}\fun{\history_2}} + \latentvalues{}{\star}{\beliefencoder^{*}\fun{\history_2}} - \latentvalues{}{\star}{\latentbeliefupdate^{*}\fun{\history_2}}} \\
    \leq & \abs{\latentvalues{}{\star}{\latentbeliefupdate^{*}\fun{\history_1}} - \latentvalues{}{\star}{\beliefencoder^{*}\fun{\history_1}}} + \abs{ \latentvalues{}{\star}{\beliefencoder^{*}\fun{\history_1}} - \latentvalues{}{\star}{\beliefencoder^{*}\fun{\history_2}}} + \abs{ \latentvalues{}{\star}{\beliefencoder^{*}\fun{\history_2}} - \latentvalues{}{\star}{\latentbeliefupdate^{*}\fun{\history_2}}} \tag{Triangular inequality} \\
    \end{align*}
    \begin{align*}
    \leq& \historydistribution_{\latentpolicy^{\star}}\fun{\history_1}^{-1} \expectedsymbol{\history \sim \historydistribution_{\latentpolicy^{\star}}} \abs{\latentvalues{}{\star}{\latentbeliefupdate^{*}\fun{\history}} - \latentvalues{}{\star}{\beliefencoder^{*}\fun{\history}}} + \abs{\latentvalues{}{\star}{\belief_1} - \latentvalues{}{\star}{\belief_2}} \\
    &+ \historydistribution_{\latentpolicy^{\star}}\fun{\history_2}^{-1} \expectedsymbol{\history \sim \historydistribution_{\latentpolicy^{\star}}} \abs{\latentvalues{}{\star}{\latentbeliefupdate^{*}\fun{\history}} - \latentvalues{}{\star}{\beliefencoder^{*}\fun{\history}}} \\
    = & \fun{\historydistribution_{\latentpolicy^{\star}}\fun{\history_1}^{-1}+ \historydistribution_{\latentpolicy^{\star}}\fun{\history_2}^{-1}} \expectedsymbol{\history \sim \historydistribution_{\latentpolicy^{\star}}} \abs{\latentvalues{}{\star}{\latentbeliefupdate^{*}\fun{\history}} - \latentvalues{}{\star}{\beliefencoder^{*}\fun{\history}}} + \abs{\latentvalues{}{\star}{\belief_1} - \latentvalues{}{\star}{\belief_2}} \\
    \leq& \fun{\historydistribution_{\latentpolicy^{\star}}\fun{\history_1}^{-1}+ \historydistribution_{\latentpolicy^{\star}}\fun{\history_2}^{-1}} \expected{\history \sim \historydistribution_{\latentpolicy^{\star}}}{ K \wassersteindist{\latentdistance}{\latentbeliefupdate^{*}\fun{\history}}{\beliefencoder^{*}\fun{\history}} + \delta} + K \wassersteindist{\latentbelief}{\latentbelief_1}{\latentbelief_2} + \delta \tag{$\latentvaluessymbol{}{\star}$ is almost Lipschitz} \\
    = & \fun{\historydistribution_{\latentpolicy^{\star}}\fun{\history_1}^{-1}+ \historydistribution_{\latentpolicy^{\star}}\fun{\history_2}^{-1}} \fun{K \beliefloss{} + \delta} + K \wassersteindist{\latentbelief}{\latentbelief_1}{\latentbelief_2} + \delta \tag{by definition of $\beliefloss{}$}\\
    = & \fun{\historydistribution_{\latentpolicy^{\star}}\fun{\history_1}^{-1}+ \historydistribution_{\latentpolicy^{\star}}\fun{\history_2}^{-1}} K \beliefloss{} + K \wassersteindist{\latentbelief}{\latentbelief_1}{\latentbelief_2} + \delta \fun{1 + \historydistribution_{\latentpolicy^{\star}}\fun{\history_1}^{-1}+ \historydistribution_{\latentpolicy^{\star}}\fun{\history_2}^{-1}} \\
    \leq & \fun{\historydistribution_{\latentpolicy^{\star}}\fun{\history_1}^{-1}+ \historydistribution_{\latentpolicy^{\star}}\fun{\history_2}^{-1}} K \beliefloss{} + K \wassersteindist{\latentbelief}{\latentbelief_1}{\latentbelief_2} + \delta \fun{1 + 2 \historydistribution_{\latentpolicy^{\star}}\fun{\history_{\bot}}^{-1}} \\
    = & \fun{\historydistribution_{\latentpolicy^{\star}}\fun{\history_1}^{-1}+ \historydistribution_{\latentpolicy^{\star}}\fun{\history_2}^{-1}} K \beliefloss{} + K \wassersteindist{\latentbelief}{\latentbelief_1}{\latentbelief_2} + \epsilon
\end{align*}
Putting all together, we have that for any $\epsilon > 0$, there almost surely exists a constant $K \geq 0$ so that:
\begin{multline*}
    \abs{\values{\latentpolicy^{\star}}{}{\history_1} - \values{\latentpolicy^{\star}}{}{\history_2}} \leq
    K \wassersteindist{\latentdistance}{\latentbelief_1}{\latentbelief_2} + \epsilon + \\
    \frac{\localrewardloss{\historydistribution_{\latentpolicy^{\star}}} + \onpolicyrewardloss{\historydistribution_{\latentpolicy^{\star}}} + \fun{K + \discount \KV + \Rmax} \beliefloss{\historydistribution_{\latentpolicy^{\star}}} + \discount \KV \cdot \fun{ \localtransitionloss{\historydistribution_{\latentpolicy^{\star}}} +  \onpolicytransitionloss{\historydistribution_{\latentpolicy^{\star}}} + \observationloss{\historydistribution_{\latentpolicy^{\star}}}}}{1 - \discount} \fun{\historydistribution_{\latentpolicy^{\star}}\fun{\history_1}^{-1} + \historydistribution_{\latentpolicy^{\star}}\fun{\history_2}^{-1}}.
\end{multline*}
\end{proof}
\newpage
\section{Algorithm}\label{appendix:algorithm}
We describe the final WBU learning procedure in Algorithm~\ref{alg:wbu}. Note that they keyword \textbf{Update} means that we compute the gradients of the input loss, and update the parameters of the neural networks of the pointed function/model accordingly.

\smallparagraph{Normalizing term.}~Given the set of parameters $\iota$ of $\beliefencoder$, we minimize the KL divergence $\dklsymbol$ by gradient descent on the Monte-Carlo estimate of the divergence:
\begin{multline*}
    \gradient_{\iota} {\dkl{\beliefencoder\fun{ \latentbelief_t, \action_t, \observation_{t + 1}; \iota}}{\latentbeliefupdate\fun{ \latentbelief_t, \action_t, \observation_{t + 1}}}} = \\
    \gradient_{\iota}  \! \! \!\expectedsymbol{\latentstate_{t+1} \sim \beliefencoder\fun{\latentbelief_t, \action_t, \observation_{t + 1}; \iota}} \left[   \log{\beliefencoder\fun{\latentstate_{t+1} \mid \latentbelief_t, \action_t, \observation_{t + 1} ; \iota}} \vphantom{ \log \expectedsymbol{\latentstate }{}} - \log \! \expectedsymbol{\latentstate \sim \latentbelief_t} \!{\latentprobtransitions_{\decoderparameter} \fun{\latentstate_{t+1} \mid \latentstate, \action_t} - \log \latentobservationfn_{\decoderparameter}\fun{\observation_{t + 1} \mid \latentstate_{t + 1}} } \right]
\end{multline*}
Notice that the first term of the divergence (the belief normalization term) of Eq.~\ref{eq:dkl} does not depend on $\beliefencoder$ and thus yields zero gradient.  
Nevertheless, we observed during our experiments that adding the normalizing term allows to stabilize and reduce the variance of the belief loss.

\smallparagraph{Optimizing Wasserstein.}~%
To optimize the Wasserstein term of the belief losses, we follow the same learning procedure than \cite[Appenix A.5]{delgrange2023wasserstein}: we introduce neural networks $\mathcal{F}_{\spadesuit}$ (for $\spadesuit \in \set{\latentprobtransitions, \observations}$) that are trained to attain the supremum of the dual formulation of the Wasserstein distance.
To do so, we need enforce the Lipschitzness of  $\mathcal{F}_{\spadesuit}$ and, as in WAE-MDPs \citep{delgrange2023wasserstein}, we do so via the gradient penalty approach of \citetAR{DBLP:conf/nips/GulrajaniAADC17}, leveraging that any differentiable function is 1-Lipschitz iff it has gradients with norm at most $1$ everywhere.
Finally, notice that we do not directly optimize the total variation distance of $\observationloss{}$, but rather the Wasserstein; we take the usual Euclidean distance as metric over $\observations$ which is proven to be Lipschitz equivalent to a distance converging to the discrete metric as the temperature of the WAE-MDP goes to zero \cite[Appendix~A.6]{delgrange2023wasserstein} to finally recover $\dtvsymbol$.
\begin{algorithm}%
\caption{\textsc{Wasserstein Belief Updater}}\label{alg:wbu}
\DontPrintSemicolon
\SetKwFor{RepTimes}{repeat}{times}{end}
\KwIn{Batch sizes $B_{\textsc{WBU}}$, $B_{\textsc{WAE}}, B_{\textsc{Next}}$; global learning steps $N$; no. of model updates per iteration $N_{\textsc{Model}}$; your favorite collect strategy $\policy_{\text{init}}$; replay buffer (RB) $\replaybuffer$; Lipschitz networks $\mathcal{F}_{\scriptscriptstyle \latentprobtransitions} \colon \latentstates \to \R$, $\mathcal{F}_{\scriptscriptstyle \observations} \colon \observations \to \R$; observation variance network $\latentobservationfn_{\sigma}$; and loss weights $w_{\scriptscriptstyle\latentrewards}, w_{\scriptscriptstyle\latentprobtransitions}$}
\SetKwComment{Comment}{$\triangleright$\ }{}
\SetCommentSty{textit}
\SetKwBlock{Begin}{function}{end function}

\textbf{collect} $\theta = \set{\state_i, \observation_i, \action_i, \reward_i, \state'_i, \observation'_i}_{i = 1}^{N_{\text{init}}}$ by executing $\policy_{\text{init}}$ for $N_{\textit{init}}$ steps; $\,$ \textbf{store} $\theta$ in $\replaybuffer$ \; \Comment*[r]{Use the exploration policy $\policy_{\text{init}}$ to collect transitions and initialize the RB}
\RepTimes{$N$}{
		\RepTimes{$N_{\textsc{Model}}$}{
		\Comment{Update the WAE-MDP model for $N_{\textsc{Model}}$ consecutive training steps}
		\For{$i \leftarrow 1$ to $B_{\textsc{WAE}}$}{
			$\tuple{\state_i, \observation_i, \action_i, \reward_i, \state'_i, \observation'_i} \sim \replaybuffer$ \Comment*[r]{Sample a transition from the RB}
			$\latentstate' \leftarrow \embed_{\encoderparameter}\fun{\state'_i, \observation'_i}$ \Comment*[r]{Embed $\tuple{\state'_i, \observation'_i}$ to the latent space}
			$\tilde{\observation}'_i \sim \latentobservationfn_{\decoderparameter}\fun{\sampledot \mid \latentstate'}$\Comment*[r]{Observe the resulting latent state via $\latentobservationfn$}
		}
		$\mathcal{L}_{\textsc{WAE}} \leftarrow$ \textbf{compute the WAE-MDP loss} on  transition batch $\set{\state_i, \observation_i, \action_i, \reward_i, \state_i', \observation_i'}_{i=1}^{B_{\textsc{WAE}}}$\;
		\textbf{Update} the WAE-MDP components (\emph{in particular, those of Eq.~\ref{eq:wae-mdp-components}}) by minimizing $\mathcal{L}_{\textsc{WAE}}$\;
		$\mathcal{L}_{\mathcal{O}} \leftarrow \nicefrac{1}{B_{\textsc{WAE}}} \cdot \sum_{i = 1}^{B_{\textsc{WAE}}}{\left[\mathcal{F}_{\scriptscriptstyle \observations}\fun{\observation'_i} - \mathcal{F}_{\scriptscriptstyle \observations}\fun{\tilde{\observation}'_i}\right]}$ \Comment*[r]{Observation loss}
		\textbf{Update} $\mathcal{F}_{\scriptscriptstyle \observations}$ by maximizing  $\mathcal{L}_{\mathcal{O}}$ \textbf{and} enforcing its $1$-{Lipschizness} w.r.t. metric $\distance_{\observations}$\;
		\textbf{Update} $\latentobservationfn_{\sigma}$ by minimizing $\mathcal{L}_{\mathcal{O}}$\;
	}
    \For{$i \leftarrow 1$ to $B_{\textsc{WBU}}$}{
        $\state_0 \leftarrow \sinit$; $\latentstate_0 \leftarrow \zinit; \;
        \latentbelief_0 \leftarrow \diracimpulsesymbol_{\latentstate_0} ; \; \beta_0 \leftarrow \beta_I$ \Comment*[r]{$\beta_I$ is arbitrary, e.g., zeroes}
       \For{$t \leftarrow 0$ to $T$}{
            $\action_t \sim \latentpolicy\fun{\sampledot \mid \beta_t}$ \Comment*[r]{Produce the action $\action_t$ according to the sub-belief $\beta_t$}
            $\textbf{execute } \action $ in the environment, \textbf{receive reward} $\reward_t$, and \textbf{perceive} the next state-observation $\tuple{\state_{t + 1}, \observation_{t + 1}}$\;
            \textbf{store} the transition $\tuple{\state_t, \observation_t, \action_t, \reward_t, \state_{t + 1}, \observation_{t + 1}}$ into $\replaybuffer$\;
            $\beta_{t + 1} \leftarrow \beliefencoder^{\textit{sub}}\fun{\text{sg}\fun{\beta_{t}}, \action_t, \observation_{t + 1}}$ \Comment*[r]{Update the sub-belief; \textnormal{sg} is stop gradients}
            $\latentbelief_{t + 1} \leftarrow \mathbb{M}_{\encoderparameter}\fun{\beta_{t + 1}}$ \Comment*[r]{Retrieve the belief distribution $\belief_{t + 1}$ via the MAF $\mathbb{M}$}
            $\latentstate_{t + 1} \sim \latentbelief_{t + 1}$\Comment*[r]{\textbf{Believe} the next latent state}
            \Comment{Marginalize the next latent state distribution w.r.t. the current belief}
            \For{$j \leftarrow 1$ to $B_{\textsc{Next}}$}{
                $\latentstate \sim \latentbelief_t$; $\;$
$\mathcal{L}^{j}_{\log\latentprobtransitions} \leftarrow \left[\log \latentprobtransitions\fun{\latentstate_{t + 1} \mid \latentstate, \action_t} - \log B_{\textsc{Next}} \right]$} 
        $\mathcal{L}_{\textsc{KL}}^{i, t} \leftarrow \log \latentbelief_{t + 1}\fun{\latentstate_{t + 1}} - \textsc{LSE}(\set{\mathcal{L}^{j}_{\log\latentprobtransitions}}{}_{j = 1}^{B_{\textsc{Next}}}) - \log \latentobservationfn\fun{o_{t + 1} \mid \latentstate_{t + 1}}$ \; \Comment*[r]{Pointwise decomposition of Eq.~\ref{eq:dkl}: divergence with the belief update rule}
        $\mathcal{L}^{i, t}_{\latentrewards} \leftarrow \abs{ \latentrewards\fun{\embed\fun{\state_t, \observation_t}, \action_t} - \latentrewards\fun{\latentstate_t, \action_t} }$
        \Comment*[r]{Latent reward regularizer}
        $\latentstate' \sim \latentprobtransitions\fun{\sampledot \mid \latentstate_t, \action_t}$ \; \Comment*[r]{Transition to the next latent state from the current believed latent state}
         $\mathcal{L}^{i, t}_{\latentprobtransitions} \leftarrow {\left[ \mathcal{F}_{\scriptscriptstyle \latentprobtransitions}\fun{\embed\fun{\state_{t + 1}, \observation_{t + 1}}} - \mathcal{F}_{\scriptscriptstyle \latentprobtransitions}\fun{\latentstate'} \right]}$
        \Comment*[r]{Latent transition regularizer}
    } 
    }
 	\textbf{Update} $\mathcal{F}_{\scriptscriptstyle \latentprobtransitions}$ by maximizing $\sum_{i = 1}^{B_{\textsc{WBU}}}\sum_{t = 0}^{T - 1} \mathcal{L}^{i, t}_{\latentprobtransitions}$  \textbf{and} enforcing its $1$-Lipschitzness w.r.t. latent metric $\latentdistance$ \;
     \textbf{Update} $\beliefencoder^{\emph{sub}}$ and $\mathbb{M}$ by minimizing $\nicefrac{1}{\fun{B_{\textsc{WBU}} + T}} \sum_{i = 1}^{B_{\textsc{WBU}}}\sum_{t = 0}^{T - 1} (\mathcal{L}^{i, t}_{\textsc{KL}} + w_{\scriptscriptstyle \latentrewards} \cdot \mathcal{L}^{i, t}_{\latentrewards} + w_{\scriptscriptstyle \latentprobtransitions} \cdot \mathcal{L}^{i, t}_{\latentprobtransitions} )$ \;
	\textbf{Update} $\latentpolicy$ by minimizing the \textsc{A2C} loss on the batch $ \set{\beta^i_{\scriptscriptstyle 0: T}, \action^i_{\scriptscriptstyle0: T - 1}, \reward^i_{\scriptscriptstyle0: T - 1}}_{i = 1}^{B_{\textsc{WBU}}}$\;
}

\Begin($\latentobservationfn_{\decoderparameter}\fun{\sampledot \mid \latentstate}$ \hfill $\triangleright\ \fun{\textit{smooth}}$ \emph{Observation Filter}){
	$\mu \leftarrow \augmentedobservationfn\fun{\latentstate}; \sigma \leftarrow \latentobservationfn_{\sigma}\fun{\latentstate}$ \Comment*[r]{Decode $\latentstate$; get the standard deviation of the reconstruction}
	\Return{$\normal{\mu}{\sigma^2}{}$}
}
\end{algorithm}

\section{Hyperparameters} 
\label{appendix:hyper}

Table \ref{tab:hyperparameters} provides the range of hyperparameters used in the search, along with the selected values for each environment. The hyperparameter search was performed using \textsc{Optuna} \citepAR{optuna}. Pre-training of the WAE-MDP involved collecting $10240$ transitions with a random policy and performing $200$ training steps. These pre-training transitions are taken into account in the reported results.

Additionally, Table \ref{tab:rnnhyper} (for R-A2C) and Table \ref{tab:dvrlhyper} (for DVRL) present the specific hyperparameters used for each algorithm. A grid search was conducted over all possible combinations for both baselines. The hidden size of all neural networks was set to $128$ neurons and two hidden layers (except for the sub-belief encoder, and the policy which uses one) without further tuning. The experiments were carried out using $16$ parallel environments. The original implementation of DVRL and their version of R-A2C were used in this study.

We ran the experiments on a cluster composed of Intel Xeon Gold 6148 CPU.

In the following tables the enviromnents are represented as follows: \textsc{StatelessCartPole} $\heartsuit$, \textsc{NoisyStatelessCartPole} $\clubsuit$, \textsc{RepeatPrevious} $\diamondsuit$, \textsc{SpaceInvaders} $\spadesuit$, \textsc{NoisySpaceInvaders} $\bigstar$.

The t-SNE \citep{JMLR:v9:vandermaaten08a} presented in the main text, in Fig.~\ref{fig:tsne}, was obtained using cuML \citep{raschka2020machine}.
Table \ref{tab:tsne} reports the hyperparameters that were used.
We used 50,000 sub-belief-value pairs, sampled from a dataset of 1 million timesteps collected at the end of the training.

\begin{table}[h]
\caption{Range of hyperparameter search and selection per environment.}
\label{tab:hyperparameters}
\resizebox{\textwidth}{!}{%
\begin{tabular}{l|l|lllll}
               & \textbf{Range}                    & $\heartsuit$ & $\clubsuit$ & $\diamondsuit$ & $\spadesuit$ & $\bigstar$ \\ \hline
WAE Updates per Belief Update      & 1-2                                          & 1                                                             & 1                                                        & 2                                                     & 2                                                     & 2                                                           \\
Activation function                & leaky relu, elu                              & leaky relu                                                    & \multicolumn{1}{l|}{leaky relu}                          & elu                                                   & elu                                                   & relu                                                        \\
Activation function lipshitz       & leaky relu, tanh                             & tanh                                                          & tanh                                                     & tanh                                                  & tanh                                                  & leaky relu                                                  \\
Activation CNN                     & leaky relu, elu                              &                                                               &                                                          &                                                       & elu                                                   & elu                                                         \\ \hline
\textbf{Policy config}             &                                              &                                                               &                                                          &                                                       &                                                       &                                                             \\
Learning rate                      & 1.e-4, 3.e-4, 5.e-4                          & 3.e-4                                                         & 1.e-4                                                    & 1.e-4                                                 & 5.e-4                                                 & 5.e-4                                                       \\
Clip norm                          & 1, 10                                        & 10                                                            & 10                                                       & 10                                                    & 10                                                    & 10                                                          \\ \hline
\textbf{Belief config}             &                                              &                                                               &                                                          &                                                       &                                                       &                                                             \\
Same opt as policy                 & Yes, No                                      & Yes                                                           & Yes                                                      & Yes                                                   & No                                                    & No                                                          \\
Learning rate                      & 1.e-4, 3.e-4, 5.e-4                          &                                                               &                                                          &                                                       & 1.e-4                                                 & 3.e-4                                                       \\
Loss factor                        & 1.e-1, 1.e-2, 1.e-3, 1.e-4, 1.e-5            & 1.e-1                                                         & 1.e-4                                                    & 1.e-5                                                 &                                                       &                                                             \\
Filter variance min                & 1.e-2, 1.e-3                                 & 1.e-2                                                         & 1.e-2                                                    & 1.e-3                                                 & 1.e-2                                                 & 1.e-2                                                       \\
Normalize log obs filter           & True, False                                  & False                                                         & False                                                    & True                                                  & True                                                  & True                                                        \\
Sub belief prior temperature       & .5, 0.66, .75, .9, .99                       & 0.66                                                          & 0.9                                                      & 0.66                                                  & 0.99                                                  & 0.99                                                        \\
Reward loss scale factor           & 0.1, 1., 10., 20., 50., 100.                 & 0.1                                                           & 20                                                       & 100                                                   & 100                                                   & 100                                                         \\
Transition loss scale factor       & 0.1, 1., 10., 20., 50., 100.                 & 50                                                            & 100                                                     & 50                                                    & 50                                                    & 100                                                         \\
Buffer size                        & 4096, 8192, 16384                            & 4096                                                          & 4096                                                     & 16384                                                 & 4096                                                  & 4096                                                        \\
N Critic Update                    & 5, 10                                        & 10                                                            & 10                                                       & 5                                                     & 5                                                     & 5                                                           \\
N State samples                    & 16, 32, 64                                   & 32                                                            & 64                                                       & 32                                                    & 32                                                    & 32                                                          \\
N next state samples               & 16, 32, 64                                   & 32                                                            & 64                                                       & 64                                                    & 16                                                    & 32                                                          \\ \hline
\textbf{WAE config}                &                                              &                                                               &                                                          &                                                       &                                                       &                                                             \\
Latent state size                  & $\heartsuit, \clubsuit: 5\to10$, Others: $18\to25$        & 5                                                             & 6                                                        & 20                                                    & 18                                                    & 19                                                          \\
Minimizer learning rate            & 1.e-3, 3.e-4,1.e-4, 5.e-5, 1.e-5             & 5.e-5                                                         & 5.e-5                                                    & 3.e-4                                                 & 1.e-5                                                 & 1.e-5                                                       \\
Maximizer learning rate            & 1.e-3, 3.e-4,1.e-4, 5.e-5, 1.e-5             & 5.e-5                                                         & 5.e-5                                                    & 1.e-3                                                 & 3.e-4                                                 & 1.e-4                                                       \\
State encoder temperature          & 0.33, .5, 0.66, .75, .9, .99                 & 0.33                                                          & 0.75                                                     & 0.66                                                  & 0.5                                                   & 0.5                                                         \\
State prior temperature            & 0.33, .5, 0.66, .75, .9, .99                 & 0.75                                                          & 0.75                                                     & 0.5                                                   & 0.75                                                  & 0.75                                                        \\
Local transition loss scaling      & 10., 25., 50., 75., 80.                      & 50                                                            & 80                                                       & 80                                                    & 10                                                    & 10                                                          \\
Steady state scaling               & 10., 25., 50., 75., 80.                      & 80                                                            & 80                                                       & 100                                                   & 25                                                    & 25                                                          \\
N critic update                    & 5, 10                                        & 10                                                            & 10                                                       & 5                                                     & 5                                                     & 10                                                          \\
Batch size                         & 128, 256                                     & 128                                                           & 128                                                      & 256                                                   & 128                                                   & 128                                                         \\
Clip grad                          & 1, 10. 100.                                  & 100                                                           & 100                                                      & 100                                                   & 10                                                    & 10                                                          \\
State vs Obs reconstruction weight & 1, 2                                         & 2                                                             & 2                                                        & 2                                                     & 1                                                     & 1                                                           \\
Obs reg. min learning rate         & 1.e-3, 3.e-4,1.e-4, 5.e-5, 1.e-5             & 5.e-5                                                         & 1.e-4                                                    & 1.e-3                                                 & 5.e-5                                                 & 5.e-5                                                       \\
Obs reg. max learning rate         & 1.e-3, 3.e-4,1.e-4, 5.e-5, 1.e-5             & 5.e-5                                                         & 3.e-4                                                    & 5.e-5                                                 & 1.e-5                                                 & 3.e-4                                                       \\
Obs reg. gradient penalty          & 50, 100, 500, 1000                           & 50                                                            & 50                                                       & 1000                                                  & 100                                                   & 1000                                                       
\end{tabular}
}
\end{table}

\begin{table}[h!]
\caption{R-A2C hyperparameters}
\label{tab:rnnhyper}
\resizebox{\textwidth}{!}{%
\begin{tabular}{l|l|lllll}
               & \textbf{Range}                    & $\heartsuit$ & $\clubsuit$ & $\diamondsuit$ & $\spadesuit$ & $\bigstar$ \\ \hline 
Opitmizer      & Adam, RMSProp                     & RMSProp      & RMSProp     & Adam           & RMSProp      & RMSProp    \\
Clip grad norm & 0.5, 1., 10.                      & 1.0          & 10          & 10.            & 10           & 10         \\
Learning rate  & 3.e-5, 1.e-4, 3.e-4, 5.e-4, 1.e-3 & 5.e-4        & 1.e-4       & 5.e-4          & 1.e-3        & 1.e-3     
\end{tabular}
}
\end{table}

\begin{table}[h!]
\caption{DVRL hyperparameters}
\label{tab:dvrlhyper}
\resizebox{\textwidth}{!}{%
\begin{tabular}{l|l|lllll}
               & \textbf{Range}                    & $\heartsuit$ & $\clubsuit$ & $\diamondsuit$ & $\spadesuit$ & $\bigstar$ \\ \hline 
Optimizer      & Adam, RMSProp                     & Adam         & Adam        & RMSProp        & RMSProp      & RMSProp    \\
Clip grad norm & 0.5, 1., 10.                      & 10.0         & 1.0         & 1.0            & 10           & 0.5        \\
Learning rate  & 3.e-5, 1.e-4, 3.e-4, 5.e-4, 1.e-3 & 5.e-4        & 1.e-3       & 3.e-4          & 1.e-3        & 1.e-3      \\ \hline \\
Encoding loss factor & 1,.1,.5,.05                 & 0.5          & 1.0         & 1.0            & 1.0          & 0.05       \\
Number of particles  & 5,10,15                     & 10           & 3           & 10             & 5            & 5
\end{tabular}
}
\end{table}

\begin{table}[h!]
    \caption{t-SNE hyperparameters}
    \centering
    \begin{tabular}{c|c}
         Method& FFT\\
         Perplexity& 10\\
         N Neighbors& 50\\
         Metric& Manhattan\\
         Learning rate& 2400\\
         Early exaggeration& 50\\
         Late exaggeration& 3\\
         Exaggeration iter& 300\\
    \end{tabular}
    \label{tab:tsne}
\end{table}

\newpage

\bibliographyAR{references}
\bibliographystyleAR{iclr2024_conference}

\end{appendices}

\end{document}